\documentclass{article} %
\usepackage{iclr2024_conference,times}

\usepackage{amsmath}
\usepackage{amssymb}
\usepackage{mathtools}
\usepackage{multirow}
\usepackage{array,multirow,graphicx}
\numberwithin{equation}{section}
\usepackage{slashed}
\usepackage{braket}
\usepackage{enumitem}

\definecolor{DarkGreen}{rgb}{0,0.40,0}
\definecolor{FireBrick}{rgb}{0.698,0.133,0.133}
\usepackage[colorlinks,citecolor=DarkGreen,linkcolor=FireBrick,urlcolor=FireBrick]{hyperref}

\usepackage{graphicx}
\usepackage{tikz}
\usepackage{tikz-cd}
\usepackage{times}
\usepackage{courier}
\usepackage{url}    
\usepackage{bm}
\usepackage{subfig}
\usepackage{physics}
\usepackage{xcolor}
\usepackage{natbib}
\usepackage{mdframed}
\usepackage{nicefrac}
\usepackage{booktabs}
\usepackage{lipsum}
\usepackage{titlesec}
\usepackage{wrapfig,lipsum,booktabs}
\usepackage{authblk}
\usepackage{blindtext}
\usepackage{algorithm}
\usepackage{algorithmicx}
\usepackage{algpseudocode}
\usepackage{titletoc}
\usepackage{authblk}

\usepackage{booktabs}  %
\usepackage{cellspace}  %
\usepackage[nameinlink,capitalize,noabbrev]{cleveref}

\usepackage{CJK}

\newenvironment{claim}{  \begin{mdframed}[linecolor=black!0,backgroundcolor=black!10]\noindent%
		\ignorespaces}{\end{mdframed}}

\renewenvironment{figure}[1][]{
  \begin{originalfigure}[#1]
    \begin{mdframed}[linecolor=black!0,backgroundcolor=black!1]
}{
    \end{mdframed}
  \end{originalfigure}
}

\usepackage{array}
\usepackage{makecell}
\usepackage{multirow}

\def\half{{\frac{1}{2}}}

\newcommand{\bea}{\begin{eqnarray}}
\newcommand{\eea}{\end{eqnarray}}

\usepackage{slashed}

\def\({\left(}
\def\){\right)}
\def\[{\left[}
\def\]{\right]}

\usepackage{dashbox}
\usepackage{xcolor}
\usepackage{colortbl}
\usepackage{arydshln}
\definecolor{lightyellow}{rgb}{1.0, 0.95, 0.7}
\definecolor{blue}{rgb}{0, 0.1, 0.6}
\definecolor{Blue}{rgb}{0,0,1}
\definecolor{darkgreen}{rgb}{0,0.40,0}
\definecolor{firebrick}{rgb}{0.698,0.133,0.133}

\definecolor{colorA}{rgb}{1,0,0}
\definecolor{colorB}{rgb}{0,0.3,1}
\definecolor{colorC}{rgb}{0.9,0.8,0.2}
\definecolor{colorD}{rgb}{0,0.65,0}
\definecolor{lesslightgray}{rgb}{0.5,0.5,0.5}
\definecolor{light-gray}{gray}{0.95}

\let\tilde\widetilde
\let\hat\widehat

\newcommand{\calH}{\mathcal{H}}

\newcommand{\calO}{\mathcal{O}}
\newcommand{\calP}{\mathcal{P}}

\newcommand{\calT}{\mathcal{T}}

\newcommand{\bE}{\mathbf{E}}

\newcommand{\bK}{\mathbf{K}}

\newcommand{\bQ}{\mathbf{Q}}
\newcommand{\bR}{\mathbf{R}}
\newcommand{\bS}{\mathbf{S}}

\newcommand{\bV}{\mathbf{V}}
\newcommand{\bW}{\mathbf{W}}
\newcommand{\bX}{\mathbf{X}}
\newcommand{\bY}{\mathbf{Y}}
\newcommand{\bZ}{\mathbf{Z}}
\newcommand{\ba}{\mathbf{a}}
\newcommand{\bb}{\mathbf{b}}

\newcommand{\bp}{\mathbf{p}}
\newcommand{\bs}{\mathbf{s}}

\newcommand{\bx}{\mathbf{x}}
\newcommand{\by}{\mathbf{y}}
\newcommand{\bz}{\mathbf{z}}
\newcommand{\bxi}{{\bm{\xi}}}

\newcommand{\Max}{\mathop{\rm Max}}
\newcommand{\Min}{\mathop{\rm Min}}
\newcommand{\argmin}{\mathop{\mathrm{ArgMin}}}  
\newcommand{\argmax}{\mathop{\mathrm{ArgMax}}}

\newcommand{\Softmax}{\mathop{\rm{Softmax}}}
\newcommand{\Sparsemax}{\mathop{\rm{Sparsemax}}}

\newcommand{\lse}{\mathop{\rm{lse}}}
\newcommand{\alphaentmax}{\mathop{\alpha\text{-}\mathrm{EntMax}}}

\newcommand{\sT}{ \mathsf{T} }

\newcommand{\sumM}{\sum_{\mu=1}^M}

\def\R{\mathbb{R}}

\let\cite\citep 

\usepackage{amsthm}
\makeatletter
\def\th@remark{%
  \thm@headfont{\bfseries}%
  \normalfont %
  \thm@preskip\topsep \divide\thm@preskip\tw@
  \thm@postskip\thm@preskip
}
\makeatother
\usepackage[many]{tcolorbox}
\theoremstyle{definition}
\newtheorem{theorem}{Theorem}[section]
\tcolorboxenvironment{theorem}{
  breakable,
  colback=black!10,
  colframe=white,%
  width=\dimexpr\linewidth+10pt\relax,%
  enlarge left by=-5pt,%
  enlarge right by=-5pt,%
  boxsep=5pt,%
  boxrule=0pt,
  left=0pt,right=0pt,top=0pt,bottom=0pt,
  sharp corners,
  before skip=\topsep,
  after skip=\topsep
}
\newtheorem{lemma}{Lemma}[section]
\tcolorboxenvironment{lemma}{
  breakable,
  colback=black!10,
  colframe=white,%
  width=\dimexpr\linewidth+10pt\relax,%
  enlarge left by=-5pt,%
  enlarge right by=-5pt,%
  boxsep=5pt,%
  boxrule=0pt,
  left=0pt,right=0pt,top=0pt,bottom=0pt,
  sharp corners,
  before skip=\topsep,
  after skip=\topsep
}
\newtheorem{corollary}{Corollary}[theorem]
\tcolorboxenvironment{corollary}{
  breakable,
  colback=black!10,
  colframe=white,%
  width=\dimexpr\linewidth+10pt\relax,%
  enlarge left by=-5pt,%
  enlarge right by=-5pt,%
  boxsep=5pt,%
  boxrule=0pt,
  left=0pt,right=0pt,top=0pt,bottom=0pt,
  sharp corners,
  before skip=\topsep,
  after skip=\topsep
}

\theoremstyle{definition}
\newtheorem{definition}{Definition}[section]
\tcolorboxenvironment{definition}{
  breakable,
  colback=black!10,
  colframe=white,%
  width=\dimexpr\linewidth+10pt\relax,%
  enlarge left by=-5pt,%
  enlarge right by=-5pt,%
  boxsep=5pt,%
  boxrule=0pt,
  left=0pt,right=0pt,top=0pt,bottom=0pt,
  sharp corners,
  before skip=\topsep,
  after skip=\topsep
}
\theoremstyle{remark}
\newtheorem{remark}{Remark}[section]

\tcolorboxenvironment{assumption}{
  breakable,
  colback=black!10,
  colframe=white,%
  width=\dimexpr\linewidth+10pt\relax,%
  enlarge left by=-5pt,%
  enlarge right by=-5pt,%
  boxsep=5pt,%
  boxrule=0pt,
  left=0pt,right=0pt,top=0pt,bottom=0pt,
  sharp corners,
  before skip=\topsep,
  after skip=\topsep
}

\tcolorboxenvironment{problem}{
  breakable,
  colback=black!10,
  colframe=white,%
  width=\dimexpr\linewidth+10pt\relax,%
  enlarge left by=-5pt,%
  enlarge right by=-5pt,%
  boxsep=5pt,%
  boxrule=0pt,
  left=0pt,right=0pt,top=0pt,bottom=0pt,
  sharp corners,
  before skip=\topsep,
  after skip=\topsep
}

\crefname{theorem}{Theorem}{Theorems}
\crefname{proposition}{Proposition}{Propositions}
\crefname{lemma}{Lemma}{Lemmas}
\crefname{corollary}{Corollary}{Corollaries}
\crefname{definition}{Definition}{Definitions}
\crefname{assumption}{Assumption}{Assumptions}
\crefname{remark}{Remark}{Remarks}
\crefname{problem}{Problem}{Problems}
\crefname{property}{Property}{property}

\numberwithin{equation}{section}
\numberwithin{theorem}{section}
\numberwithin{proposition}{section}
\numberwithin{definition}{section}
\numberwithin{lemma}{section}
\numberwithin{assumption}{section}
\numberwithin{remark}{section}

\usepackage{lipsum}

\newcommand\blfootnote[1]{%
  \begingroup
  \renewcommand\thefootnote{}\footnote{#1}%
  \addtocounter{footnote}{-1}%
  \endgroup
}
\newcommand*{\annot}[1]{\tag*{\footnotesize{\textcolor{black!50}{\big(#1\big)}}}}

\makeatletter
\let\save@mathaccent\mathaccent
\newcommand*\if@single[3]{%
    \setbox0\hbox{${\mathaccent"0362{#1}}^H$}%
    \setbox2\hbox{${\mathaccent"0362{\kern0pt#1}}^H$}%
    \ifdim\ht0=\ht2 #3\else #2\fi
}
\newcommand*\rel@kern[1]{\kern#1\dimexpr\macc@kerna}
\newcommand*\widebar[1]{\@ifnextchar^{{\wide@bar{#1}{0}}}{\wide@bar{#1}{1}}}
\newcommand*\wide@bar[2]{\if@single{#1}{\wide@bar@{#1}{#2}{1}}{\wide@bar@{#1}{#2}{2}}}
\newcommand*\wide@bar@[3]{%
    \begingroup
    \def\mathaccent##1##2{%
        \let\mathaccent\save@mathaccent
        \if#32 \let\macc@nucleus\first@char \fi
        \setbox\z@\hbox{$\macc@style{\macc@nucleus}_{}$}%
        \setbox\tw@\hbox{$\macc@style{\macc@nucleus}{}_{}$}%
        \dimen@\wd\tw@
        \advance\dimen@-\wd\z@
        \divide\dimen@ 3
        \@tempdima\wd\tw@
        \advance\@tempdima-\scriptspace
        \divide\@tempdima 10
        \advance\dimen@-\@tempdima
        \ifdim\dimen@>\z@ \dimen@0pt\fi
        \rel@kern{0.6}\kern-\dimen@
        \if#31
        \overline{\rel@kern{-0.6}\kern\dimen@\macc@nucleus\rel@kern{0.4}\kern\dimen@}%
        \advance\dimen@0.4\dimexpr\macc@kerna
        \let\final@kern#2%
        \ifdim\dimen@<\z@ \let\final@kern1\fi
        \if\final@kern1 \kern-\dimen@\fi
        \else
        \overline{\rel@kern{-0.6}\kern\dimen@#1}%
        \fi
    }%
    \macc@depth\@ne
    \let\math@bgroup\@empty \let\math@egroup\macc@set@skewchar
    \mathsurround\z@ \frozen@everymath{\mathgroup\macc@group\relax}%
    \macc@set@skewchar\relax
    \let\mathaccentV\macc@nested@a
    \if#31
    \macc@nested@a\relax111{#1}%
    \else
    \def\gobble@till@marker##1\endmarker{}%
    \futurelet\first@char\gobble@till@marker#1\endmarker
    \ifcat\noexpand\first@char A\else
    \def\first@char{}%
    \fi
    \macc@nested@a\relax111{\first@char}%
    \fi
    \endgroup
    }
\makeatother

\makeatletter
\newcommand*{\redefinesymbolwitharg}[1]{%
  \expandafter\let\csname ltx#1\expandafter\endcsname\csname #1\endcsname
  \@namedef{#1}{\@ifnextchar{^}{\@nameuse{#1@}}{\@nameuse{#1@}^{}}}%
  \expandafter\def\csname #1@\endcsname^##1##2{%
     \csname ltx#1\endcsname\ifx!##1!\else^{##1}\fi\mathopen{}\mathclose\bgroup\left(##2\aftergroup\egroup\right)
     }%
}
\makeatother
\redefinesymbolwitharg{sin}
\redefinesymbolwitharg{cos}

\usepackage{url}
\usepackage{todonotes}

\newcommand*\rot{\rotatebox{90}}

\titlespacing*{\section}{0pt}{0pt}{0pt}
\titlespacing*{\subsection}{0pt}{0pt}{0pt}
\titlespacing*{\subsubsection}{0pt}{0pt}{0pt}

\setlist[itemize]{leftmargin=2em, before=\vspace{-0.5em}, after=\vspace{-0.5em}, itemsep=0.1em}

\usepackage{colortbl}
\definecolor{lightgray}{gray}{0.9}
\usepackage{array}
\newcolumntype{g}{>{\columncolor{lightgray}}c}

\title{\textsf{STanHop}: Sparse Tandem Hopfield Model for\\ Memory-Enhanced Time Series Prediction}

\author{
{\bf
Dennis Wu\thanks{These authors contributed equally to this work.}\;\,$^{\dagger}$
\quad
Jerry Yao-Chieh Hu$^{*\dagger}$
\quad
Weijian Li$^{*\dagger}$
\quad
Bo-Yu Chen$^{\ddag}$
\quad
Han Liu$^{\dagger\natural}$}
\\
\vspace{0.5em}
\begin{flushleft}
     $^\dagger$ Department of Computer Science, Northwestern University, Evanston, IL 60208 USA  \\
     $^\ddag$ Department of Physics, National Taiwan University, Taipei 10617, Taiwan \\
     $^\natural$  Department of Statistics and Data Science, Northwestern University, Evanston, IL 60208 USA
\end{flushleft}
{\footnotesize
   \texttt{\{\href{mailto:jhu@u.northwestern.edu}{jhu}, \href{mailto:hibb@u.northwestern.edu}{hibb}, \href{mailto:weijianli@u.northwestern.edu}{weijianli}\}@u.northwestern.edu,
   \href{mailto:b12202023@ntu.edu.tw}{b12202023@ntu.edu.tw}
   }
   }
   \\
{\footnotesize
   \texttt{
   \href{mailto:hanliu@northwestern.edu}{hanliu@northwestern.edu}}
   }
   \vspace{-2em}
}

\iclrfinalcopy %
\begin{document}

\maketitle

\begin{abstract}
We present \textbf{STanHop-Net} (\textbf{S}parse \textbf{Tan}dem \textbf{Hop}field \textbf{Net}work) for multivariate time series prediction with memory-enhanced capabilities.
At the heart of our approach is \textbf{STanHop}, a novel Hopfield-based neural network block, which sparsely learns and stores both temporal and cross-series representations in a data-dependent fashion.
In essence, STanHop sequentially learn temporal representation and cross-series representation using two tandem sparse Hopfield layers.
In addition, StanHop incorporates two additional external memory modules: a Plug-and-Play module and a Tune-and-Play module for train-less and task-aware  memory-enhancements, respectively.   
They allow StanHop-Net to swiftly respond to certain sudden events.  
Methodologically, we construct the StanHop-Net by stacking STanHop blocks in a hierarchical fashion, enabling multi-resolution feature extraction with resolution-specific sparsity. 
Theoretically, we introduce a sparse extension of the modern Hopfield model (Generalized Sparse Modern Hopfield Model) and show that it endows a tighter memory retrieval error compared to the dense counterpart without sacrificing memory capacity.
Empirically, we validate the efficacy of our framework on both synthetic and real-world settings.  
\blfootnote{Reproducible Code will be publicly available upon conference acceptance.}
\end{abstract}

\clearpage
{
\setlength{\parskip}{-0em}
\setcounter{tocdepth}{2}
\tableofcontents
}
\clearpage
\section{Introduction}
\label{sec:intro}

In this work, we aim to enhance multivariate time series prediction by incorporating relevant additional information specific to the inference task at hand.
This problem holds practical importance due to its wide range of real-world applications.
On one hand, multivariate time series prediction itself poses a unique challenge given its multi-dimensional sequential structure and noise-sensitivity \cite{masini2023machine,reneau2023feature,nie2022time,fawaz2019deep}. 
A proficient model should robustly not only discern the correlations between series within each time step, but also grasp the intricate dynamics of each series over time. 
On the other hand, in many real-world prediction tasks, one significant challenge with existing time series models is their slow responsiveness to sudden or rare events.
For instance, events like the 2008 financial crisis and the pandemic-induced market turmoil in 2021 \cite{laborda2021volatility,bond2021failing,sevim2014developing,bussiere2006towards}, or extreme climate changes in weather forecasting \cite{le2023climate,sheshadri2021midlatitude} often lead to compromised model performance.
To combat these challenges, 
we present \textbf{STanHop-Net} (\textbf{S}parse \textbf{Tan}dem \textbf{Hop}field \textbf{Net}work), a novel Hopfield-based deep learning model, for multivariate time series prediction, equipped with optional memory-enhanced capabilities.

Our motivation comes from the  connection between associative memory models of human brain (specifically, the modern Hopfield models) and the attention mechanism \cite{hu2023SparseHopfield,ramsauer2020hopfield}.
Based on this link, we propose to enhance time series models with external information (e.g., real-time or relevant auxiliary data) via the memory retrieval mechanism of Hopfield models. 
In its core, we utilize and extend the deep-learning-compatible Hopfield layers \cite{hu2023SparseHopfield,ramsauer2020hopfield}.
Differing from typical transformer-based architectures, these layers not only replace the attention mechanisms \cite{ramsauer2020hopfield,widrich2020modern} but also serve as differentiable memory modules, enabling integration of external stimuli for enhanced predictions.

In this regard, 
we first introduce a set of generalized sparse Hopfield layers, as an extension of the sparse modern Hopfield model \cite{hu2023SparseHopfield}.
Based on these layers, we propose a structure termed the \textbf{STanHop} (\textbf{S}parse \textbf{Tan}dem \textbf{Hop}field layers) block.
In STanHop, there are two sequentially joined sub-blocks of generalized sparse Hopfield layers, hence tandem.
This tandem design sparsely learn and store temporal and cross-series representations in a sequential manner. 

Furthermore, 
we introduce \textbf{STanHop-Net} (\textbf{S}parse \textbf{Tan}dem \textbf{Hop}field \textbf{Net}work) for time series,  consisting of multiple layers of STanHop blocks to cater for multi-resolution representation learning.
To be more specific, rather than relying only on the input sequence for predictions, each stacked StanHop block is capable of  incorporating additional information through the Hopfield models' memory retrieval mechanism from a pre-specified external memory set.
This capability facilitates the injection of external memory at every resolution level when necessary. 
Consequently, STanHop-Net not only excels at making accurate predictions but also allows users to integrate additional information they consider valuable for their specific downstream inference tasks with minimal effort.

We provide visual overviews of  STanHop-Net in \cref{fig:pipeline} and STanHop block in \cref{fig:StanHop}.

\paragraph{Contributions.} 
We summarize our contributions as follows:
\begin{itemize}
    \item 
    Theoretically, we introduce an sparse extension of the modern Hopfield model, termed the generalized sparse Hopfield model.
    We show that it not only offer a tighter memory retrieval error bound compared to the dense modern Hopfield model \cite{ramsauer2020hopfield}, but also retains the robust theoretical properties of the dense model, such as fast fixed point convergence and exponential memory capacity. 
    Moreover, it serves as a generalized model that encompasses both the sparse \cite{hu2023SparseHopfield} and dense \cite{ramsauer2020hopfield} models as its special cases.

    \item
    Computationally, 
    we show the one-step approximation of the retrieval dynamics of the generalized sparse Hopfield model is connected to sparse attention mechanisms, akin to \cite{hu2023SparseHopfield,ramsauer2020hopfield}.
    This connection allows us to introduce the $\mathtt{GSH}$ layers featuring learnable sparsity, for time series representation learning.
    As a result, these layers achieve faster memory-retrieval convergence and greater noise-robustness compared to the dense model.
    
    \item 
    Methodologically, 
    with $\mathtt{GSH}$ layer, we present \textbf{STanHop} (\textbf{S}parse \textbf{Tan}dem \textbf{Hop}field layers) block, a hierarchical tandem Hopfield model design to capture the intrinsic multi-resolution structure of both temporal and cross-series dimensions of time series with resolution-specific sparsity at each level.
    In addition, 
    we introduce the idea of pseudo-label retrieval, and debut two external memory plugin schemes --- Plug-and-Play and Tune-and-Play memory plugin modules --- for memory-enhanced predictions.

    \item
    Experimentally, we validate STanHop-Net in multivariate time series predictions, considering both with and without the incorporation of external memory.
    When external memory isn't utilized, STanHop-Net consistently matches or surpasses many popular baselines, including Crossformer \cite{zhang2022crossformer} and DLinear \cite{zeng2023transformers}, across diverse real-world datasets.
    When external memory is utilized, STanHop-Net demonstrates further performance boosts in many settings, benefiting from both proposed external memory schemes.
\end{itemize}

\paragraph{Notations.}
We write $\Braket{\ba,\bb}\coloneqq \ba^\sT \bb$ as the inner product for vectors $\ba,\bb$.
The index set $\{1,\cdots,I\}$ is denoted by $[I]$, where $I\in\mathbb{N}_+$. 
The spectral norm is denoted by $\norm{\cdot}$, which is equivalent to the $l_2$-norm when applied to a vector. 
Throughout this paper, we denote the memory patterns (keys) by $\bxi\in\R^d$ and the state/configuration/query pattern by $\bx\in\R^d$ with $n\coloneqq\norm{\bx}$, and $\bm{\Xi}\coloneqq\(\bxi_1,\cdots,\bxi_M\)\in \R^{d\times M}$ as shorthand for stored memoery (key) patterns $\{\bxi_\mu\}_{\mu\in[M]}$.
Moreover, we set $m\coloneqq \Max_{\mu\in[M]}\norm{\bxi_\mu}$ be the largest norm of memory patterns.

\paragraph{Note Added [December 27, 2023].}
After the completion of this work, the authors learn of an upcoming work by \citet{martins2023sparse}, addressing similar
topics from the perspective of the Fenchel-Young loss \cite{blondel2020learning}.

\section{Background: Modern Hopfield Models}
\label{sec:background}

Let $\bx \in \mathbb{R}^d$ be the query pattern and $\bm{\Xi} = (\bxi_1, \cdots, \bxi_M) \in \mathbb{R}^{d \times M}$ the $M$ memory patterns. 

\paragraph{Hopfield Models.}
Hopfield models are associative models that store a set of memory patterns $\bm{\Xi}$ in such a way that a stored pattern $\bxi_\mu$ can be retrieved based on a partially known or contaminated version, a query $\bx$.
The models achieve this by embedding the memories $\bm{\Xi}$ in the \textit{energy landscape} $E(\bx)$ of a physical system (e.g., the Ising model in \cite{hopfield1982neural} or its higher-order generalizations \cite{lee1986machine,peretto1986long,newman1988memory}), where each memory $\bxi_\mu$ corresponds to a local minimum. 
When a query $\bx$ is introduced, the model initiates energy-minimizing \textit{retrieval dynamics} $\calT$ at the query's location.
This process then navigate the energy landscape to locate the nearest local minimum $\bxi_\mu$, effectively retrieving the memory most similar to the query $\bx$.

Constructing the energy function, $E(\bx)$, is straightforward. 
As outlined in \cite{krotov2016dense}, memories get encoded into $E(\bx)$ using the \textit{overlap-construction}: $E(\bx) = F(\bm{\Xi}^{\sT}\bx)$, where $F: \mathbb{R}^M \to \mathbb{R}$ is a smooth function. 
This ensures that the memories $\{\bxi_\mu\}_{\mu \in [M]}$ sit at the stationary points of $E(\bx)$, given $\grad_{\bx}F(\bm{\Xi}^{\sT}\bx)|_{\bxi_{\mu}} = 0$ for all $\mu \in [M]$. 
The choice of $F$ results in different Hopfield model types, as demonstrated in \cite{krotov2016dense, demircigil2017model, ramsauer2020hopfield, krotov2020large}.
However, determining a suitable retrieval dynamics, $\calT$, for a given energy $E(\bx)$ is more challenging. 
For effective memory retrieval, $\calT$ must:
\begin{itemize}
    \item [(T1)] \label{item:T1} Monotonically reduce $E(\bx)$ when applied iteratively.
    \item [(T2)] \label{item:T2} Ensure its fixed points coincide with the stationary points of $E(\bx)$ for precise retrieval.
\end{itemize}
\paragraph{Modern Hopfield Models.}
\citet{ramsauer2020hopfield} propose the modern Hopfield model with a specific set of $E$ and $\calT$ satisfying above requirements, and integrate it into deep learning architectures via its strong connection with attention mechanism, offering enhanced performance, and theoretically guaranteed exponential memory capacity.
Specifically, they introduce
\begin{equation}
\label{eqn:MHM}
    E(\bx) = -\lse(\beta,\bm{\Xi}^\sT \bx) + \frac{1}{2} \langle \bx,\bx \rangle + \text{Const.},\;\;
    \text{and}\;\;
    \calT_{\text{Dense}}(\bx) = \bm{\Xi}  \Softmax(\beta \bm{\Xi}^\sT \bx)=\bx^{\text{new}},
\end{equation}
where $\bm{\Xi}^\sT \bx=(\Braket{\bxi_1,\bx},\ldots,\Braket{\bxi_M,\bx})\in\R^M$, $\lse\(\beta,\bz\)\coloneqq \log\(\sumM \exp{\beta z_\mu}\)/\beta$ is the log-sum-exponential for any given vector $\bz\in\R^M$ and $\beta>0$.
Their analysis concludes that:
\begin{itemize}
    \item  
    $\calT_{\text{Dense}}$ converges well \cite[Theorem~1,2]{ramsauer2020hopfield} and can retrieve patterns accurately in just one step \cite[Theorem~4]{ramsauer2020hopfield}, i.e.
    \hyperref[item:T1]{(T1)} and \hyperref[item:T2]{(T2)} are satisfied.
    \item  The modern Hopfield model \eqref{eqn:MHM} possesses an exponential memory capacity in pattern size $d$ \cite[Theorem~3]{ramsauer2020hopfield}.
    \item
    Notably, the one-step approximation of  $\calT_{\text{Dense}}$ mirrors the attention mechanism in transformers, leading to a novel deep architecture design: the Hopfield  layers.
\end{itemize}
In a related vein, \citet{hu2023SparseHopfield} introduce a principled approach to constructing modern Hopfield models using the convex conjugate of the entropy regularizer. 
Unlike the original modern Hopfield model \cite{ramsauer2020hopfield}, the key insight of \cite{hu2023SparseHopfield} is that the convex conjugate of various entropic regularizers can yield distributions with varying degrees of sparsity. 
Leveraging this understanding, we introduce the generalized sparse Hopfield model in the next section.

\section{Generalized Sparse Hopfield Model}
\label{sec:model}

In this section, 
we extend the entropic regularizer construction of the sparse modern Hopfield model \cite{hu2023SparseHopfield} by replacing the Gini entropic regularizer with the Tsallis $\alpha$-entropy \cite{tsallis1988possible},
\bea
\label{eqn:T_entropy}
\Psi_\alpha(\bp)\coloneqq
\begin{cases}
    \frac{1}{\alpha(\alpha-1)}\sum^M_{\mu=1}\(p_\mu-p_\mu^\alpha\), \quad & \alpha\neq 1,\\
    -\sumM p_\mu\ln p_\mu,\quad &\alpha=1,
\end{cases},\quad\text{for }\alpha\ge 1,
\eea
thereby introducing the generalized sparse Hopfield model. 
Subsequently, we verify the connection between the memory retrieval dynamics of the generalized sparse Hopfield model and attention mechanism. 
This leads to the Generalized Sparse Hopfield ($\mathtt{GSH}$) layers for deep learning.

\subsection{Energy Function, Retrieval Dynamics and Fundamental Limits}
Let $\bz,\bp\in\R^M$, and $\Delta^{M}\coloneqq\{\bp\in\R^M_+ \mid \sum_\mu^M p_\mu=1\}$ be the $(M-1)$-dimensional unit simplex.
\paragraph{Energy Function.} We introduce the generalized sparse Hopfield energy function\footnote{This energy function \eqref{eqn:H_entmax} is equivalent up to an additive constant.}:
\bea
\label{eqn:H_entmax}
\calH(\bx)=-\Psi^\star_\alpha\(\beta \bm{\bm{\Xi}}^\sT \bx\) +\half \Braket{\bx,\bx}+\text{Const.},\quad \text{with } \Psi^\star_\alpha(\bz) \coloneqq \int\dd\bz \alphaentmax(\bz),
\eea
 where $\alphaentmax(\cdot):\R^M\to \Delta^M$ is a finite-domain distribution map defined as follows.
\begin{definition} The variational form of $\alphaentmax$ is defined by the optimization problem
\bea
\label{eqn:entmax}
\alphaentmax(\bz) \coloneqq \argmax_{\bp \in \Delta^M} [\Braket{\bp,\bz}-\Psi_\alpha(\bp)],
\eea
where $\Psi_\alpha(\cdot)$ is the Tsallis entropic regularizer given by \eqref{eqn:T_entropy}.
See \cref{remark:closeform} for a closed form.
\end{definition}
\noindent
$\Psi_\alpha^\star(\bp)$ is the convex conjugate of the Tsallis entropic regularizer $\Psi_\alpha(\bp)$ (\cref{def:convex_conjugate}) and hence 
\begin{lemma}
\label{lemma:Danskin}
    $\grad \Psi^\star_\alpha(\bz)=\argmax_{\bp\in \Delta^M} [\Braket{\bp,\bz}-\Psi_\alpha(\bp)]=\alphaentmax(\bz) $.
\end{lemma}
\vspace{-0.5em}
\begin{proof}
See \cref{proof:Danskin} for a detailed proof.
\end{proof}\vspace{-1em}
\paragraph{Retrieval Dynamics.} With \cref{lemma:Danskin}, it is clear to see that the energy function \eqref{eqn:H_entmax} aligns with the overlap-function construction of Hopfield models, as in \cite{hu2023SparseHopfield,ramsauer2020hopfield}.
Next, we introduce the corresponding retrieval dynamics satisfying the monotinicity property \hyperref[item:T1]{(T1)}.
\begin{lemma}
[Generalized Sparse Hopfield Retrieval Dynamics]
\label{lemma:retrieval_dyn}
    Let $t$ be the iteration number.
The retrieval dynamics of the generalized sparse Hopfield model is a 1-step update of the form 
    \begin{equation}
    \label{eqn:retrieval_dyn}
       \calT(\bx_t)\coloneqq \grad_\bx \Psi^\star_\alpha\(\beta \bm{\Xi}^\sT \bx_t\) = \alphaentmax\( \beta \bm{\Xi}^\sT \bx_t \)=\bx_{t+1} ,
    \end{equation}
    that minimizes the energy function \eqref{eqn:H_entmax} monotonically over $t$.
\end{lemma}
\begin{proof}[Proof]
See \cref{proof:retrieval_dyn} for a detailed proof.
\end{proof}
\noindent
To see how this model store and retrieve memory patterns, we  first introduce the following definition.
\begin{definition} [Stored and Retrieved]
\label{def:stored_and_retrieved}
Assuming that every pattern $\bxi_\mu$ surrounded by a sphere $S_\mu$ with finite radius $R\coloneqq \half \Min_{\mu,\nu\in[M]}\norm{\bxi_\mu-\bxi_\nu}$, we say $\bxi_\mu$ is \textit{stored} if there exists a generalized fixed point of $\calT$, $\bx^\star_\mu \in S_\mu$, to which all limit points $\bx \in S_\mu$ converge to, and $S_\mu \cap S_\nu=\emptyset$ for $\mu \neq \nu$. 
We say $\bxi_\mu$ is $\epsilon$-\textit{retrieved} by $\calT$ with $\bx$ for an error $\epsilon$, if $\norm{\calT(\bx)-\bxi_\mu}\le \epsilon$.
\end{definition}
\noindent
To ensure the convergence property  \hyperref[item:T2]{(T2)} of retrieval dynamics \eqref{eqn:retrieval_dyn}, we have the next lemma.
\begin{lemma}[Convergence of Retrieval Dynamics $\calT$]
\label{lemma:convergence_sparse}
Given the energy function \eqref{eqn:H_entmax} and retrieval dynamics $\calT(\bx)$ \eqref{eqn:retrieval_dyn}, respectively.
For any sequence $\{\mathbf{x}_t\}_{t=0}^{\infty}$ generated by the iteration $\mathbf{x}_{t'+1} = \mathcal{T}(\mathbf{x}_{t'})$, all limit points of this sequence are stationary points of $\calH$.
\end{lemma}
\begin{proof}
See \cref{proof:convergence_sparse} for a detailed proof.
\end{proof}
Intuitively,  \cref{lemma:convergence_sparse} suggests that
for any query $\bx$, $\calT$ (given by \eqref{eqn:retrieval_dyn}) monotonically and iteratively approaches stationary points of $\calH$ (given by \eqref{eqn:H_entmax}), where the memory patterns $\{\bxi_\mu\}_{\mu\in[M]}$ are stored.
This completes the construction of a well-defined modern Hopfield model.

\paragraph{Fundamental Limits.}
To highlight the computational benefits of the generalized sparse Hopfield model, we analyze the fundamental limits of the memory retrieval error and memory capacity.
\begin{theorem}[Retrieval Error]
\label{thm:eps_sparse_dense}
Let $\calT_{\text{Dense}}$ be the retrieval dynamics of the dense modern Hopfield model \cite{ramsauer2020hopfield}.
Let $\bz\in\R^M$, $z_{(\nu)}$ be the $\nu$'th element in a sorted descending $z$-sequence $\bz_{\text{sorted}}\coloneqq z_{(1)}\ge \ldots\ge z_{(M)}$, and $\kappa(\bz)\coloneqq \Max\{k\in[M]\mid 1+kz_{(k)}>\sum_{\nu\le k} z_{(\nu)}\}$.
For all $\bx\in S_\mu$, it holds $\norm{\calT(\bx)-\bxi_\mu} 
\leq
\norm{\calT_{\text{Dense}}(\bx)-\bxi_\mu}$, and
\begin{align}
\label{eqn:eps_sparse_dense}
    &\text{for }2\ge\alpha\ge1,\quad\norm{\calT(\bx)-\bxi_\mu}\leq2m(M-1) \exp{-\beta \(\Braket{\bxi_\mu,\bx}-\Max_{\nu\in[M]}\Braket{\bxi_\mu,\bxi_\nu}\)},\\
\label{eqn:eps_sparse_over_2}
    &\text{for }\alpha\geq 2,\quad \norm{\calT(\bx)-\bxi_\mu}\leq m+d^{\nicefrac{1}{2}}m\beta \[\kappa \(\Max_{\nu\in[M]}\Braket{\bxi_\nu,\bx}-\[ \bm{\Xi}^\sT \bx\]_{(\kappa)}\)+\frac{1}{\beta}\].
\end{align}
\end{theorem}
\begin{corollary}[Noise-Robustness]
    In cases of noisy patterns with noise $\bm{\eta}$, i.e. $\Tilde{\bx}=\bx+\bm{\eta}$ (noise in query) or $\Tilde{\bxi}_\mu=\bxi_\mu+\bm{\eta}$ (noise in memory), the impact of noise $\bm{\eta}$ on the sparse retrieval error $\norm{\calT(\bx)-\bxi_\mu}$ is linear for  $\alpha\geq2$, while its effect on the dense retrieval error $\norm{\calT_{\text{Dense}}(\bx)-\bxi_\mu}$ (or $\norm{\calT(\bx)-\bxi_\mu}$ with $2\geq\alpha\geq1$) is exponential.
\end{corollary}
\begin{proof}
    See \cref{proof:eps_sparse_dense} for a detailed proof.
\end{proof}
Intuitively, \cref{thm:eps_sparse_dense} implies the sparse model converge faster to memory patterns than the dense model \cite{ramsauer2020hopfield}, and the larger sparsity leads the lower retrieval error.

\begin{lemma}[Memory Capacity Lower Bound]
\label{thm:memory_capacity}
Suppose the probability of successfully storing and retrieving memory pattern  is given by $1-p$.
The number of memory patterns sampled from a sphere of radius $m$ that the sparse Hopfield model can store and retrieve has a lower bound:
$M \ge \sqrt{p}C^\frac{d-1}{4}$,
where $C$ is the solution for $C=\nicefrac{b}{W_0({\exp\{a+\ln{b}\}})}$ with $W_0(\cdot)$ being the principal branch of Lambert $W$ function \cite{olver2010nist}, $a\coloneqq \nicefrac{4}{d-1}\big\{\ln\[\nicefrac{2m(\sqrt{p}-1)}{(R+\delta)}\]+1\big\}$ and $b\coloneqq \nicefrac{4m^2\beta}{5(d-1)}$.
For sufficiently large $\beta$, the sparse Hopfield model has a larger lower bound on the exponential-in-d memory capacity compared to that of dense counterpart \cite{ramsauer2020hopfield}:
$
M\ge M_{\text{Dense}}.
$
\end{lemma}
\begin{proof}
See \cref{proof:memory_capacity} for a detailed proof.
\end{proof}
\Cref{thm:memory_capacity} offers a lower bound on the count of patterns effectively stored and retrievable by $\calT$ with a minimum precision of $R$, as defined in \cref{def:stored_and_retrieved}. 
Essentially, the capacity of the generalized sparse Hopfield model to store and retrieve patterns grows exponentially with pattern size $d$. This mirrors findings in \cite{hu2023SparseHopfield,ramsauer2020hopfield}.
Notably, when $\alpha=2$, the results of \cref{thm:eps_sparse_dense} and \cref{thm:memory_capacity} reduce to those of \cite{hu2023SparseHopfield}.

\subsection{Generalized Sparse Hopfield  (GSH) Layers for Deep Learning} 
\vspace{-0.3em}
Now we introduce the Generalized Sparse Hopfield (GSH) layers for deep learning, by drawing the connection between the generalized sparse Hopfield model and attention mechanism.

\paragraph{Generalized Sparse Hopfield ($\mathtt{GSH}$) Layer.}
Following \cite{hu2023SparseHopfield}, we extend \eqref{eqn:retrieval_dyn} to multiple queries $\bX\coloneqq\{\bx_i\}_{i\in[T]}$.
From previous section, we say that the Hopfield model, as defined by \eqref{eqn:H_entmax} and \eqref{eqn:retrieval_dyn}, functions within the associative spaces $\bX$ and $\bm{\Xi}$.
Given any \textit{raw} query ${\bR}$ and memory $\bY$ that are input into the Hopfield model\footnote{The \textit{raw} query ${\bR}$ and memory $\bY$ may originate from data, external sets, or hidden representations throughout a given deep learning pipeline.
They are not necessarily
usable as $\bX$ and $\bm{\Xi}$. 
Therefore, to use \eqref{eqn:retrieval_dyn}, they must be mapped into $d$-dimensional associative spaces.}, we compute $\bX$ and $\bm{\Xi}$ as $\bX^\sT={\bR}\bW_Q\coloneqq \bQ$ and $\bm{\Xi}^\sT=\bY \bW_K \coloneqq \bK$, using matrices $\bW_Q$ and $\bW_K$.
Therefore, we rewrite $\calT$ in 
\eqref{eqn:retrieval_dyn} as
$\(\bQ^{\text{new}}\)^\sT = \bK^\sT \alphaentmax\(\beta \bK \bQ^\sT\)$.
Taking transpose 
and projecting $\bK$ to $\bV$ with $\bW_V$, we have
\begin{align}
    \bZ\coloneqq \bQ^{\text{new}} \bW_V=\alphaentmax\(\beta  \bQ\bK^\sT\)\bK \bW_V
    =\alphaentmax\(\beta  \bQ\bK^\sT\)\bV,
\end{align}
which leads to the attention mechanism with $\alphaentmax$ activation function.
Plugging back the raw patterns $\bR$ and $\bY$, 
we arrive the foundation of the Generalized Sparse Hopfield ($\mathtt{GSH}$) layer,
\bea
\label{eqn:GSH}
\mathtt{GSH}\(\bR,\bY\)=\bZ= \alphaentmax\(\beta  {\bR}\bW_Q \bW_K^\sT\bY^\sT\)\bY\bW_K \bW_V.
\eea
By \eqref{eqn:eps_sparse_dense}, $\calT$ retrieves memory patterns with high accuracy after a single activation.
This allows \eqref{eqn:GSH} to integrate with deep learning architectures just like \cite{hu2023SparseHopfield,ramsauer2020hopfield}.
\begin{remark}
    $\alpha$ is a learnable parameter \cite{correia2019adaptively}, enabling $\mathtt{GSH}$ to learn input sparsity.
\end{remark}

\paragraph{$\mathtt{GSHPooling}$ and $\mathtt{GSHLayer}$ Layers.}
Following \cite{hu2023SparseHopfield}, we introduce two more variants: the
$\mathtt{GHSPooling}$ and  $\mathtt{GSHLayer}$ layers.
They are similar to the $\mathtt{GSH}$, and only differ in how to obtain the associative sets $\bQ, \bY$.
For $\mathtt{GHSPooling}\(\bY\)$, $\bK = \bY \bW_K, \bV = \bK \bW_V$, and $\bQ$ is a learnable variable independent from any input.
For $\mathtt{GSHLayer}\(\bR, \bY\)$, we have $\bK = \bV = \bY$, and $\bQ = \bR$.
Note that $\mathtt{GSHLayer}$ can have $\bQ$ as learnable parameter or as an input.
Where if $\bQ$ was served as an input, the whole $\mathtt{GSHLayer}$ has no learnable parameters and can be used as a lookup table. 
We provide an example of memory retrieval for image completion using $\mathtt{GSHLayer}$ in \cref{sec:example_retrieval}.

\section{STanHop-Net: Sparse Tandem Hopfield Networt}
\label{sec:method}

\begin{figure*}[t]

    \centering
    \includegraphics[width=\textwidth]{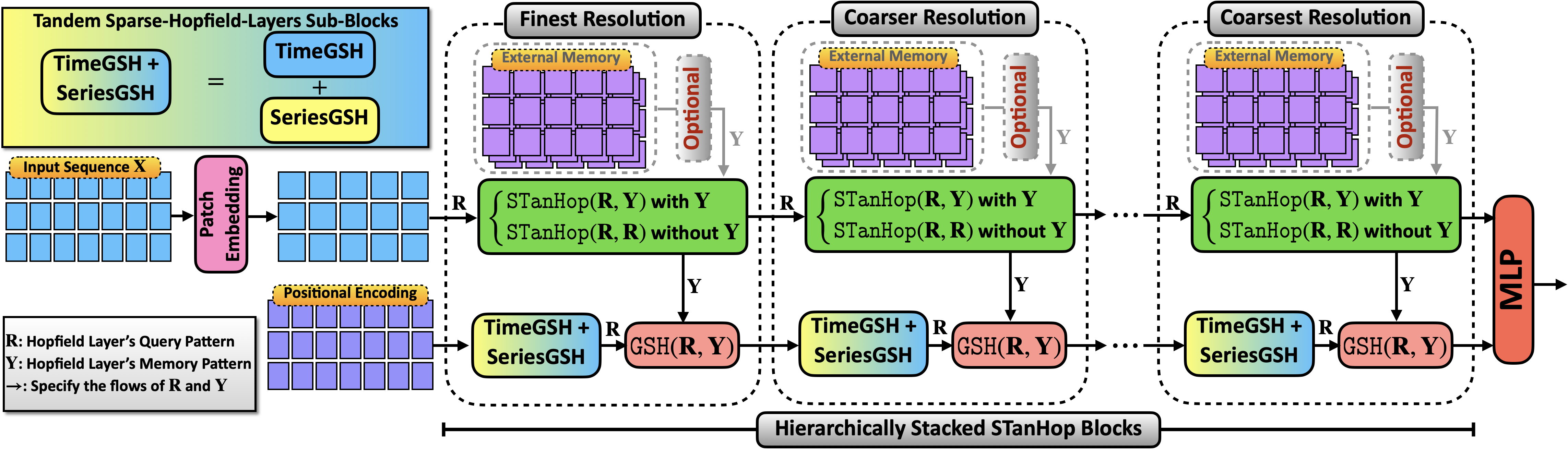}

    \caption{\small
    \textbf{STanHop-Net Overview.}
    \textbf{Patch Embedding:}
    Given an input multivariate time series $\bX\in\R^{ C\times T \times d}$ consisting $C$ univariate series, $T$ time steps and $d$ features, the patch embedding aggregates temporal information for each univariate series, subsequently reducing temporal dimensionality from $T$ to $P =T/P$ for all $d$ features. 
    \textbf{STanHop Block:}
    The STanHop block leverages the Generalized Sparse Hopfield (GSH) model (\cref{sec:model}). 
    It captures time series representations from its input through two tandem sparse-Hopfield-layers sub-blocks (i.e. TimeGSH and SeriesGSH, see \cref{fig:StanHop}), catering to both temporal and cross-series dimensions. 
    \textbf{STanHop-Net:}
    Using a stacked encoder-decoder structure, STanHop-Net facilitates hierarchical multi-resolution learning. 
    This design allows STanHop-Net to extract  distill representations from both temporal and cross-series dimensions across multiple scales (multi-resolution in a hardwired fashion via coarse-graining layers, see \cref{sec:coarse}). 
    Moreover, each stacked block has optional external memory plugin functionalities for enhanced predictions (\cref{sec:memoryplugin}).
    These representations from all resolutions are then merged, providing a holistic  representation learning for downstream predictions specially tailored for time series data.}
    \label{fig:pipeline}
\end{figure*}

In this section, we introduce a Hopfield-based deep architecture (STanHop-Net) tailored for memory-enhanced learning of noisy multivariate time series. 
These additional memory-enhanced functionalities enable STanHop-Net to effectively handle the problem of slow response to sudden or rare events (e.g, 2021 pandemic meltdown in financial market) by making predictions using both in-context inputs (e.g., historical data) and external stimuli (e.g., real-time or relevant past data). 
In the following, we consider multivariate time series $\bX\in\R^{ C\times T \times d}$ comprised of $C$ univariate series. 
Each univariate series has $T$ time steps and $d$ features.
\subsection{Patched Embedding}

Motivated by \cite{zhang2023crossformer}, we use a patching technique on model input that groups adjacent time steps into subseries patches.
This method extends the input time horizon without altering token length, enabling us to capture local semantics and critical information more effectively, which is often missed at the point-level. 
We define the multivariate input sequence as $\bX\in\R^{C\times T\times d}$, where $C, T, d$ denotes the number of variates, number of time steps and the number of dimensions of each variate.
Given a time series sequence $\bX =\{\bx_1, ..., \bx_T\}$ and a patch size $P$, the patching operation divides $X$ into $\bS = \{ \bs_1, ..., \bs_{T/P} \}$.
For each patched sequence $\bs_i \in \mathbb{R}^{C\times P\times d}$ for $i\in[T/P]$, we define the patched embedding as 
$\text{EMB} \( \bs_i \) = \mathbf{E}^{\text{emb}} \bs_i + \mathbf{E}^{\text{pos}} \( i \) \in \R^{D_{\text{emb}}}$,
where  $D_{\text{emb}}$ is the embedding dimension, $\mathbf{E}^{\text{emb}} \in \mathbb{R}^{D_{\text{emb}} \times P}$, and $\mathbf{E}^{\text{pos}} \in \mathbb{R}^{T/P \times D_{\text{emb}}}$ is the positional encoding.
When $T$ is not divisible to $P$, assuming $T = P \times C_{n} + c$ with $C_n,c\in \mathbb{N}_+$, we pad the sequence by repeating the first $c$ elements in the sequence.
Consequently, this patching embedding significantly improves computational efficiency and memory usage.

\subsection{STanHop: Sparse Tandem Hopfield Block}
\label{sec:STanHop}

We introduce STanHop (\textbf{S}parse \textbf{Tan}dem \textbf{Hop}field) block which comprises one $\mathtt{GSHLayer}$-based external memory plugin module, and  two tandem sub-blocks of $\mathtt{GSH}$ layers to process both time and series dimensions, i.e. TimeGSH and SeriesGSH sub-blocks in \cref{fig:StanHop}.
In essence, STanHop not only sequentially extracts temporal and cross-series information of multivariate time series with (learnable) data-dependent sparsity, but also utilizes both acquired (in-context) representations and external stimulus through the memory plugin modules for the downstream prediction tasks.

Given a hidden vector, $\bR \in \mathbb{R}^{ C \times T \times D_{\text{hidden}}}$, and its corresponding external memory set $\bY \in \mathbb{R}^{M \times C \times T \times D_{\text{hidden}}}$, where $C$ denotes the channel number and $T$ denotes the number of time segments (patched time steps),
To clarify, the $\mathtt{GSH}$ layer only operates on the last two dimensions, i.e., $t\in T$ and $d\in D_{\text{hidden}}$. 
Thus, the operation $\mathtt{GSH} \( \bZ, \bZ \)$ extracts information of the temporal dynamics of $\bZ$ from the segmented time series.
Here we define the dimensional transpose operation $\mathsf{T}$.
For a given tensor $\bX \in \mathbb{R}^{a \times b \times c}$, we have
$ \mathsf{T}_{abc}^{acb}(\bX)\coloneqq \bX^\prime \in \mathbb{R}^{a \times c \times b}$, i.e. this operation rearranges the dimensions of the original tensor $\bX$ from $(a, b, c)$ to a new order $(a, c, b)$.
Given a set of query pattern $ \bQ \in \R^{ \text{len}_Q \times D_{\text{hidden}}}$, we define a single block of STanHop as
\begin{align*}
    \bZ &= \mathtt{Memory}\( \bR, \bY \), \tag{Memory Plugin Module, see \cref{sec:memoryplugin}}\\
    \bZ^t &=  \mathsf{T}^{cth}_{tch} \( {\rm{LayerNorm} \(\bZ + \text{\rm{FF}} \( \mathtt{GSH}(\bZ, \bZ) \) \)} \) \in \R^{T \times C \times D_\text{hidden}} \tag{Temporal $\mathtt{GSH}$}, \\
    \bZ^p &= \mathtt{GSHPooling} \(  \bR^\star, \bZ^t \) \in \R^{T \times \text{len}_Q \times D_{\text{hidden}}},  \tag{$\bR^\star$ is learnable and randomly initialized} \\
    \bZ^c &= \mathtt{GSH} \(  \bZ^t, \bZ^p \)  \in \R^{T \times C \times D_\text{hidden}} \tag{Cross-series  $\mathtt{GSH}$}, \\
    \bZ^* &=   \rm{LayerNorm} \( \bZ^t + \rm{FF} \( \bZ^c \)\) \in \R^{T \times C \times D_\text{hidden}},\\
    \bZ_{\text{out}} &= \rm{LayerNorm} \( \bZ^* + \rm{FF}(\bZ^*) \) \in \R^{T \times C \times D_\text{hidden}},
\end{align*}
where $\mathtt{Memory}(\cdot,\cdot)$ is the external memory plugin module introduced in the next section.
Note that, if we choose to turn off the external memory functionalities (or external memory is not available) during training, we set $\bY = \bR$ such that $\mathtt{Memory}(\bR,\bR)=\bR$ (see \cref{sec:memoryplugin} for details).
Here $\mathtt{GSHPooling}(\bR^\star,\bZ^t)$ takes $\bZ^t$ and a randomly initialized query $\bR^\star$ as input.
Importantly, $\bR^\star$ not only acts as learnable prototype patterns learned by pooling over $\bZ^t$, but also as a knob to control the computational complexity by picking the hidden dimension of $\bR^\star$.
We summarize the STanHop block as 
    $\bZ_{\text{out}} = \mathtt{STanHop} \( \bR, \bY \) \in \R^{T \times C \times D_{\text{hidden}}}$.

\begin{figure*}[t]
    \centering
    \includegraphics[width=\textwidth]{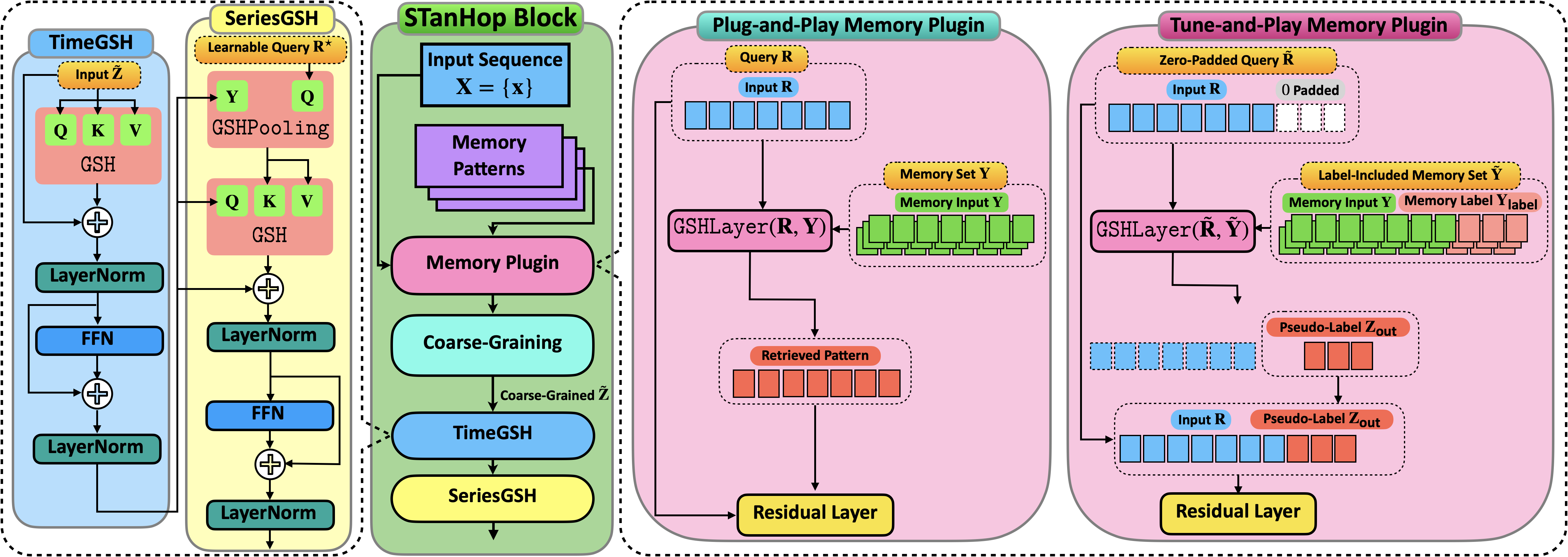}

    \caption{\small
    \textbf{STanHop Block.}
    \textbf{(Left)} Tandem Hopfield-Layer Blocks: TimeGSH and SeriesGSH. 
    Notably, in the $\mathtt{GSHPooling}$ block of SeriesGSH, the learnable query $\bR^\star$ is initialized randomly and employed to store learned prototype patterns from temporal representations extracted during  training. 
    \textbf{(Right)} Plug-and-Play and Tune-and-Play Memory Plugins.}
    \label{fig:StanHop}
\end{figure*}
\subsection{External Memory Plugin Module and Pseudo-Label Retrieval}
\label{sec:memoryplugin} 

Here we introduce the external memory modules (i.e., $\mathtt{Memory}(\cdot,\cdot)$  in \cref{sec:STanHop} or Memory Plugin blocks in \cref{fig:StanHop}) for external memory functionalities.
These modules are tailored for time series modeling by incorporating task-specific 
 supplemental information (such as relevant historical data for sudden or rare events predictions) for subsequent inference.
To this end, we introduce two memory plugin modules:  \textbf{Plug-and-Play Memory Plugin} and \textbf{Tune-and-Play Memory Plugin}.
For query $\bR$ and memory $\bY$, we denote them by $\mathtt{PlugMemory} ( \bR, \bY )$ and $\mathtt{TuneMemory} ( \bR,\bY )$.

    \paragraph{Plug-and-Play Memory Plugin.}
    This module enables performance enhancement utilizing external memory without any fine-tuning.
    Given a trained STanHop-Net (without external memory), we use a parameter fixed $\mathtt{GSHLayer}$ for memory retrieval.
    Explicitly, given an input sequence $\bR \in \R^{|R| \times D_{\text{hidden}}}$ and a corresponding external memory set $\bY \in \R^{M \times |R| \times d}$, where $|\bR|$ and  $D_{\text{hidden}}$ are the sequence length and hidden dimension of $\bR$ respectively.  
    We define the memory retrieval operation as $\bZ = \mathtt{PlugMemory}\( \bR, \bY \)= \text{LayerNorm}\( \bR + \mathtt{GSHLayer}\( \bR, \bY \) \)$ with all parameters fixed. 

    \paragraph{Tune-and-Play Memory Plugin.} 
    Here we propose the idea of ``pseudo-label retrieval" using $\mathtt{GSHLayer}$ for time series prediction.
    Specifically, we use modern Hopfield models' memory retrieval mechanism to generate pseudo-labels for a given $\bR$ from a \textit{label-included} memory set $\Tilde{\bY}$, thereby enhancing predictions.
    Intuitively, this method supplements predictions by learning from demonstrations and we use the retrieved pseudo-labels (i.e., learned \textit{pseudo}-predictions) as additional features.
    An illustration of this mechanism is shown in \cref{fig:StanHop}.
    Firstly, we prepare the \textit{label-included} external memory as $\Tilde{\bY} = \bY \oplus \bY_{\text{label}}$, where $\Tilde{\bY}$ is the concatenation of memory sequences and their corresponding labels.
    Next, we denote the padded $\bR$ as $\Tilde{\bR}$, where $\Tilde{\bR} \in \R^{|\Tilde{\bY}| \times d}$.
    And we utilize the $\mathtt{GSHLayer}$ to retrieve the pseudo-label from the memory sequences as $\bZ_{\text{out}}$.
    Then we concatenate $\bR$ and the pseudo-label $\bZ_{\text{out}}$ and send it to a feed forward layer to encode the pseudo-label information: $\bZ_{\text{out}} = \mathtt{GSHLayer}( \Tilde{\bR}, \Tilde{\bY} )$, $\bZ_{\text{pseudo}} = \bR \oplus \bZ_{\text{out}}$ and then $\Tilde{\bZ} = \text{LayerNorm} \( \text{FF}\( \bZ_{\text{pseudo}} \) + \bZ_{\text{pseudo}} \)$.
    In other words, we first obtain a weight matrix from the association between $\Tilde{\bR}$ and $\Tilde{\bY}$, and then multiply this weight matrix with $\bY_{\text{label}}$ to obtain $\bZ_{\text{out}}$.
    We summarize the Tune-and-Play memory plugin as $\Tilde{\bZ} = \mathtt{TuneMemory}( {\bR}, {\bY} )$.

\subsection{Coarse-Graining}
\label{sec:coarse}

To cope with the intrinsic multi-resolution inductive bias of time series,  we introduce a coarse-graining layer in each  STanHop block.
Given an hidden vector output, $\bZ \in \mathbb{R}^{C \times T \times D_{\text{hidden}}}$, grain level $\Delta$, and a weight matrix $\mathbf{W} \in \R^{ D_{\text{hidden}} \times  2  D_{\text{hidden}}}$, and $\oplus$ denotes the concatenation operation.
We denote $\bZ_{c,t,d}$ with $c\in[C],t\in[T],d\in[D_{\text{hidden}}]$ as the element representing the $c$-th series, $t$-th time segment, and $d$-th dimension.
The coarse-graining layer consists a vector concatenation and a matrix multiplication:
   $\hat{ \bZ }_{c,t, :} = \bZ_{c,t, :} \oplus \bZ_{c, t+\Delta , :} \in \R^{2D_{\text{hidden}}}$ and then $
    \Tilde{\bZ}_{c,t, :} = \mathbf{W} \hat{\bZ}_{c,t, :} \in \R^{D_{\text{hidden}}}$,
such that $\hat{\bZ}\in \R^{C\times T\times 2D_{\text{hidden}}}$ and $\Tilde{\bZ}\in \R^{C\times T\times D_{\text{hidden}}}$, similar to \cite{liu2021swin, zhang2023crossformer}.
Operationally, it first obtains the representation of smaller resolution, and then distills  information via a linear transformation.
We express this course-graining layer as 
$\rm{CoarseGrain}\( \bZ, \Delta \) = \Tilde{\bZ}$.

\subsection{Multi-Layer STanHop for Multi-Resolution Learning}

Finally, we construct the STanHop-Net by stacking STanHop blocks in a hierarchical fashion, enabling multi-resolution feature extraction with resolution-specific sparsity.
Given a prediction window size $P \in \R$, number of layer $L \in \R$, and a learnable positional embedding for the decoder $\bE_{\text{dec}}$, we construct our multi-layer STanHop as an autoencoder structure.
The encoder structure consists of a course-graining operation first, following by an $\mathtt{STanHop}$ layer.
The decoder follows the similar structure as the standard transformer decoder \cite{vaswani2017attention}, but we replace the cross-attention mechanism to a $\mathtt{GSH}$ layer, and self-attention layer as $\mathtt{STanHop}$ layer.
We summarize the STanHop-Net network structure in \cref{fig:pipeline}, and in \cref{alg:stacked_STanHop} in appendix.

\section{Experimental Studies}
\label{sec:exp}

We demonstrate the validity of STanHop-Net and external memory modules by testing them on various experimental settings with both synthetic and real-world datasets.

\subsection{Multivariate Time Series Prediction without external memory}
\label{sec:MTS}

\cref{table:result} includes the experiment results of the multivariate time series predictions using STanHop-Net without external memory.
We implement three variants of STanHop-Net: \textbf{StanHop-Net}, \textbf{StanHop-Net (D)} and \textbf{StanHop-Net (S)}, with $\mathtt{GSH}$, $\mathtt{Hopfield}$ \cite{ramsauer2020hopfield} and $\mathtt{SparseHopfield}$ \cite{hu2023SparseHopfield} layers respectively. 
Our results show that in 47 out of 58 cases, STanHop-Nets rank in the top two, delivering top-tier performance compared to all baselines.

\textbf{Data.}
Following  \cite{zhang2023crossformer, zhou2022fedformer, wu2021autoformer}, we use 6 realistic datasets:
ETTh1 (Electricity Transformer Temperature-hourly), ETTm1
(Electricity Transformer Temperature-minutely), WTH (Weather),
ECL
(Electricity Consuming Load),
ILI
(Influenza-Like Illness),
Traffic.
The first four datasets are split into train/val/test ratios of 14/5/5, and the last two are split into 7/1/2.
\textbf{Metrics.} We use Mean Square Error (MSE) and Mean Absolute Error (MAE) as accuracy metrics.
\textbf{Setup.}
Here we use the same setting as in \cite{zhang2022crossformer}: multivariate time series predictions tasks on 6 real-world datasets.
For each dataset, we evaluate our models with several different prediction horizons.
For all experiments, we report the mean MSE, MAE over 10 runs.
\textbf{Baselines.}
We benchmark our method against 5 leading methods listed in \cref{table:result}. 
Baseline results are quoted from competing papers when possible and reproduced otherwise.
\textbf{Hyperparameters.}
For each experiment, we optimize the hyperparameters using the ``sweep" function from Weights and Biases \cite{biewald2020experiment}. We conduct 100 random search iterations for each setting, selecting the best set based on the validation performance.

For datasets, hyperparameter tuning, implementations and training details, please see \cref{sec:exp_detail}.

\begin{table}[h!]
\centering
\caption{\small
\textbf{Accuracy Comparison for Multivariate Time Series Predictions without External Memory.}
We implement 3 STanHop variants, \textbf{STanHop-Net (D)} with \textbf{D}ense $\mathtt{Hopfield}$ layer \cite{ramsauer2020hopfield}, \textbf{STanHop-Net (S)} with \textbf{S}parse $\mathtt{SparseHopfield}$ layer \cite{hu2023SparseHopfield} and \textbf{STanHop-Net} with our $\mathtt{GSH}$ layer respectively. 
We report the average Mean Square Error (MSE) and Mean Absolute Error (MAE) metrics of 10 runs, with variance omitted as they are all $\le 0.44$\%. 
We benchmark our method against leading transformer-based methods (FEDformer \cite{zhou2022fedformer}, Informer \cite{zhou2021informer} and Autoformer \cite{wu2021autoformer}, Crossformer \cite{zhang2022crossformer}) and a linear model with seasonal-trend decomposition (DLinear \cite{zeng2023transformers}). 
We evaluate each dataset with different prediction horizons (showed in the second column).
We have the best results \textbf{bolded} and the second best results \underline{underlined}.
In 47 out of 58 settings, STanHop-Nets rank either first or second.
Our results indicate that our proposed STanHop-Net delivers consistent top-tier performance compared to all the baselines, even without external memory.
}
\resizebox{\textwidth}{!}{    
\begin{tabular}{ccccccccccccccccccccc}
\toprule
 \multicolumn{2}{c}{Models} &  \multicolumn{2}{c}{FEDFormer} &
 \multicolumn{2}{c}{DLinear} &
 \multicolumn{2}{c}{Informer} & \multicolumn{2}{c}{Autoformer} & \multicolumn{2}{c}{Crossformer} &
 \multicolumn{2}{c}{\footnotesize\textbf{STanHop-Net (D)}} &
 \multicolumn{2}{c}{\footnotesize\textbf{STanHop-Net (S)}} &
 \multicolumn{2}{c}{\textbf{STanHop-Net}} 
 \\
\midrule
 \multicolumn{2}{c}{Metric} & MSE & MAE & MSE & MAE & MSE & MAE & MSE & MAE & MSE & MAE & MSE & MAE  & MSE & MAE & MSE & MAE  \\
\midrule
 \multirow{6}{1em}{\rot{ETTh1}} & 24 & 0.318 & 0.384 & 0.312 & \underline{0.355} & 0.577 & 0.549 & 0.439 & 0.440 & 0.305 & 0.367 & 0.301 & 0.363 & \underline{0.298} & {\underline{0.360}} & \textbf{0.294}  &  {\textbf{0.351}} \\ 
& 48 & \underline{0.342} & 0.396 & {0.352} & \textbf{0.383} & 0.685 & 0.625 & 0.429 & 0.442 & 0.352 & 0.394 & {0.356} & 0.406 & 0.355 & {0.399} & \textbf{0.340}  & \underline{0.387}  \\ 
 & 168 & {0.412} & 0.449 & {0.416} & \textbf{0.430} & 0.931 & 0.752 & 0.493 & 0.479 & \underline{0.410} & 0.441  & \textbf{0.398} & 0.440 & 0.419 & 0.458  & {\textbf{0.398}} & \underline{0.437} \\ 
& 336 & 0.456 & 0.474 & \underline{0.450} & \textbf{0.452} & 1.128 & 0.873 & 0.509 & 0.492 & \textbf{0.440} & \underline{0.461} & 0.458 & 0.472 & 0.484 & 0.484 & \underline{0.450} & 0.472 \\ 
& 720 & 0.521 & 0.515 & \textbf{0.484} & \textbf{0.501} & 1.215 & 0.896 & 0.539 & 0.537 & 0.519 & 0.524 & 0.516 & 0.522 & 0.541 & 0.533 & \underline{ 0.512} & {\underline{0.511}} \\ 
\midrule
 \multirow{6}{1em}{\rot{ETTm1}} & 24 & 0.290 & 0.364 & 0.217 & 0.289 & 0.323 & 0.369 & 0.410 & 0.428 & 0.211 & 0.293 & 0.205 & 0.278 & \textbf{0.191} & \textbf{0.270} & \underline{0.195} & \underline{0.273} \\ 
& 48 & 0.342 & 0.396 & \underline{0.278} & \textbf{0.330} & 0.494 & 0.503 & 0.483 & 0.464 & 0.300 & 0.352 & 0.303 & 0.340 & 0.293 & 0.341 & \textbf{0.270} & \underline{0.333} \\ 
 & 96 & 0.366 & 0.412 & \underline{0.310} & \underline{0.354} & 0.678 & 0.614 & 0.502 & 0.476 & 0.320 & 0.373 & 0.325 & 0.377 & 0.322 & 0.362 & \textbf{0.286} & {\textbf{0.352}} \\ 
& 288 & 0.398 & 0.433 & \textbf{0.369} & \textbf{0.386} & 1.056 & 0.786 & 0.604 & 0.522 & 0.404 & 0.427 & 0.410 & 0.429 & 0.395 & 0.413 & \underline{0.373} & \underline{0.405} \\ 
& 672 & 0.455 & 0.464 & \underline{0.416} & \textbf{0.417} & 1.192 & 0.926 & 0.607 & 0.530 & 0.569 & 0.528 & 0.574 & 0.516 & 0.556 & 0.510 & \textbf{0.400} & \underline{0.460} \\ 
\midrule
 \multirow{6}{1em}{\rot{ECL}} & 48 & 0.229 & 0.338 & \underline{0.155} & \underline{0.258} & 0.344 & 0.393 & 0.241 & 0.351 & 0.156 & 0.255 & 0.159 & 0.264 & {0.170} & {0.273} & {\textbf{0.152}} &{\textbf{0.252}} \\ 
& 168 & {0.263} & {0.361} & \textbf{0.195} & \textbf{0.287} & 0.368 & 0.424 & 0.299 & 0.387 & 0.231 & 0.309 & 0.296 & 0.368 & {0.288} & {0.373} & {\underline{0.227}} & {\underline{0.304}} \\ 
 & 336 & \underline{0.305} & 0.386 & \textbf{0.238} & \textbf{0.316} & 0.381 & 0.431 & 0.375 & 0.428 & 0.323 & \underline{0.369} & 0.326 & 0.374 & {0.317} & {0.375} & {{0.317}} & {\underline{0.369}} \\ 
& 720 & \underline{0.372} & 0.434 & \textbf{0.272} & \textbf{0.346} & 0.406 & 0.443 & {0.377} & 0.434 & 0.404 & \underline{0.423} & 0.412 & 0.428 & 0.440 & 0.450 & 0.435  & 0.447 \\ 
& 960 & 0.393 & 0.449 & \textbf{0.299} & \textbf{0.367} & 0.460 & 0.548 & \underline{0.366}& \underline{0.426} & 0.433 & 0.438 & 0.446 & 0.447 & 0.467 & 0.463 & 0.443 & 0.446 \\ 
\midrule
 \multirow{6}{1em}{\rot{WTH}} & 24 & 0.357 & 0.412 & 0.357 & 0.391 & 0.335 & 0.381 & 0.363 & 0.396 & \underline{0.294} & \underline{0.343} & 0.304 & 0.351 & 0.303 & 0.352 & \textbf{0.292} & \textbf{0.341} \\ 
& 48 & 0.428 & 0.458 & 0.425 & 0.444 & 0.395 & 0.459 & 0.456 & 0.462 & \underline{0.370} & \underline{0.411} & 0.374 & 0.411 & 0.372 & 0.411 & \textbf{0.363} & \textbf{0.402} \\ 
& 168 & 0.564 & 0.541 & 0.516 & 0.516 & 0.608 & 0.567 & 0.574 & 0.548 &  \underline{0.473} & \underline{0.494} & 0.480 & 0.501 & 0.496 & 0.511 & \textbf{0.332} & \textbf{0.393} \\ 
 & 336 & 0.533 & 0.536 & 0.536 & 0.537 & 0.702 & 0.620 & 0.600 & 0.571 & \textbf{0.495} & \textbf{0.515} & 0.507 & 0.526 & 0.514 & 0.530 & \underline{0.499} & \textbf{0.515} \\ 
& 720 & 0.562 & 0.557 & 0.582 & 0.571 & 0.831 & 0.731 & 0.587 & 0.570 & \textbf{0.526} & \textbf{0.542} & 0.545 & 0.557 & 0.548 & 0.556 & {\underline{0.533}}  & {\underline{0.546}}  \\ 
\midrule
 \multirow{5}{1em}{\rot{ILI}} & 24 & \textbf{2.687} & \underline{1.147} & \underline{2.940} & 1.205 & 4.588 & 1.462 & 3.101 & 1.238 & {3.041} & 1.186 & 3.305 & 1.241 & 3.194 & {1.176} & {3.121} & {\textbf{1.139}}  \\ 
& 36 &\underline{ 2.887} & \textbf{1.160} & \textbf{2.826} & {1.184} & 4.845 & 1.496 & 3.397 & 1.270 & 3.406 & 1.232 & 3.542 & 1.314 & 3.193 & 1.169  & {3.288} & \underline{1.182} \\ 
 & 48 & \underline{2.797} & \underline{1.155} & \textbf{2.677} & \underline{1.155} & 4.865 & 1.516 & {2.947} & 1.203 & 3.459 & 1.221 & 3.409 & 1.208 & 3.15 & \underline{1.142} & 3.122 & \textbf{1.120} \\ 
& 60 & \textbf{2.809} & \textbf{1.163} & \underline{3.011} & 1.245 & 5.212 & 1.576 & {3.019} & 1.202 & 3.640 & 1.305 & 3.668 & 1.269 & 3.43 & {1.196} & 3.416 & \underline{1.180} \\ 
\midrule
\multirow{6}{1em}{\rot{Traffic}} & 24 & 0.562 & 0.375 & \underline{0.351} & \textbf{0.261} & 0.608 & 0.334 & 0.550 & 0.363 & 0.491 & 0.271 & {0.484} & \underline{0.266} & 0.499 & 0.277 & {{0.505}} & {0.294} \\ 
& 48 & 0.567 & 0.374 & 0.370 & \underline{0.270} & 0.644 & 0.359 & 0.595 &0.376 & 0.519 & 0.295 & 0.516 & 0.293 &  0.516 & \underline{0.290} & \textbf{0.315} & {\textbf{0.269}} \\ 
 & 168 & 0.607 & 0.385 & \textbf{0.395} & \textbf{0.277} & 0.660 & 0.391  & 0.649 & 0.407 & 0.513 & 0.289 & 0.511 & 0.301  & 0.517 & 0.289 & \underline{0.508} & \underline{0.286} \\ 
& 336 & 0.624 & 0.389 & \textbf{0.415} & \textbf{0.289} & 0.747 & 0.405 & 0.624 & 0.388 & 0.530 & 0.300 & 0.531 & 0.316 &  0.544 & {0.303} & \underline{0.506} & {\underline{0.299}} \\ 
& 720 & 0.623 & 0.378 & \textbf{0.455} & 0.313 & 0.792 & 0.430 & 0.674 & 0.417 & 0.573 & 0.313 & 0.569 & \underline{0.303} & {0.563} & 0.311 & {\underline{0.539}} & {\textbf{0.300}} \\ 
\bottomrule
\end{tabular}
}
\label{table:result}
\end{table}

\subsection{Memory-Enhanced Prediction: Memory Plugin via Hopfield Layer}

In \cref{tab:mem_enhanced} and \cref{fig:case34}, we 
showcase STanHop-Net with external memory enhancements delivers performance boosts in many scenarios. The external memory enhancements support two plugin schemes, \textbf{Plug-and-Play} and \textbf{Tune-and-Play}. They focus on different benefits.
$\mathtt{TuneMemory}$ is especially useful for task-relevant knowledge incorporation by fine-tuning on an external task-relevant memory set\footnote{Task-relevant means the relevance to the inputs of the time series forecasting. A task-relevant memory set could be a set of some history time series segments that are relevant to the inputs of the prediction.}. 
On the other hand,
$\mathtt{PlugMemory}$ provides a more robust representation of inputs with high uncertainty by doing a retrieval (\cref{fig:StanHop}) on an external task-relevant memory set, without the work of any training or fine-tuning.
Below we provide 4 practical scenarios to showcase the aforementioned benefits of $\mathtt{TuneMemory}$ and $\mathtt{PlugMemory}$ external memory modules. 
The detailed setups of each case can be found in the appendix.

\paragraph{Case 1 ($\mathtt{TuneMemory}$).}
We take the single variate, \textit{Number of Influenza incidence in a week} (denoted as ILI OT), from the \textbf{ILI} dataset as a straightforward example.
In this dataset, we are aware of the existence of recurring annual patterns, which can be readily identified through visualizations in \cref{fig:ILI_OT}.
Notably, the signal patterns around the spring of 2014 closely resemble past springs.
Thus, in predictions tasks with input located in the yearly recurring period, we collect similar patterns from the past to form a task-relevant external memory set.

\paragraph{Case 2 ($\mathtt{TuneMemory}$).}
In many sociological studies \cite{elec1, wang2021electricity}, electricity usage exhibits consistent patterns across different regions, influenced by the daily and weekly routines of residents and local businesses.
Thus,  we collect sequences that match the length of the input sequence but are from 1 to 20 weeks prior, obtaining a task-relevant external memory set of size 20.

In addition, we also include analysis of \textbf{``bad" external memory sets}, to verify the effectiveness of incorporating informative external memory sets.
We construct the ``bad" external memory sets by randomly selecting from dataset without any task-relevant preference, see \cref{sec:ext-mem} for more details about such selection.
The results indicate that, by properly selecting external memory sets, we further improve the models' performance.
On the contrary, randomly chosen external memory sets can negatively impact performance.
We report the results of {Case 1} and {Case 2} in \cref{tab:mem_enhanced}.

\paragraph{Case 3 ($\mathtt{PlugMemory}$).}
Through $\mathtt{PlugMemory}$, informative patterns can be extracted from a  memory set for the given noisy input.
To verify this ability, we construct the external memory sets based on the weekly pattern spotted in ETTh1 and ETTm1, and add noise of different scales into the input sequence.
We add the noise following $x \gets x + \text{scale} \cdot \text{std}(x).$
For {Case 3}, we use the ETTh1 dataset.
\paragraph{Case 4 ($\mathtt{PlugMemory}$).}
For {Case 4}, we evaluate $\mathtt{PlugMemory}$ on the ETTm1 dataset.

\begin{figure*}[t!]

    \centering
    \includegraphics[width=\textwidth]{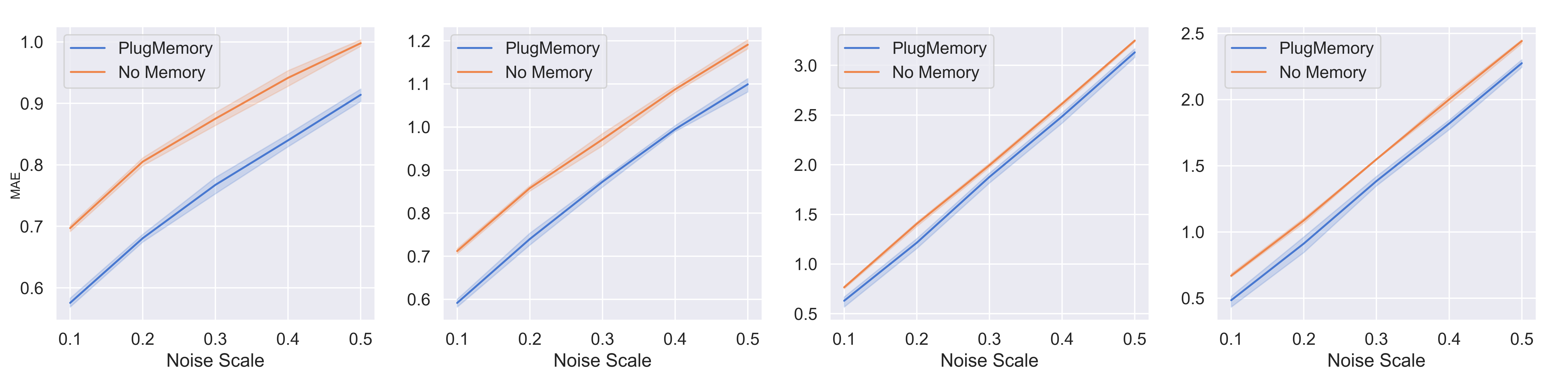}

    \caption{\small
    \textbf{Visualization of Memory Plugin Scenarios Case 3 \& 4.} \textbf{From Left to Right:} MAE against different noise levels with (1) ETTh1 + prediction horizon 336; (2) ETTh1 + prediction horizon 168; (3) ETTm1 + prediction horizon 288; and (4) ETTm1 + prediction horizon 96. 
    The results show the robustness of $\mathtt{PlugMemory}$ against different level of noise.}
    \label{fig:case34}
\end{figure*}

\begin{table}[h!]
    \centering
    \caption{\small
    \textbf{Performance Comparison of the StanHop Model with $\mathtt{TuneMemory}$ and Ablation Using Bad External Memory Sets ($\mathtt{TuneMemory}$-(b))}.
    We report the mean MSE and MAE over 10 runs with variances omitted as they are $\leq 0.79 \%$.
    For ILI OT, we consider prediction horizons of 12, 24, and 60.
    For ETTh1, we choose prediction horizons of 24, 48, and 720, covering both short and long durations.
    The results indicate that using dataset insights and $\mathtt{TuneMemory}$ enhances our model's performance.
    }
    \resizebox{\textwidth}{!}{
    \begin{tabular}{cccccccccccccccc}
    \toprule
     & \multicolumn{5}{c}{\small{\textbf{Case 1} (ILI OT)}} & & $\mid$ & & \multicolumn{5}{c}{\small{\textbf{Case 2} (ETTh1)}} \\
     \midrule
      & \multicolumn{2}{c}{\small{Default}} & \multicolumn{2}{c}{\small{$\mathtt{TuneMemory}$}} & \multicolumn{2}{c}{\small{$\mathtt{TuneMemory}$-(b)}} &   $\mid$ & &  \multicolumn{2}{c}{\small{Default}} & \multicolumn{2}{c}{\small{$\mathtt{TuneMemory}$}} & \multicolumn{2}{c}{\small{$\mathtt{TuneMemory}$-(b)}} \\
    \midrule
      & \small{MSE} & \small{MAE} & \small{MSE} & \small{MAE} & \small{MSE} & \small{MAE} & $\mid$ & &  \small{MSE} & \small{MAE} & \small{MSE} & \small{MAE} & \small{MSE} & \small{MAE} \\
    \midrule
    12 & 4.011 & 1.701 & 3.975 {\footnotesize(\textbf{-0.9\%})} & 1.693 {\footnotesize(\textbf{-0.5\%})} & 4.340 {\footnotesize(+8.2\%)} & 1.789 {\footnotesize(+5.1\%)} & $\mid$ & 24 & 0.294 & 0.351 & 0.284 {\footnotesize(\textbf{-3.4\%})} & 0.351 {\footnotesize($\pm$0\%)} & 0.300 {\footnotesize(+2\%)} & 0.361 {\footnotesize(+2.8\%)} \\
    24 & 4.254 & 1.771 & 3.960 {\footnotesize(\textbf{-6.9\%})} & 1.690 {\footnotesize(\textbf{-4.6\%})} & 4.271 {\footnotesize(+0.4\%)} & 1.776 {\footnotesize(+0.3\%)} & $\mid$ & 48 & 0.340 & 0.387 & 0.328 {\footnotesize(\textbf{-3.5\%})} & 0.379 {\footnotesize(\textbf{-2.1\%})} & 0.342 {\footnotesize(+0.6\%)} & 0.388 {\footnotesize(+0.3\%)} \\
    60 & 3.613 & 1.685 & 3.572 {\footnotesize(\textbf{-1.1\%})} & 1.528 {\footnotesize(\textbf{-9.3\%})} & 3.821 {\footnotesize(+5.8\%)} & 1.725 {\footnotesize(+2.4\%)} & $\mid$ & 720 & 0.512 & 0.511 & 0.504 {\footnotesize(\textbf{-1.6\%})} & 0.512 {\footnotesize(\textbf{-0.2\%})} & 0.514 {\footnotesize(+0.4\%)} & 0.521 {\footnotesize(+2.0\%)}\\
    \bottomrule
    \end{tabular}
    }
    \label{tab:mem_enhanced}
\end{table}

\section{Conclusion}
\label{sec:conclusion}
We propose the generalized sparse modern Hopfield model and present STanHop-Net, a Hopfield-based time series prediction model with external memory functionalities. 
Our design improves time series forecasting performance, quickly reacts to unexpected or rare events, and offers both strong theoretical guarantees and empirical results.
Empirically, STanHop-Nets rank in the top two in 47 out of our 58 experiment settings compared to the baselines.
Furthermore, with $\mathtt{PlugMemory}$ and $\mathtt{TuneMemory}$ modules, it showcases average performance boosts of $\sim$12\% and  $\sim$3\% for each.
In the appendix, we also show that STanHop-Net consistently outperforms DLinear in scenarios with strong correlations between variates.

\clearpage

\section*{Acknowledgments}
JH would like to thank Feng Ruan, Dino Feng and Andrew Chen for enlightening discussions, the Red Maple Family for support, and Jiayi Wang for facilitating experimental deployments.

JH is partially supported by the Walter P. Murphy Fellowship.
BY is supported by the National Taiwan University Fu Bell Scholarship.
HL is partially supported by NIH R01LM1372201, NSF CAREER1841569, DOE DE-AC02-07CH11359, DOE LAB 20-2261 and a NSF TRIPODS1740735.
This research was supported in part through the computational resources and staff contributions provided for the Quest high performance computing facility at Northwestern University which is jointly supported by the Office of the Provost, the Office for Research, and Northwestern University Information Technology.
The content is solely the responsibility of the authors and does not necessarily represent the official
views of the funding agencies.

\def\arxivfont{\rm}
\bibliographystyle{plainnat}

\bibliography{refs}

\newpage  %

\titlespacing*{\section}{0pt}{*1}{*1}
\titlespacing*{\subsection}{0pt}{*1.25}{*1.25}
\titlespacing*{\subsubsection}{0pt}{*1.5}{*1.5}

\setlength{\abovedisplayskip}{10pt}
\setlength{\abovedisplayshortskip}{10pt}
\setlength{\belowdisplayskip}{10pt}
\setlength{\belowdisplayshortskip}{10pt}

\appendix
\label{sec:append}
\part*{Supplementary Material}
{
\setlength{\parskip}{-0em}
\startcontents[sections]
\printcontents[sections]{ }{1}{}
}

\clearpage
\section{Broader Impacts}

We envision this approach as a means to refine large foundation models for time series, through a perspective shaped by neuroscience insights.
Such memory-enhanced time series foundation models are vital in applications like eco- and climatic-modeling.
For example, with a multi-modal time series foundation model, we can effectively predict, detect, and mitigate emerging biological threats associated with the rapid changes in global climate.
To this end, the differentiable external memory modules become handy, as they allow users to integrate real-time data into pre-trained foundation models and thus enhance the model's responsiveness in real-time scenarios. 
Specifically, one can use this memory-enhanced technique to embed historical, sudden, or rare events into any given time series foundation model, thereby boosting its overall performance.

\section{Related Works and Limitations}
\label{sec:related_works}

\paragraph{Transformers for Time Series Prediction.}
As suggested in \cref{sec:model} and \cite{hu2023SparseHopfield,ramsauer2020hopfield}, besides the additional memory functionalities, the Hopfield layers act as promising alternatives for the attention mechanism.
Therefore, we discuss and compare STanHop-Net with existing transformer-based time series prediction methods here. 

Transformers have gained prominence in time series prediction, inspired by their success in Natural Language Processing and Computer Vision. 
One challenge in time series prediction is managing transformers' quadratic complexity due to the typically long sequences. To address this, many researchers have not only optimized for prediction performance but also sought to reduce memory and computational complexity. LogTrans \cite{li2019enhancing} proposes a transformer-based neural network for time series prediction. They propose a convolution layer over the vanilla transformer to better capture local context information and a sparse attention mechanism to reduce memory complexity. Similarly, Informer \cite{zhou2021informer} proposes convolutional layers in between attention blocks to distill the dominating attention and a sparse attention mechanism where the keys only attend to a subset of queries. 
Reformer \cite{kitaev2020reformer} replaces the dot-product self-attention in the vanilla transformer with a hashing-based attention mechanism to reduce the complexity. Besides directly feeding the raw time series inputs to the model, many works focus on transformer-based time series prediction by modeling the decomposed time series. Autoformer \cite{wu2021autoformer} introduces a series decomposition module to its transformer-based model to separately model the seasonal component and the trend-cyclical of the time series. FEDformer \cite{zhou2022fedformer} also models the decomposed time series and they introduce a block to extract signals by transforming the time series to the frequency domain.

Compared to STanHop, the above methods do not model multi-resolution information. Besides, Reformer's attention mechanism sacrifices the global receptive field compared to the vanilla self-attention mechanism and our method, which harms the prediction performance. 

Some works intend to model the multi-resolution or multi-scale signals in the time series with a dedicated network design. Pyraformer \cite{liu2021pyraformer} designs a pyramidal attention module to extract the multi-scale signals from the raw time series. Crossformer \cite{zhang2022crossformer} proposes a multi-scale encoder-decoder architecture to hierarchically extract signals of different resolutions from the time series. Compared to these methods. STanHop adopts a more fine-grained multi-resolution modeling mechanism that is capable of learning different sparsity levels for signals in the data of different resolutions.

Furthermore, all of the above works on time series prediction lack the external memory retrieval module as ours. Thus, our STanHop method and its variations have a unique advantage in that we have a fast response to real-time unexpected events.

\paragraph{Hopfield Models and Deep Learning.}
Hopfield Models \cite{hopfield1984neurons,hopfield1982neural,krotov2016dense} have garnered renewed interest in the machine learning community due to 
the connection between their memory retrieval dynamics and attention mechanisms in transformers via the Modern Hopfield Models \cite{hu2023SparseHopfield,ramsauer2020hopfield}. 
Furthermore, these modern Hopfield models enjoy superior empirical performance and possess several appealing theoretical properties, such as rapid convergence and guaranteed exponential memory capacity. 
By viewing modern Hopfield models as generalized attentions with enhanced memory functionalities, these advancements pave the way for innovative Hopfield-centric architectural designs in deep learning \cite{hoover2023energy,seidl2022improving,furst2022cloob,ramsauer2020hopfield}.
Consequently, their applicability spans diverse areas like  physics \cite{krotov2023new}, biology \cite{schimunek2023contextenriched,kozachkov2023building,widrich2020modern}, reinforcement learning \cite{paischer2022history}, and large language models \cite{furst2022cloob}. 

This work pushes this line of research forward by
presenting a Hopfield-based deep architecture (StanHop-Net)   tailored for memory-enhanced learning in noisy multivariate time series. 
In particular, our model emphasizes in-context memorization during training and bolsters retrieval capabilities with an integrated external memory component.

\paragraph{Sparse Modern Hopfield Model.}
Our work extends the theoretical framework proposed in \cite{hu2023SparseHopfield} for modern Hopfield models. Their primary insight is that using different entropic regularizers can lead to distributions with varying sparsity. 
Using the Gibbs entropic regularizer, they reproduce the results of the standard dense Hopfield model \cite{ramsauer2020hopfield} and further propose a sparse variant with the Gini entropic regularizer, providing improved theoretical guarantees. However, their sparse model primarily thrives with data of high intrinsic sparsity.
To combat this, we enrich the link between Hopfield models and attention mechanisms by introducing \textit{learnable sparsity} and showing that  the sparse model from \cite{hu2023SparseHopfield} is a specific case of our model when setting $\alpha=2$.
Unlike \cite{hu2023SparseHopfield}, our generalized sparse Hopfield model ensures adaptable sparsity across various data types without sacrificing theoretical integrity. 
By making this sparsity learnable, we introduce the $\mathtt{GSH}$ layers. 
These new Hopfield layers adeptly learn and store sparse representations in any deep learning pipeline, proving invaluable for inherently noisy time series data.

\paragraph{Memory Augmented Neural Networks.}
The integration of external memory mechanisms with neural networks has emerged as a pivotal technique for machine learning, particularly for tasks requiring complex data manipulation and retention over long sequences, such as open question answering, and few-shot learning.

Neural Turing Machines (NTMs) \cite{graves2014neural} combine the capabilities of neural networks with the external memory access of a Turing machine. 
NTMs use a differentiable controller (typically an RNN) to interact with an external memory matrix through read and write heads. 
This design allows NTMs to perform complex data manipulations, akin to a computer with a read-write memory.
Building upon this, \citet{graves2016hybrid} further improve the concept through Differentiable Neural Computers (DNCs), which enhance the memory access mechanism using a differentiable attention process. This includes a dynamic memory allocation and linkage system that tracks the relationships between different pieces of data in memory.
This feature makes DNCs particularly adept at tasks that require complex data relationships and temporal linkages.

Concurrently, Memory Networks \cite{weston2014memory} showcase the significance of external memory in tasks requiring complex reasoning and inference. 
Unlike traditional neural networks that rely on their inherent weights to store information, Memory Networks incorporate a separate memory matrix that stores and retrieves information across different processing steps. 
This capability allows the network to maintain and manipulate a ``memory'' of past inputs and computations, which is particularly crucial for tasks requiring persistent memory, such as question-answering systems where the network needs to remember context or facts from previous parts of a conversation or text to accurately respond to queries. 
This concept is further developed into the End-to-End Memory Networks \cite{sukhbaatar2015end}, which extend the utility of Memory Networks \cite{weston2014memory} beyond the limitations of traditional recurrent neural network architectures, transitioning them into a fully end-to-end trainable framework, thereby making them more adaptable and easier to integrate into various learning paradigms.

A notable application of memory-augmented neural networks is in the domain of one-shot learning.
The concept of meta-learning with memory-augmented neural networks, as explored by \citet{santoro2016meta}, has demonstrated the potential of these networks to rapidly adapt to new tasks by leveraging their external memory, highlighting their versatility and efficiency in learning.
They showcase the potential of these networks to adapt rapidly to new tasks, a crucial capability in scenarios where data availability is limited. 
Complementing this, \citet{kaiser2017learning} focus on enhancing the recall of rare events and  is particularly notable for its exploration of memory-augmented neural networks designed to improve the retention and recall of infrequent but significant occurrences, highlighting the potential of external memory modules in handling rare data challenges. 
This is achieved through a unique soft attention mechanism, which dynamically assigns relevance weights to different memory entries, enabling the model to draw on a broad spectrum of stored experiences.
This approach not only facilitates the effective recall of rare events but also adapts to new data, ensuring the memory's relevance and utility over time.

In all above methods, \cite{kaiser2017learning} is closest to this work.
However, our approach diverges in two key aspects:
\textbf{(Enhanced Generalization):} Firstly, our external memory enhancements are external plugins with an option for fine-tuning. 
This design choice avoids over-specialization on rare events, thereby broadening our method's applicability and enhancing its generalization capabilities across various tasks where the frequency and recency of data are less pivotal.
\textbf{(Adaptive Response over Rare Event Memorization):} Secondly, our approach excels in real-time adaptability. 
By integrating relevant external memory sets tailored to specific inference tasks, our method can rapidly respond and improve performance, even without the necessity of prior learning. 
This flexibility contrasts with the primary focus on \textit{memorizing} rare events in \cite{kaiser2017learning}.

\subsection{Limitations}
\label{sec:limitations}
The proposed generalized sparse modern Hopfield model shares the similar inefficiency due to the  $\calO(d^2)$ complexity.
In addition, the effectiveness of our memory enhancement methods is contingent on the relevance of the external memory set to the specific inference task. 
Achieving a high degree of relevance in the external memory set often necessitates considerable human effort and domain expertise, just like our selection process detailed in \cref{sec:ext-mem}.
This requirement could potentially limit the model's applicability in scenarios where such resources are scarce or unavailable.

\clearpage
\section{Proofs of Main Text}

\subsection{\cref{lemma:Danskin}}
\begin{claim}
    Our proof relies on verifying that $\Psi^\star$ meets the criteria of Danskin's theorem.
\end{claim}
\begin{proof}[Proof of \cref{lemma:Danskin}]
\label{proof:Danskin}

Firstly, we introduce the notion of convex conjugate.
\begin{definition}
\label{def:convex_conjugate}
Let $F(\bp,\bz)\coloneqq\Braket{\bp,\bz}-\Psi^\alpha(\bp)$. 
The convex conjugate of $\Psi^\alpha$, $\Psi^\star$ takes the form:
\bea
\Psi^\star(\bz)=\Max_{\bp\in\Delta^M}\Braket{\bp,\bz}-\Psi^\alpha(\bp)=\Max_{\bp\in\Delta^M}F(\bp,\bz).
\eea
\end{definition}

    By Danskin's theorem \cite{danskin2012theory, bertsekas1999nonlinear}, the function $\Psi^\star$ is convex and its partial derivative with respect to $\bz$ is equal to that of $F$, i.e. $\nicefrac{\partial \Psi^\star}{\partial \bz}=\nicefrac{\partial F}{\partial \bz}$, if the following three conditions are satisfied for $\Psi^\star$ and $F$:
\begin{enumerate}
    \item [(i)]
    $F(\bp,\bz):\calP\times \R^M\to \R$ is a continuous function, where $\calP\subset\R^M$ is a compact set.
    \item [(ii)]
    $F$ is convex in $\bz$, i.e. for each given $\bp\in\calP$, the mapping $\bz\to F(\bp,\bz)$ is convex.
    \item [(iii)]
    There exists an unique maximizing point $\hat{\bp}$ such that $F(\hat{\bp},\bz)=\Max_{\bp\in\calP}F(\bp,\bz)$.
\end{enumerate}

Since both $\Braket{\bp,\bz}$ and $\Psi^\alpha$ are continuous functions and every component of $\bp$ is ranging from $0$ to $1$,  the function $F$ is continuous and the domain $\calP$ is a compact set.
Therefore, condition (i) is satisfied. 

Since we require $\bp\in \Delta^M$ (i.e. $\calP= \Delta^M$) to be probability distributions, 
for any fixed $\bp$, $F(\bp,\bz)=\Braket{\bp,\bz}-\Psi^\alpha(\bp)$ reduces to an affine function depending only on input $\bz$.
Due to the inner product form, this affine function is convex in $\bz$, and hence condition (ii) holds for all given $\bp\in\calP=\Delta^M$. 

Since, for any given $\bz$, $\alphaentmax$ only produces one unique probability distribution $\bp^\star$, condition (iii) is satisfied. 
Therefore, from Danskin's theorem, 
it holds
\bea
\grad_z \Psi^\star(\bz)=\frac{\partial F}{\partial \bz}=\frac{\partial}{\partial \bz}(\Braket{\bp,\bz}-\Psi^\alpha(\bp))=\bp=\alphaentmax(\bz).
\eea
\end{proof}

\subsection{\cref{lemma:retrieval_dyn}}
\begin{claim}
Our proof is built on \cite[Lemma~2.1]{hu2023SparseHopfield}.
We first derive $\calT$ by utilizing \cref{lemma:Danskin} and \cref{remark:closeform},  along with the convex-concave procedure \cite{yuille2003concave,yuille2001concave}.
Then, we show the monotonicity of minimizing $\calH$ with $\calT$ by constructing an iterative upper bound of $\calH$ which is convex in $\bx_{t+1}$ and thus, can be lowered iteratively by the convex-concave procedure.
\end{claim}

\begin{proof}
\label{proof:retrieval_dyn}

From \cref{lemma:Danskin}, the conjugate convex of $\Psi$, $\Psi^\star$, is always convex, and, therefore, $-\Psi^\star$ is a concave function. 
Then, the energy function $\calH$ defined in \eqref{eqn:H_entmax} is the sum of the convex function $\calH_1(\bx)\coloneqq\half\Braket{\bx,\bx}$ and the concave function $\calH_2(\bx)\coloneqq-\Psi^\star\(\bm{\Xi}^\sT\bx\)$.

Furthermore, by definition, the energy function $\calH$ is differentiable.

Every iteration step of convex-concave procedure applied on $\calH$ gives
\begin{align}
\grad_\bx\calH_1\(\bx_{t+1}\)=-\grad_{\bx}\calH_2\(\bx_t\),
\end{align}
which implies that
\bea
\bx_{t+1}=\grad_{\bx}\Psi\(\bm{\Xi}\bx_t\)= \bm{\Xi}\alphaentmax\(\bm{\Xi}^\sT \bx_t\).
\eea

On the basis of \cite{yuille2003concave,yuille2001concave}, we show the decreasing property of \eqref{eqn:H_entmax} over $t$
via solving the minimization problem of energy function:
\bea
\label{eqn:energy_minimizaiton}
\Min_\bx \[\calH(\bx)\]&=&\Min_\bx \[\calH_1(\bx)+\calH_2(\bx)\],
\eea
which, in convex-concave procedure, is equivalent to solve the iterative programming
{
\bea
\label{eqn:iterative_argmin}
\bx_{t+1} &\in& \argmin_{\bx} \[\calH_1(\bx)+
\Braket{\bx,\grad_\bx \calH_2\(\bx_t\)}
\],
\eea
for all $t$.
}
{
The concept behind this programming is to linearize the concave function $\calH_2$ around the solution for current iteration, $\bx_t$, which makes $\calH_1(\bx_{t+1})+\Braket{\bx_{t+1},\grad_\bx \calH_2(\bx_t)}$ convex in $\bx_{t+1}$.

The convexity of $\calH_1$ and concavity of $\calH_2$ imply that the inequalities
\begin{align}
\calH_1(\bx)&\ge \calH_1(\by)+\Braket{\(\bx-\by\),\grad_{\bx}\calH_1(\by)},\\
\calH_2(\bx)&\le \calH_2(\by)+\Braket{\(\bx-\by\),\grad_{\bx}\calH_2(\by)},
\end{align}
hold for all $\bx,\by$,
}
which leads to 
\begin{align}
\calH(\bx)&=\calH_1(\bx)+\calH_2(\bx)\\
&\le \calH_1(\bx)+\calH_2(\by)  +\Braket{(\bx-\by),\grad_\bx \calH_2(\by)}
\coloneqq \calH_U\(\bx,\by\),
\label{eqn:H_U}
\end{align}
where the upper bound of $\calH$ is defined as $\calH_U$.
Then, the iteration \eqref{eqn:iterative_argmin}
\bea
\bx_{t+1} \in \argmin_{\bx} \[\calH_U(\bx,\bx_t)\]=\argmin_{\bx}\[\calH_1(\bx)+\Braket{\bx,\grad_\bx \calH_2(\bx_t)}\],
\eea
can make $\calH_U$ decrease iteratively and thus decreases the value of energy function $\calH$ monotonically, i.e.
\begin{align}
\calH(\bx_{t+1}) &\leq \calH_U(\bx_{t+1},\bx_t) 
\leq \calH_U(\bx_t,\bx_t)
=\calH(\bx_t),
\end{align}
for all $t$.
\cref{eqn:H_U} shows that the retrieval dynamics defined in \eqref{lemma:retrieval_dyn} can lead the energy function $\calH$ to decrease with respect to the increasing $t$.
\end{proof}

\subsection{\cref{lemma:convergence_sparse}}
\label{proof:convergence_sparse}
To prove the convergence property of retrieval dynamics $\calT$, first we introduce an auxiliary lemma from \cite{sriperumbudur2009convergence}.

\begin{lemma}[\cite{sriperumbudur2009convergence}, Lemma 5]
\label{KKT for stationary points}
Following \cref{lemma:convergence_sparse}, $\bx$ is called the fixed point of iteration $\calT$ with respect to $\calH$ if $\bx = \calT(\bx)$ and is considered as a generalized fixed point of $\calT$ if $\bx\in\calT(\bx)$.
If $\bx^\star$ is a generalized fixed point of $\calT$, then, $\bx^\star$ is a stationary point of the energy minimization problem  \eqref{eqn:energy_minimizaiton}.
\end{lemma}

\begin{proof}
Since the energy function $\calH$ monotonically decreases with respect to increasing $t$ in \cref{lemma:retrieval_dyn}, we can follow \cite[Lemma~2.2]{hu2023SparseHopfield} to guarantee the convergence property of $\calT$ by checking the necessary conditions of Zangwill's global convergence.
After satisfying these conditions, Zangwill global convergence theory ensures that all the limit points of $\{\bx_t\}_{t=0}^\infty$ are generalized fixed points of the mapping $\calT$ and it holds $\lim_{t \to \infty} \calH\(\bx_t\)=\calH\(\bx^\star\)$, where $\bx^\star$ are some generalized fixed points of $\calT$.
Furthermore, auxiliary \cref{KKT for stationary points} implies that $\bx^\star$ are also the stationary points of energy function $\calH$.
Therefore, we guarantee that $\calT$ can iteratively lead the query $\bx$ to converge to the local optimum of $\calH$.
\end{proof}

\clearpage
\subsection{\cref{thm:eps_sparse_dense}}
\begin{proof}
\label{proof:eps_sparse_dense}
We observe
\begin{align}
&\norm{\calT(\bx)-\bxi_\mu}-\norm{\calT_{\text{Dense}}(\bx)-\bxi_\mu}\nonumber\\
&=\norm{\sum_{\nu=1}^{\kappa} \bxi_\nu \[\text{$(\alpha+\delta)$-entmax}\(\beta \bm{\Xi}^\sT \bx\) \]_\nu-\bxi_\mu}
-\norm{\sum_{\nu=1}^{\kappa} \bxi_\nu \[\text{$\alpha$-entmax}\(\beta \bm{\Xi}^\sT \bx\) \]_\nu-\bxi_\mu}\\
&\leq
\norm{\sum^{\kappa}_{\nu=1}\[\text{$(\alpha+\delta)$-entmax}(\beta\bm{\Xi}^\sT \bx)\]_\nu\bxi_\nu }-
\norm{\sum^{\kappa}_{\nu=1} \[\text{$\alpha$-entmax}\(\beta \bm{\Xi}^\sT \bx\)\]_\nu \bxi_\nu}\leq 0,
\end{align}
which gives
\bea
\norm{\calT(\bx)-\bxi_\mu} 
&\leq& 
\norm{\calT_{\text{Dense}}(\bx)-\bxi_\mu}.
\eea

\paragraph{For $2\ge\alpha\ge1$:}
Then, we derive the upper bound on $\norm{\calT_{\text{dense}}(\bx)-\bxi_\mu}$ based on \cite[Theorem~2.2]{hu2023SparseHopfield}:
\begin{align}
\norm{\calT_{\text{dense}}(\bx)-\bxi_\mu}
&=
\norm{\sum_{\nu=1}^M [\Softmax(\beta\bm{\Xi}^\sT\bx)]_\nu\bxi_\nu-\bxi_\mu}\\
&=
\norm{\sum_{\nu=1, \nu\neq\mu}^M[\Softmax(\beta\bm{\Xi}^\sT\bx)]_\nu\bxi_\nu-(1-\Softmax(\beta\bm{\Xi}^\sT\bx))\bxi_\mu}\\
&\leq
2\Tilde{\epsilon}m,
\end{align}
where $\tilde{\epsilon}\coloneqq (M-1)\exp{-\beta \tilde{\Delta}_\mu}
=(M-1)\exp{-\beta \(\Braket{\bxi_\mu,\bx}-\Max_{\nu\in[M]}\Braket{\bxi_\mu,\bxi_\nu}\)}
$.
Consequently,
\eqref{eqn:eps_sparse_dense} results from above and \cite[Theorem~4,5]{ramsauer2020hopfield}.

\paragraph{For $\alpha\ge2$.}
Following the setting of $\alphaentmax$ in \cite{peters2019sparse}, the equation
\bea
2\text{-EntMax}(\beta\bm{\Xi}^\sT\bx)=\Sparsemax(\beta\bm{\Xi}^\sT\bx)
\eea
holds.
According to the closed form solution of $\Sparsemax$ in \cite{martins2016softmax}, 
it holds
\begin{align}
[ \Sparsemax \(\beta \bm{\Xi}^\sT \bx\) ]_\mu 
&\le 
\[\beta \bm{\Xi}^\sT \bx \]_{\mu}-\[\beta \bm{\Xi}^\sT \bx\]_{(\kappa)}+\frac{1}{\kappa},
\label{eqn:sparsemax_upper_identity}
\end{align}
for all $\mu\in[M]$.
Then, the sparsemax retrieval error is 
\begin{align}
\norm{\calT_{\Sparsemax}\(\bx\)-\bxi^\mu}&=
\norm{\bm{\bm{\Xi}}\Sparsemax\(\beta \bm{\bm{\Xi}}^\sT \bx\)-\bxi^\mu}
=\norm{\sum_{\nu=1}^{\kappa} \bxi_{(\nu)} \[\Sparsemax\(\beta \bm{\bm{\Xi}}^\sT \bx\) \]_{(\nu)}-\bxi^\mu}\nonumber\\
&\le 
m+m\beta \norm{\sum^{\kappa}_{\nu=1} \(\[ \bm{\Xi}^\sT \bx \]_{(\nu)}-\[ \bm{\Xi}^\sT \bx\]_{(\kappa)}+\frac{1}{\beta\kappa}\)\frac{\bxi_{(\nu)}}{m}}
\annot{By \eqref{eqn:sparsemax_upper_identity}}\\
&\le 
m+d^{\nicefrac{1}{2}}m\beta \[\kappa \(\Max_{\nu\in[M]}\Braket{\bxi_\nu,\bx}-\[ \bm{\Xi}^\sT \bx\]_{(\kappa)}\)+\frac{1}{\beta}\].
\end{align}
By the first inequality of \cref{thm:eps_sparse_dense}, for $\alpha\geq 2$, we have
\bea
\norm{\calT(\bx)-\bxi_\mu}\leq\norm{\calT_{\Sparsemax}(\bx)-\bxi_\mu}\leq m+d^{\nicefrac{1}{2}}m\beta \[\kappa \(\Max_{\nu\in[M]}\Braket{\bxi_\nu,\bx}-\[ \bm{\Xi}^\sT \bx\]_{(\kappa)}\)+\frac{1}{\beta}\],\nonumber
\eea
which completes the proof of \eqref{eqn:eps_sparse_over_2}.
\end{proof}

\subsection{\cref{thm:memory_capacity}}
\label{proof:memory_capacity}

\begin{claim}
Our proof, built on \cite[Lemma~2.1]{hu2023SparseHopfield}, proceeds in 3 steps:
\begin{itemize}
    \item \textbf{(Step 1.)} We establish a more refined well-separation condition, ensuring that patterns $\{\bxi_\mu\}_{\mu\in[M]}$ are well-stored in $\calH$ and can be retrieved by $\calT$ with an error $\epsilon$ at most $R$.

    \item \textbf{(Step 2.)} This condition is then related to the cosine similarity of memory patterns, from which we deduce an inequality governing the probability of successful pattern storage and retrieval.

    \item \textbf{(Step 3.)} We pinpoint the conditions for exponential memory capacity and confirm their satisfaction.
\end{itemize}
Since the generalized sparse Hopfield shares the same well-separation condition (shown in below \cref{thm:well_separation_condition}), it has the same exponential memory capacity as the sparse Hopfield model \cite[Lemma~3.1]{hu2023SparseHopfield}.
For completeness, we restate  the proof of \cite[Lemma~3.1]{hu2023SparseHopfield} below.
\end{claim}

\paragraph{Step 1.} 
To analyze the memory capacity of the proposed model, we first present the following two auxiliary lemmas.

\begin{lemma}
\label{thm:well_separation_condition}[Corollary~3.1.1 
 of \cite{hu2023SparseHopfield}]
Let $\delta\coloneqq \norm{\calT_{\text{Dense}}-\bxi_\mu}-\norm{\calT-\bxi_\mu}$.
Then, the {well-separation} condition can be formulated as:
\bea
\Delta_\mu \ge 
\frac{1}{\beta}\ln(\frac{2(M-1)m}{R+\delta})+2mR.
\eea
Furthermore, if $\delta=0$, this bound reduces to {well-separation} condition of Softmax-based Hopfield model.
\end{lemma}
\begin{proof}[Proof of \cref{thm:well_separation_condition}]

Let $\calT_{\text{Dense}}$ be the retrieval dynamics given by the dense modern Hopfield model \cite{ramsauer2020hopfield},
and $\norm{\calT(\bx)-\bxi_\mu}$ and $\norm{\calT_{\text{Dense}}(\bx)-\bxi_\mu}$ be the retrieval error of generalized sparse and dense modern Hopfield model, respectively.
By \cref{thm:eps_sparse_dense}, we have
\bea
\norm{\calT(\bx)-\bxi_\mu} 
\leq 
\norm{\calT_{\text{Dense}}(\bx)-\bxi_\mu}.
\eea

By \cite[Lemma~A.4]{ramsauer2020hopfield}, we have
\begin{align}
\norm{\calT_{\text{Dense}}(\bx)-\bxi_\mu}
\label{eqn:bound_eps}
\leq
2\tilde{\epsilon} m,
\end{align}
where $\tilde{\epsilon}\coloneqq (M-1)\exp{-\beta \tilde{\Delta}_\mu}
=(M-1)\exp{-\beta \(\Braket{\bxi_\mu,\bx}-\Max_{\nu\in[M]}\Braket{\bxi_\mu,\bxi_\nu}\)}
$.
Then, by the Cauchy-Schwartz inequality
\bea
\abs{\Braket{\bxi_\mu,\bxi_\mu}-\Braket{\bx,\bxi_\mu}}
\leq
\norm{\bxi_\mu-\bx} \cdot \norm{\bxi_\mu}
\leq
\norm{\bxi_\mu-\bx}m,\quad\forall \mu\in [M],
\eea
we observe that $\tilde{\Delta}_\mu$ can be expressed in terms of $\Delta_\mu$: 
\begin{align}
\Tilde{\Delta}_\mu 
&\le
\Delta_\mu - 2 \norm{\bxi_\mu-\bx}m=\Delta_\mu-2mR,
\end{align}
where $R$ is radius of the sphere $S_{\mu}$.
Thus, inserting the upper bound given by \eqref{eqn:bound_eps} into \eqref{eqn:eps_sparse_dense}, we obtain
\bea
\norm{\calT(\bx)-\bxi_\mu}
&\leq&
\norm{\calT_{\text{Dense}}(\bx)-\bxi_\mu} 
\leq 
2 \Tilde{\epsilon} m\\
&\leq&
2(M-1) \exp{-\beta\(\Delta_\mu-2mR\)}m.
\eea
Then, for any given $\delta\coloneqq \norm{\calT_{\text{Dense}}(\bx)-\bxi_\mu}-\norm{\calT(\bx)-\bxi_\mu}\le 0 $, the retrieval error $\norm{\calT(\bx)-\bxi_\mu}$ has an upper bound:
\bea
\norm{\calT(\bx)-\bxi_\mu} \le 2(M-1) \exp{-\beta\(\Delta_\mu-2mR +\delta\)}m - \delta
\le \norm{\calT_{\text{Dense}}(\bx)-\bxi_\mu} .
\eea
Therefore, for $\calT$ to be a mapping $\calT:S_\mu\to S_\mu$, we need the well-separation condition
\bea
\Delta_\mu
\geq
\frac{1}{\beta}\ln(\frac{2(M-1)m}{R+\delta})+2mR.
\eea
\end{proof}

\begin{lemma}[\cite{hu2023SparseHopfield,ramsauer2020hopfield}]
\label{lemma:W} 
If the identity
\begin{equation}
ac + c\ln{c} - b = 0,
\end{equation}
holds for all real numbers $a, b \in \mathbb{R}$, then $c$ takes a solution:
\begin{equation}
c = \frac{b}{W_0(\exp(a+\ln{b}))}.
\end{equation}
\end{lemma}
\begin{proof}[Proof of \cref{lemma:W}]
We restate the proof of \cite[Lemam~3.1]{hu2023SparseHopfield} here for completeness.

With the given equation $ac + c\ln{c} - b = 0$, we solve for $c$ by following steps:
\begin{align*}
ac + c\ln{c} - b &= 0, \\
a + \ln{c} &= \frac{b}{c}, \\
\frac{b}{c} + \ln\left(\frac{b}{c}\right) &= a + \ln{b}, \\
\frac{b}{c}\exp\left(\frac{b}{c}\right) &= \exp(a + \ln{b}), \\
\frac{b}{c} &= W_0(\exp(a + \ln{b})), \\
c &= \frac{b}{W_0(\exp(a + \ln{b}))}.
\end{align*}
\end{proof}

Then, we present the main proof of \cref{thm:memory_capacity}.

\begin{proof}[Proof of \cref{thm:memory_capacity}]
Since the generalized Hopfield model shares the same well-separation condition as the sparse Hopfield model \cite{hu2023SparseHopfield}, the proof of the exponential memory capacity automatically follows that of \cite{hu2023SparseHopfield}.
We restate the proof of \cite[Corollary~3.1.1]{hu2023SparseHopfield} here for completeness.

\paragraph{(Step 2.) \& (Step 3.)}
Here we define $\Delta_{\min}$ and $\theta_{\mu\nu}$ as $\Delta_{\min}\coloneqq\Min_{\mu\in[M]}\Delta_\mu$ and the angle between two patterns $\bxi^\mu$ and $\bxi^\nu$, respectively. 
Intuitively, $\theta_{\mu\nu}\in[0,\pi]$ represent the pairwise correlation of two patterns the two patterns and hence 
\bea
\Delta_{\min}=\Min_{1\le\mu\le\nu\le M}\[m^2\(1-\cos{\theta_{\mu\nu}}\)\]
=m^2 \[1-\cos{\theta_{\min}}\],
\eea
where $\theta_{\min}\coloneqq \Min_{1\le\mu\le\nu\le M} \theta_{\mu\nu}\in[0,\pi]$.

From the well-separation condition \eqref{thm:well_separation_condition}, we have
\bea
\Delta_\mu\ge \Delta_{\min}\ge 
\frac{1}{\beta}\ln(\frac{2(M-1)m}{R+\delta})+2mR.
\eea
Hence, we have
\bea
\label{eqn:cos}
m^2 \[1-\cos{\theta_{\min}}\] 
\ge 
\frac{1}{\beta}\ln(\frac{2(M-1)m}{R+\delta})+2mR.
\eea

Therefore, we are able to write down the probability of successful storage and retrieval, i.e. minimal separation $\Delta_{\min}$ satisfies \cref{thm:well_separation_condition}:
\bea
\label{success}
P\( m^2 \[1-\cos{\theta_{\min}}\] \ge  \frac{1}{\beta}\ln(\frac{2(M-1)m}{R+\delta})+2mR \) = 1-p.
\eea
By \cite[(4.22.2)]{olver2010nist}, it holds
\bea
\cos{\theta_{\min}} \le 1- \frac{\theta_{\min}^2}{5} \quad \text{for} \quad 0 \le \cos{\theta_{\min}} \le 1,
\eea
and hence
\bea
P\(M^{\frac{2}{d-1}}\theta_{\min} \ge \frac{\sqrt{5}M^{\frac{2}{d-1}}}{m}\[\frac{1}{\beta}\ln(\frac{2(M-1)m}{R+\delta})+2mR\]^\half\) = 1-p.
\label{eqn:ineq1}
\eea
Here we introduce $M^{\nicefrac{2}{d-1}}$ on both sides in above for later convenience. 

Let
$
\omega_d
\coloneqq 
\frac{2\pi^{\nicefrac{d+1}{2}}}{\Gamma\(\frac{d+1}{2}\)} 
$, be the surface area of a $d$-dimensional unit sphere, where $\Gamma(\cdot)$ represents the gamma function.
By \cite[Lemma~3.5]{brauchart2018random}, it holds
\begin{align}
1-p
\ge
1-\half\gamma_{d-1}5^{\frac{d-1}{2}}M^2 m^{-(d-1)}\[\frac{1}{\beta}\ln(\frac{2(M-1)m}{R+\delta})+2mR\]^{\frac{d-1}{2}},
\label{eqn:ineq2}
\end{align}
where $\gamma_d$ is characterized as the ratio between the surface areas of the unit spheres in $(d-1)$ and $d$ dimensions, respectively:
$\gamma_d \coloneqq \frac{1}{d} \frac{\omega_{d-1}}{\omega_d}$.

Since
$M=\sqrt{p}C^{\frac{d-1}{4}}$ is always true for $d, M\in \mathbb{N}_+$, $p\in [0,1]$ and  some real values $C\in\R$,
we have
\bea
5^{\frac{d-1}{2}}C^{\frac{d-1}{2}} m^{-(d-1)}\Bigg\{\frac{1}{\beta}\ln{\[\frac{2\(\sqrt{p}C^{\frac{d-1}{4}}-1\)m}{R+\delta}\]}+\frac{1}{\beta}\Bigg\}^{\frac{d-1}{2}}\le 1.
\label{eqn:ineq3}
\eea

Then, we rearrange above as
\bea
\frac{5C}{m^2\beta}\Bigg\{\ln\[\frac{2\(\sqrt{p}C^{\frac{d-1}{4}}-1\)m}{R+\delta}\]+1\Bigg\}-1 \leq 0,
\eea
and identify
\bea
\label{eqn:identified_ab}
a\coloneqq \frac{4}{d-1}\Bigg\{\ln\[\frac{2m(\sqrt{p}-1)}{R+\delta}\]+1\Bigg\}, 
\quad 
b\coloneqq \frac{4m^2\beta}{5(d-1)}.
\eea

By \cref{lemma:W}, we have
\bea
\label{eqn:c}
C=\frac{b}{W_0({\exp\{a+\ln{b}\}})},
\eea
where $W_0(\cdot)$  is the upper branch of the Lambert $W$ function.
Since the domain of the Lambert $W$ function is $x>(-\nicefrac{1}{e},\infty)$ and {the fact $\exp{a+\ln{b}}>0$}, the solution for \eqref{eqn:c} exists.
When the inequality \eqref{eqn:ineq3} holds, we arrive the lower bound on the exponential storage capacity $M$:
\bea
M
\geq
\sqrt{p}C^{\frac{d-1}{4}}.
\eea
In addition, by the asymptotic expansion of the Lambert $W$ function \cite[Lemma~3.1]{hu2023SparseHopfield}, it also holds
$M\ge M_{\text{Dense}}$, where $M_{\text{Dense}}$ is the memory capacity of the dense modern Hopfield model \cite{ramsauer2020hopfield}.
\end{proof}

\section{Methodology Details}

\subsection{The Multi-Step GSH Updates}

$\mathtt{GSH}$ inherits the capability of multi-step update for better retrieval accuracy, which is summarized in below \cref{alg:multistep} for a given number of update steps $\kappa$.
In practice, we find that a single update suffices, consistent with our theoretical finding \eqref{eqn:eps_sparse_dense} of  \cref{thm:eps_sparse_dense}.

\begin{algorithm}
\caption{Multi-Step Generalized Sparse Hopfield Update}\label{alg:multistep}
\begin{algorithmic}
\Require $\kappa \in \mathbb{R} \geq 1, \bQ \in \R^{\text{len}_Q \times D_Q}, \bY \in \R^{\text{len}_Y \times D_Y}$ 
\For{$i \rightarrow 1 \text{ to } \kappa $}
\State $\bQ^{\text{new}} = \mathtt{GSH}\( \bQ,  \bY \)$ \hfill{\textit{Hopfield Update}}
\State $\bQ \leftarrow \bQ^{\text{new}}$
\EndFor \\
\Return $\bQ$
\end{algorithmic}
\end{algorithm}

\subsection{$\mathtt{GSHPooling}$ and $\mathtt{GSHLayer}$}

Here we provide the operational definitions of the $\mathtt{GSHPooling}$ and the $\mathtt{GSHLayer}$.
\begin{definition}[Generalized Sparse Hopfield Pooling ($\mathtt{GSHPooling}$)]
    Given inputs $\bY \in \mathbb{R}^{\text{len}_Y \times D_Y}$, and $\text{len}_Q$ query patterns $\bQ \in \R^{\text{len}_Q \times D_K}$ the 1-step Sparse Adaptive Hopfield Pooling update is 
    \begin{equation}
        \mathtt{GSHPooling} \( \bY \) = \alphaentmax \( \bQ \bK^T/\sqrt{D_k} \) \bV,
    \end{equation}
    Here we have $\bK, \bV$ equal to $\bV = \bY \mathbf{W}_K \mathbf{W}_V$, $\bK = \bY \bW_K$, and $\mathbf{W}_V \in \mathbb{R}^{D_K \times D_K}, \mathbf{W}_K \in \mathbb{R}^{D_K \times D_K}$. Where $d$ is the dimension of $K$.
    And the query pattern $\bQ$ is a learnable variable, and is independent from the input, the size of $\text{len}_Q$ controls how many query patterns we want to store.
\end{definition}

\begin{definition}[Generalized Sparse Hopfield Layer ($\mathtt{GSHLayer}$)]
    Given inputs $\bY \in \mathbb{R}^{\text{len}_Y \times D_Y}$, and $\text{len}_Q$ query patterns $\bQ \in \R^{\text{len}_Q \times D_K}$ the 1-step Sparse Adaptive Hopfield Layer update is 
    \begin{equation}
        \mathtt{GSHLayer} \( \bR, \bY \) = \alphaentmax \( \bR \bY^T/\sqrt{D_k} \) \bY,
    \end{equation}
    Here $\bR$ is the input and $\bY$ can be either learnable weights or given as an input.
\end{definition}

\subsection{Example: Memory Retrieval for Image Completion}
\label{sec:example_retrieval}
The standard memory retrieval mechanism of Hopfield Models contains two inputs, the query $\bx$ and the associative memory set $\bm{\Xi}$.
The goal is to retrieve an associated memory $\bxi$ most similar to the query $\bx$ from the stored memory set $\bm{\Xi}$.
For example, in \cite{ramsauer2020hopfield}, the query $\bx$ is a corrupted/noisy image from CIFAR10, and the associative memory set $\bm{\Xi}$ is the CIFAR10 image dataset.
All images are flattened into vector-valued patterns.
This task can be achieved by taking the query as $\bR = \bx$ and the associative memory set as $\bY = \bm{\Xi}$ for $\mathtt{GSHLayer}$ with fixed parameters.
After steps of updates, we expect the output of the $\mathtt{GSHLayer}$ to be the recovered version of $\bx$.

\subsection{Pseudo Label Retrieval}

Here, we present the use of the memory retrieval mechanism from modern Hopfield models to generate pseudo-labels for queries $\bR$, thereby enhancing predictions.
Given a set of memory patterns $\bY$ and their corresponding labels $\bY_{\text{label}}$, we concatenate them together to form the \textit{label-included} memory set $\Tilde{\bY}$.
Take CIFAR10 for example, we can concatenate the flatten images along with their one-hot encoded labels together as the memory set.
For the query, we use the input with padded zeros concatenated at the end of it.
The goal here is to ``retrieve" the padding part in the query, which is expected to be the retrieved ``pseudo-label".
Note that this pseudo-label will be a weighted sum over all other labels in the associative memory set.
An illustration of this mechanism can be found in \cref{fig:StanHop}.
For the retrieved pseudo-label, we can either use it as the final prediction, or use it as pseudo-label to provide extra information for the model.

\subsection{Algorithm for STanHop-Net}
Here we summarize the STanHop-Net as below algorithm.
\begin{algorithm}[h]
\caption{STanHop-Net}\label{alg:stacked_STanHop}
\begin{algorithmic}
\Require $L \geq 1, \bZ \in \R^{T \times C \times D_{hidden}^0}$
\For{$\ell \rightarrow 1 \text{ to } L$}
\State $\bZ_{\text{enc}}^\ell = \mathtt{STanHop} \(\text{Coarse-Graining}\( \bZ_{\text{enc}}^{\ell-1}, \Delta \)\)$ \hfill{\textit{encoder forward}}
\EndFor
\State $\bZ_{\text{dec}}^0 = \bE_{\text{dec}}$ \hfill{\textit{learnable positional embedding}}
\For{$\ell \rightarrow 1 \text{ to } L$} \hfill{\textit{decoder forward}}
\State $\Tilde{\bZ}_{\text{dec}}^\ell = \mathtt{STanHop}\( \bZ_{\text{dec}}^{\ell-1} \)$ 
\State $\hat{\bZ}_{\text{dec}}^\ell = \mathtt{GSH}\( \bZ_{\text{dec}}^\ell, \bZ_{\text{enc}}^\ell \) $
\State $\check{\bZ}_{\text{dec}}^\ell =\rm{LayerNorm} 
\(\hat{\bZ}_{\text{dec}}^\ell + \Tilde{\bZ}_{\text{dec}}^\ell\)$
\State $\bZ_{\text{dec}}^\ell =\rm{LayerNorm} \( \check{\bZ}_{\text{dec}}^\ell + \text{\rm{MLP}}\( \check{\bZ}_{\text{dec}}^\ell\) \)$
\EndFor \\
\Return $\bZ_{\text{dec}}^L \in \R^{ \frac{P}{T}
\times C \times D_{\text{hidden}}}$
\end{algorithmic}
\end{algorithm}

\clearpage
\section{Additional Numerical Experiments}
Here we provide additional experimental investigations to back up the effectiveness of our method.

\subsection{Numerical Verification's of Theoretical Results}

\paragraph{Faster Fixed Point Convergence and Better Generalization.}
In \cref{fig:loss_curves}, to support our theoretical results in \cref{sec:method}, 
we numerically analyze the convergence behavior of the $\mathtt{GSH}$, compared with the dense modern Hopfield layer $\mathtt{Hopfield}$.

\begin{figure*}[h]
    \centering
    \includegraphics[width=1\textwidth]{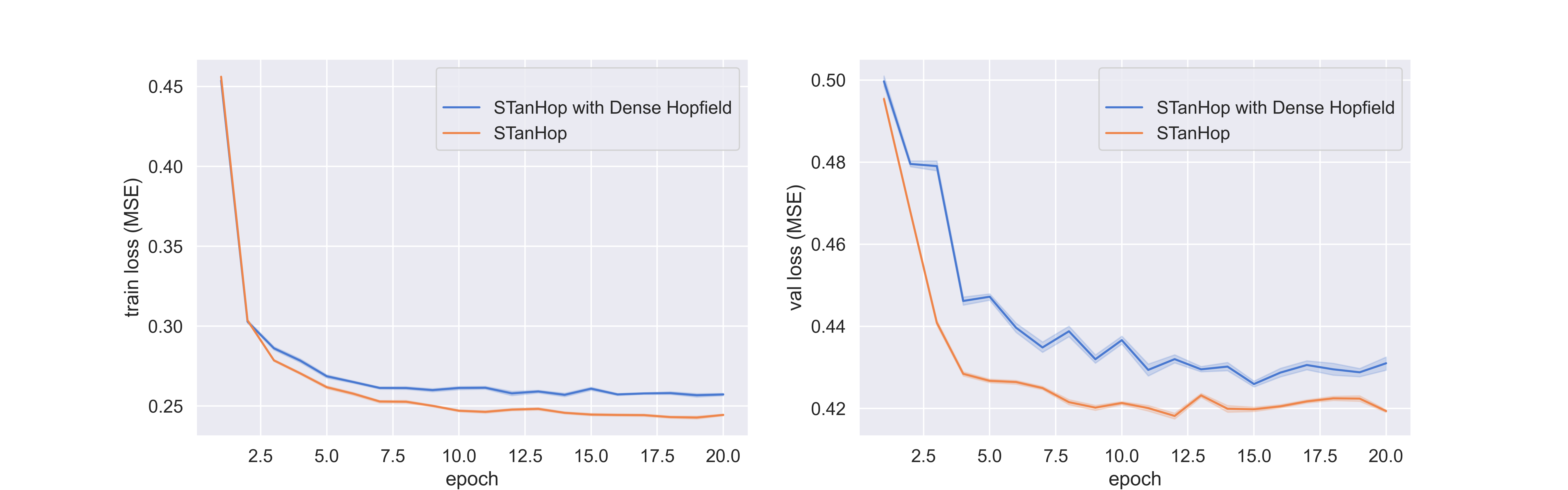}
    \vspace{-2em}
    \caption{\small
    The training and validation loss curves of STanHop (D), i.e. STanHop-Net with dense modern Hopfield $\mathtt{Hopfield}$ layer, and STanHop-Net with $\mathtt{GSH}$ layer.
    The results show that the generalized sparse Hopfield model enjoys faster convergence than the dense model and also obtain better 
    generalization.
    }
    \label{fig:loss_curves}
\end{figure*}

In \cref{fig:loss_curves}, we plot the loss  curves for STanHop-Net using both generalized sparse and dense modern models on the ETTh1 dataset for the multivariate time series prediction tasks.

The results reveal that the generalized sparse Hopfield model ($\mathtt{GSH}$) converges faster than the dense model ($\mathtt{Hopfield}$) and also achieves better generalization. This empirically supports our theoretical findings presented in \cref{thm:eps_sparse_dense}, which suggest that the generalized sparse Hopfield model provides faster retrieval convergence with enhanced accuracy.

\paragraph{Memory Capacity and Noise Robustness.}
Following \cite{hu2023SparseHopfield}, we also conduct experiments verifying our memory capacity and noise robustness theoretical results (\cref{thm:memory_capacity} and \cref{thm:eps_sparse_dense}), and report the results in \cref{fig:capacity_robustness}.
The plots present average values and standard deviations derived from 10 trials.

\begin{figure*}[h!]
    \centering
    \includegraphics[width=0.48\textwidth]{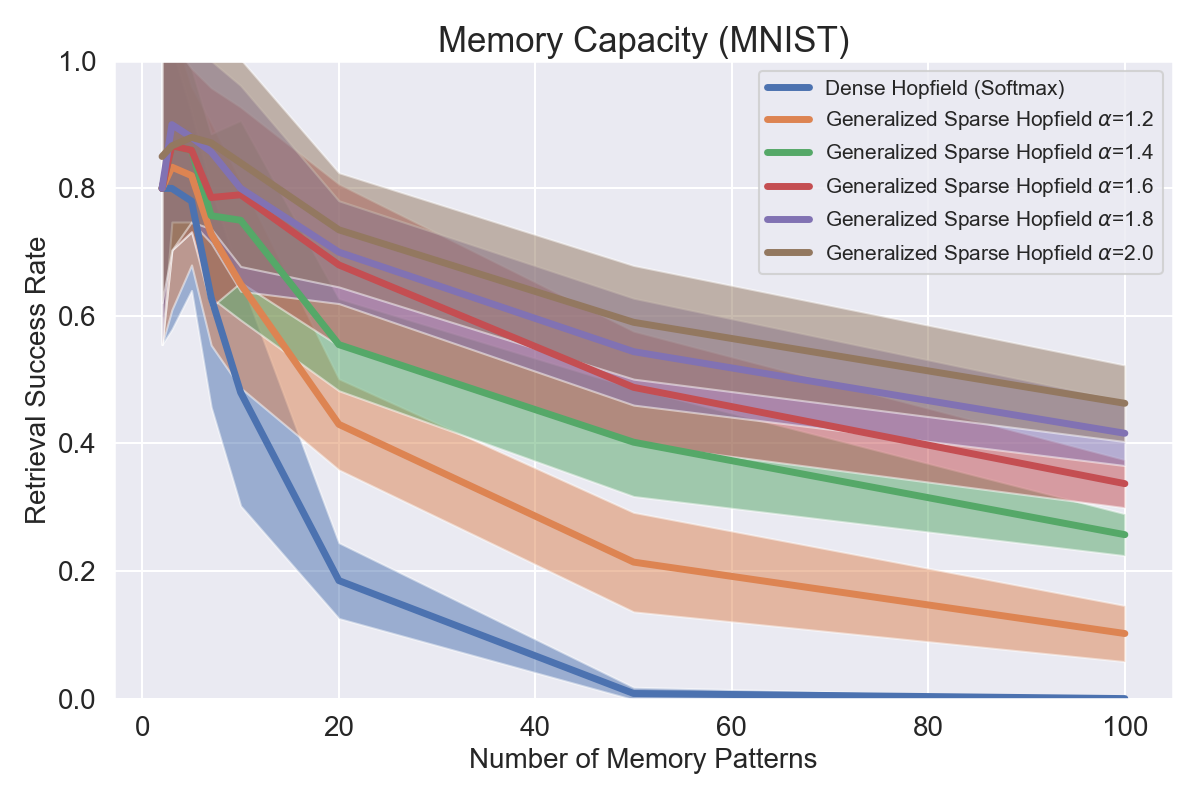}
    \hfill
    \includegraphics[width=0.48\textwidth]{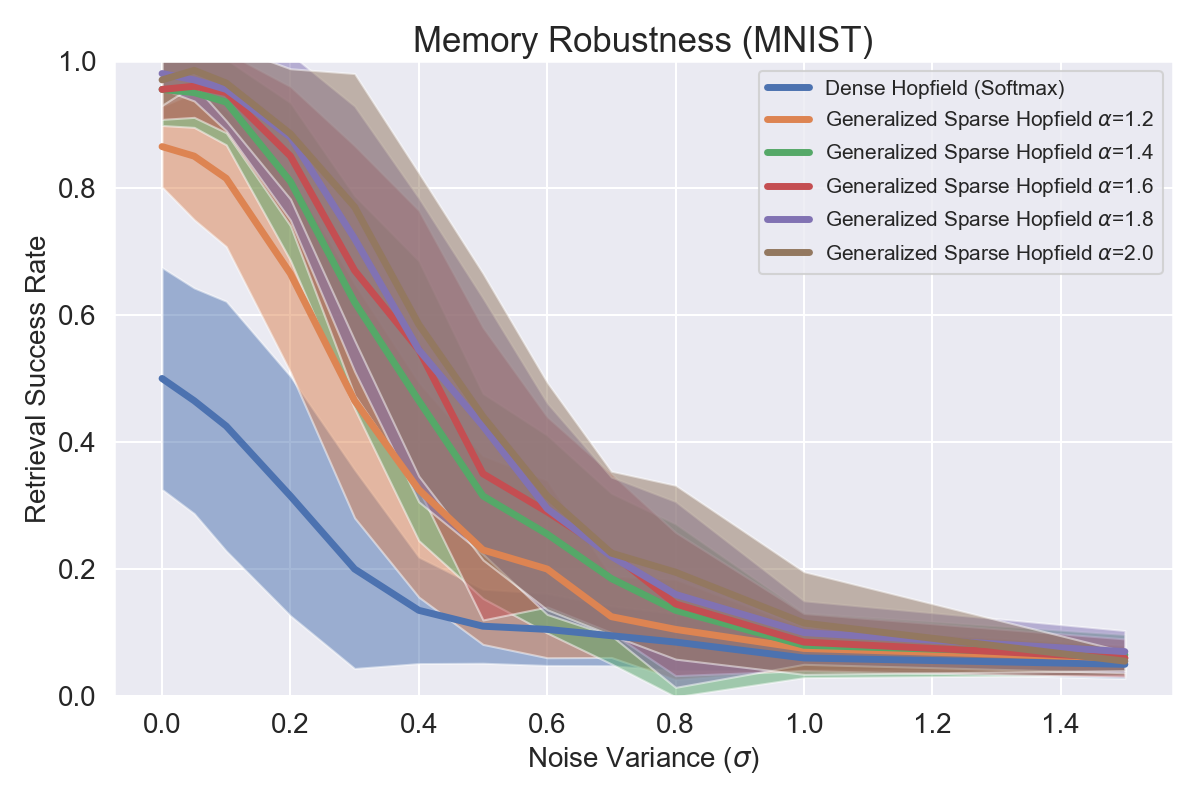}
    \vspace{-1em}
    \caption{\small
    \textbf{Left:} Memory Capacity measured by successful half-masked retrieval rates. 
    \textbf{Right:} Memory Robustness measured by retrieving patterns with various noise levels.
    A query pattern is considered accurately retrieved if its cosine similarity error falls below a specified threshold. 
    We set  error threshold of 20\% and $\beta$=0.01 for better visualization.
    We plot the average and variance from 10 trials. 
    These findings demonstrate the generalized sparse Hopfield model's ability of capturing data sparsity, improved memory capacity and its noise robustness.
    }
    \label{fig:capacity_robustness}
\end{figure*}

Regarding memory capacity (displayed on the left side of \cref{fig:capacity_robustness}), we evaluate the generalized sparse Hopfield model's ability to retrieve half-masked patterns from the MNIST dataset, in comparison to the Dense modern Hopfield model \cite{ramsauer2020hopfield}.

Regarding robustness against noisy queries (displayed on the right side of \cref{fig:capacity_robustness}), we introduce Gaussian noises of varying variances ($\sigma$) to the images.

These findings demonstrate the generalized sparse Hopfield model's ability of capturing data sparsity, improved memory capacity and its noise robustness.

\subsection{Computational Cost Analysis of Memory Modules}
Here we analyze the computational cost between the \textbf{Plug-and-Play} memory plugin module and the baseline.
We evaluate 2 matrices: (i) the number of floating point operations (flops) (ii) number of parameters of the model.
Note that for \textbf{Plug-and-Play} module, the parameter amount will not be affected by the size of external memory set.
The result can be found in \cref{fig:flops} and \cref{fig:params}.

\begin{figure*}[h]
    \centering
    \begin{minipage}{0.48\textwidth}
        \includegraphics[width=\textwidth]{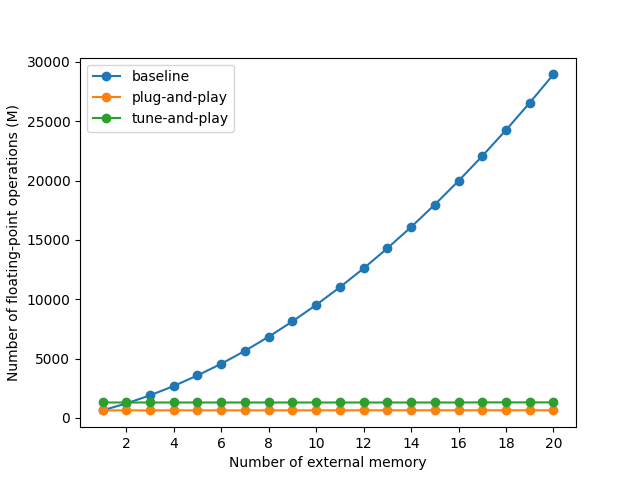}   
        \caption{The number of floating-point operations (flops) (in millions) comparison between \textbf{Plug-and-Play}, \textbf{Tune-and-Play} and the baseline. The result shows that the \textbf{Plug-and-Play}, \textbf{Tune-and-Play}  successfully reduce the required computational cost to process an increased amount of data.}
        \label{fig:flops}
    \end{minipage}\hfill
    \begin{minipage}{0.48\textwidth}
        \includegraphics[width=\textwidth]{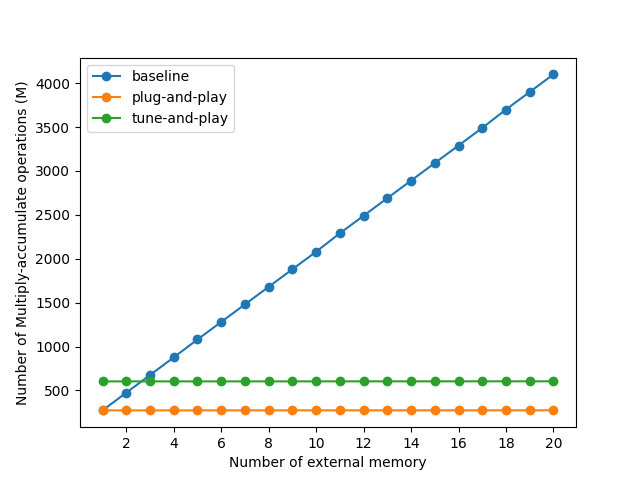}
        \caption{The number of Multiply–accumulate operations (MACs) (in millions) comparison between \textbf{Plug-and-Play}, \textbf{Tune-and-Play} and the baseline. The result shows that both of our memory plugin modules face little MACs increasement while the baseline model MACs increase almost linearly w.r.t. the input size. }
        \label{fig:params}
    \end{minipage}
\end{figure*}

\subsection{Ablation Studies}

\paragraph{Hopfield Model Ablation.}
Beside our proposed generalized sparse modern Hopfield model, we also test STanHop-Net with 2 other existing different modern Hopfield models: the dense modern Hopfield model \cite{ramsauer2020hopfield} and the sparse modern Hopfield model \cite{hu2023SparseHopfield}.
We report their results in \cref{table:result}.

We terms them as \textbf{STanHop-Net (D)} and \textbf{STanHop-Net (S)} where (D) and (S) are for ``Dense'' and ``Sparse'' respectively.

\paragraph{Component Ablation.}
In order to evaluate the effectiveness of different components in our model, we perform an ablation study by removing one component at a time.
In below, we denote Patch Embedding as (PE), StanHop as (SH), Hopfield Pooling as (HP), Multi-Resolution as (MR).
We also denote their removals with ``w/o'' (i.e., without.)

For w/o PE, we set the patch size $P$ equals 1.
For w/o MR, we set the coarse level $\Delta$ as 1.
For w/o SH and w/o HP, we replace those blocks/layers with an MLP layer with GELU activation and layer normalization.

The results are showed in \cref{tab:ab}.
From the ablation study results, we observe that removing the STanHop block gives the biggest negative impact on the performance.
Showing that the STanHop block contributes the most to the model performance.
Note that patch embedding also provides a notable improvement on the performance.
Overall, every component provides a different level of performance boost.

\begin{table}[h!]
    \centering
    \caption{\small
    \textbf{Component Ablation.}
    We conduct component ablation by separately removing Patch Embedding (PE), STanHop (SH), Hopfield Pooling (HP), and Multi-Resolution (MR).
    We report the mean MSE and MAE over 10 runs, with variances omitted as they are all $\le 0.15$\%.
    The results indicate that while every single component in STanHop-Net provides performance boost, the impact of STanHop block on model performance is the most significant among all other components.
    }
    \label{tab:ab}
    \resizebox{\textwidth}{!}{    
    \begin{tabular}{lccccccccccc}
    \toprule
     \multicolumn{2}{c}{Models} & \multicolumn{2}{c}{STanHop} & \multicolumn{2}{c}{w/o PE} & \multicolumn{2}{c}{w/o SH} & \multicolumn{2}{c}{w/o HP} & \multicolumn{2}{c}{w/o MR} \\
    \midrule
    & Metric & MSE & MAE & MSE & MAE & MSE & MAE & MSE & MAE & MSE & MAE  \\
    \midrule
    \multirow{5}{1em}{\rot{ETTh1}} 
    & 24 & 0.294 & 0.360 & 0.318 & 0.368 & 0.305 & 0.365 & 0.306 & 0.363 & 0.307 & 0.363 \\
    & 48 &  0.340 & 0.387 & 0.357 & 0.389 & 0.352 & 0.393 & 0.346 & 0.387 & 0.348 & 0.385\\
    & 168 &  0.420 & 0.452 & 0.454 & 0.476 & 0.480 & 0.500 & 0.434 & 0.455 & 0.447 & 0.464 \\
    & 336 & 0.450 & 0.472 & 0.501 & 0.524 & 0.530 & 0.535 & 0.462 & 0.473 & 0.482 & 0.486 \\
    & 720 &  0.512 & 0.520 & 0.540 & 0.538 & 0.610 & 0.581 & 0.524 & 0.526 & 0.537 & 0.531 \\
    \midrule
    \multirow{5}{1em}{\rot{WTH}} 
    & 24 & 0.292 & 0.341 & 0.318 & 0.365 & 0.335 & 0.375 & 0.340 & 0.374 & 0.325 & 0.373 \\
    & 48 & 0.363 & 0.402 & 0.386 & 0.421 & 0.414 & 0.439 & 0.385 & 0.420 & 0.391 & 0.427 \\
    & 168  & 0.499 & 0.515 & 0.504 & 0.521 & 0.507 & 0.519 & 0.503 & 0.525 & 0.520 & 0.509 \\
    & 336 & 0.499 & 0.515 & 0.514 & 0.529 & 0.532 & 0.541 &  0.513 & 0.528 & 0.533 & 0.542 \\
    & 720 & 0.548 & 0.556 & 0.570 & 0.565 & 0.569 & 0.565 & 0.539 & 0.548 & 0.555 & 0.557 \\
    \bottomrule
    \end{tabular}
    }
\end{table}
\clearpage

\subsection{The Impact of Varying $\alpha$}
We examine the impact of increasing the value of $\alpha$ on memory capacity and noise robustness. 
It is known that as $\alpha$ approaches infinity, the $\alpha$-entmax operation transitions to a hardmax operation \cite{peters2019sparse,correia2019adaptively}. Furthermore, it is also known that memory pattern retrieval using hardmax is expected to exhibit perfect retrieval ability, potentially offering a larger memory capacity than the softmax modern Hopfield model in pure retrieval tasks \cite{millidge2022universal}.
Our empirical investigations confirm that higher values of $\alpha$ frequently lead to higher memory capacity.
We report results only up to $\alpha=5$, as we observed that values of $\alpha$ greater than 5 consistently lead to numerical errors, especially under float32 precision.
It is crucial to note, that while the hardmax operation (realized when $\alpha \rightarrow \infty$) may maximize memory capacity, its lack of differentiability renders it unsuitable for gradient descent-based optimization.

\begin{figure*}[h!]
    \centering
    \includegraphics[width=0.48\textwidth]{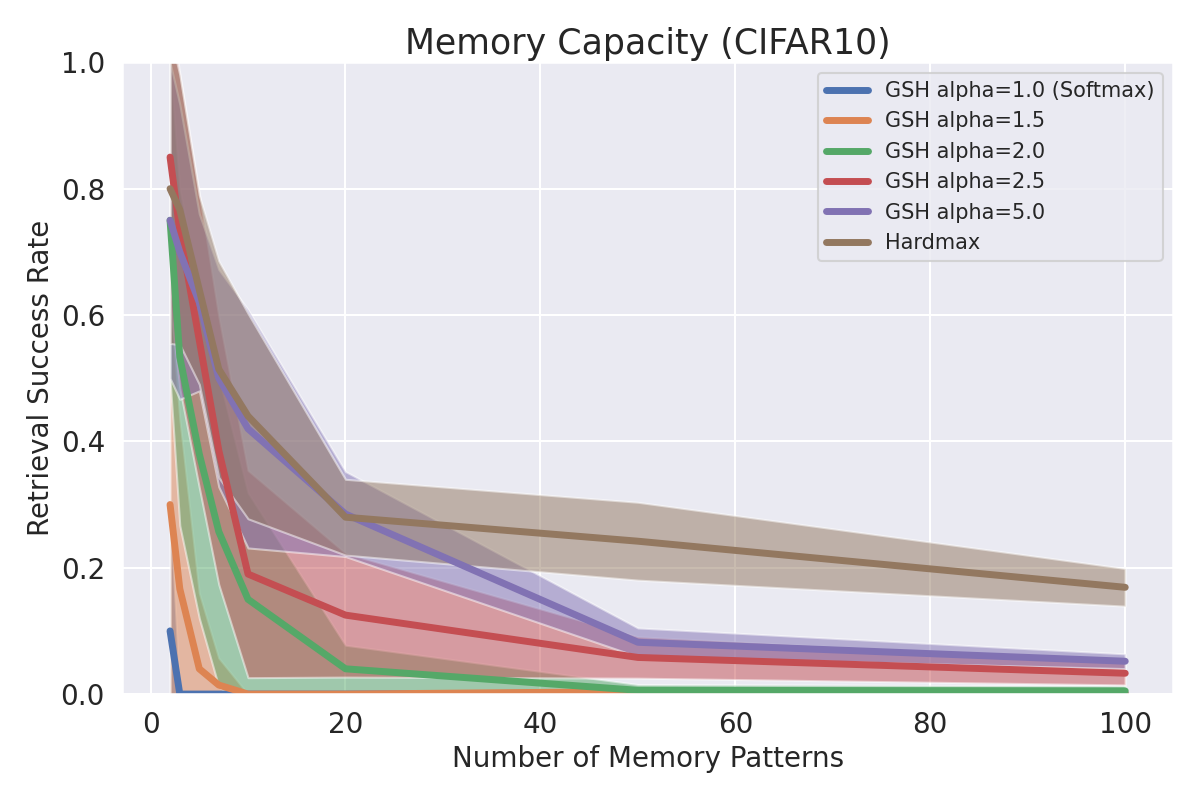}
    \hfill
    \includegraphics[width=0.48\textwidth]{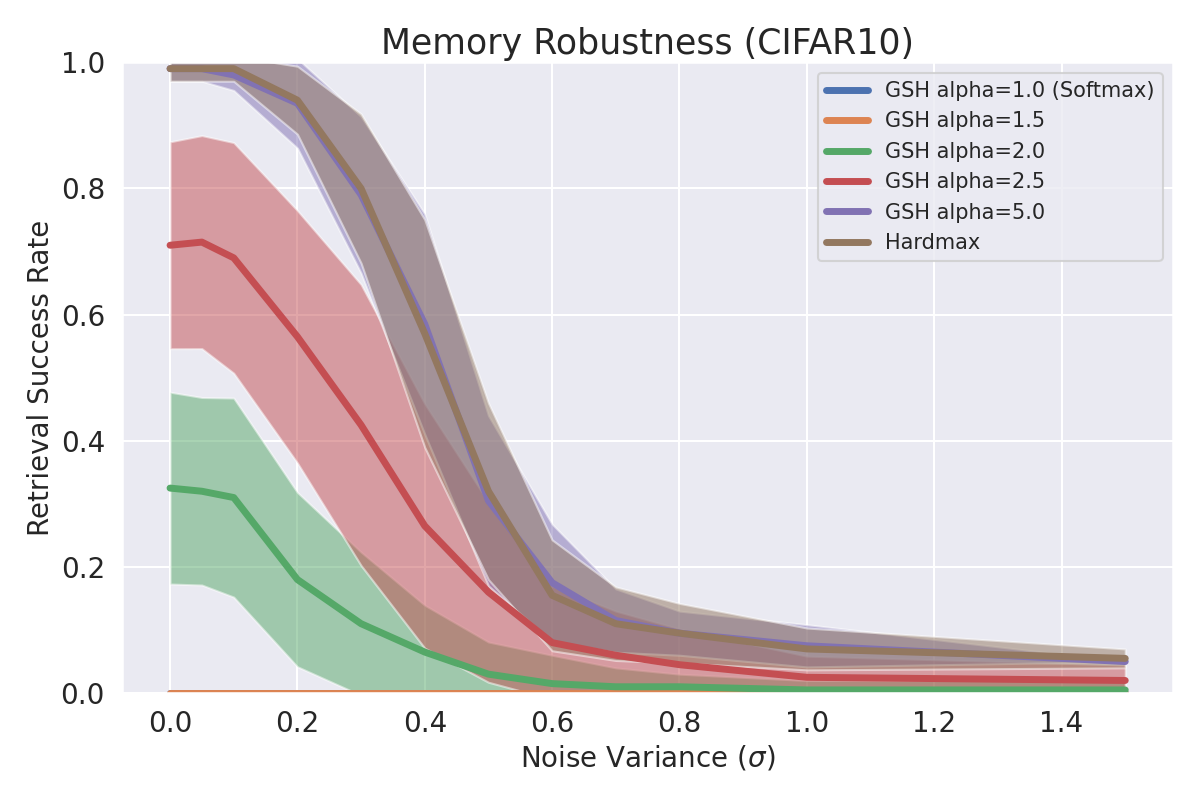}
    \vspace{-1em}
    \caption{\small
    \textbf{Left:} Memory Capacity measured by successful half-masked retrieval rates w.r.t. different values of $\alpha$ on CIFAR10. 
    \textbf{Right:} Memory Robustness measured by retrieving patterns with various noise levels on CIFAR10.
    A query pattern is considered accurately retrieved if its cosine similarity error falls below a specified threshold. 
    We set  error threshold of 20\% and $\beta=0.01$ for better visualization.
    We plot the average and variance from 10 trials. 
    We can see that using hardmax (argmax) normally gives the best retrieval result as it retrieves only the most similar pattern w.r.t. dot product distance. 
    And setting $\alpha = 5$ approximately gives the similar result while having $\alpha = 5$ keeps the overall mechanism differentiable. 
    }

    \label{fig:alpha_cifar}
\end{figure*}

\begin{figure*}[h!]
    \centering
    \includegraphics[width=0.48\textwidth]{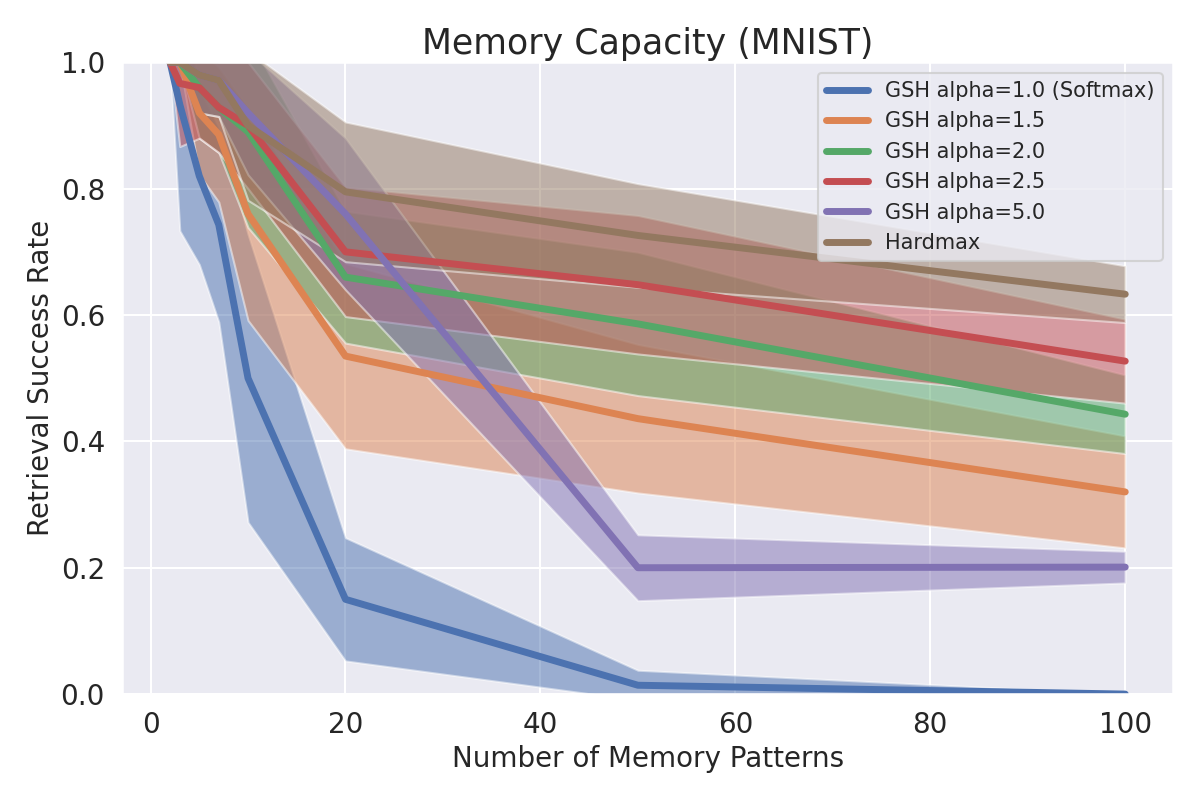}
    \hfill
    \includegraphics[width=0.48\textwidth]{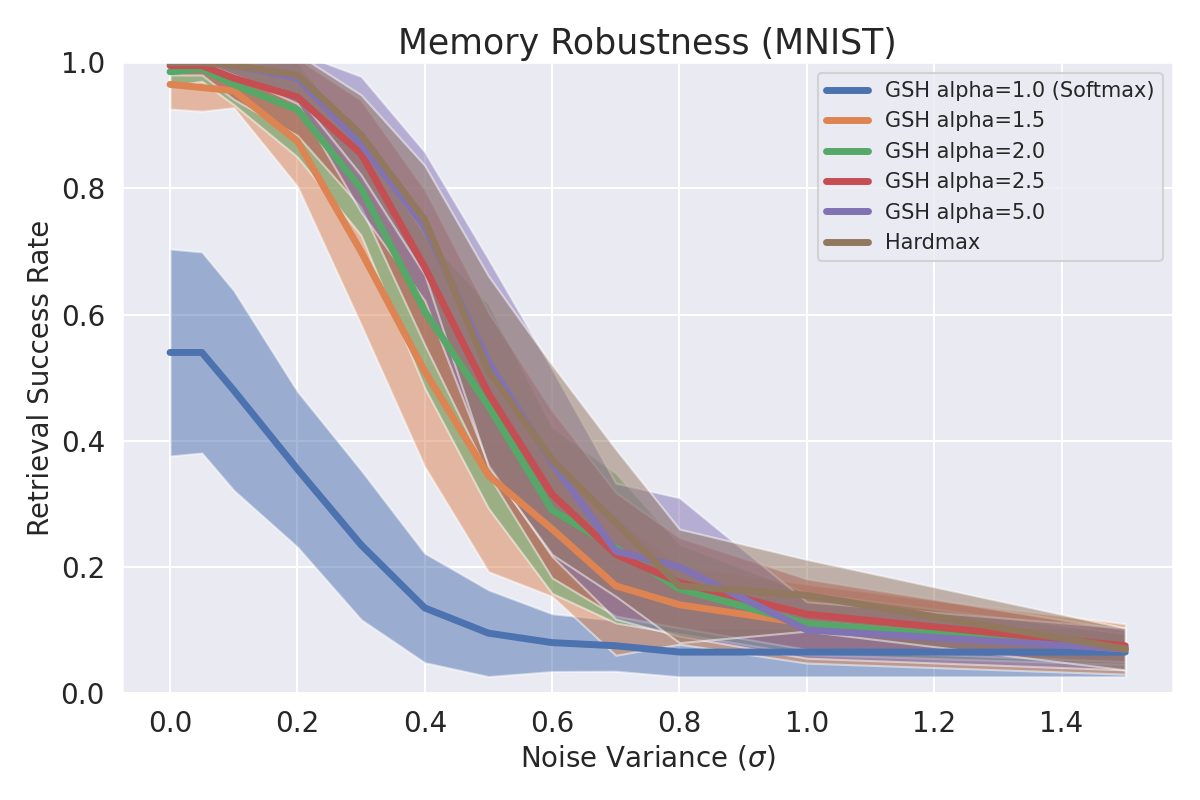}
    \vspace{-1em}
    \caption{\small
    \textbf{Left:} Memory Capacity measured by successful half-masked retrieval rates w.r.t. different values of $\alpha$ on MNIST. 
    \textbf{Right:} Memory Robustness measured by retrieving patterns with various noise levels on MNIST.
    A query pattern is considered accurately retrieved if its cosine similarity error falls below a specified threshold. 
    We set  error threshold of 20\% and $\beta=0.1$ for better visualization.
    We plot the average and variance from 10 trials. 
    We can see that using hardmax (argmax) normally gives the best retrieval result as it retrieves only the most similar pattern w.r.t. dot product distance. 
    And setting $\alpha = 5$ approximately gives the similar result while having $\alpha = 5$ keeps the overall mechanism differentiable. 
    }

    \label{fig:alpha_mnist}
\end{figure*}

\clearpage
\subsection{Memory Usage of STanHop and learnable $\alpha$}
To compare the memory and GPU usage, we benchmark STanHop with STanHop-(D) (using \textit{dense} modern Hopfield layers).
In below \cref{fig:gpu_usage} and \cref{fig:usage_gpu_utl}, we report the footprints of  from the weight and bias (wandb) system \cite{biewald2020experiment}.
The figures clearly demonstrate that the computational cost associated with learning an additional $\alpha$ for adaptive sparsity is negligible.

\begin{figure*}[h]
    \centering
    \begin{minipage}{0.48\textwidth}
        \includegraphics[width=\textwidth]{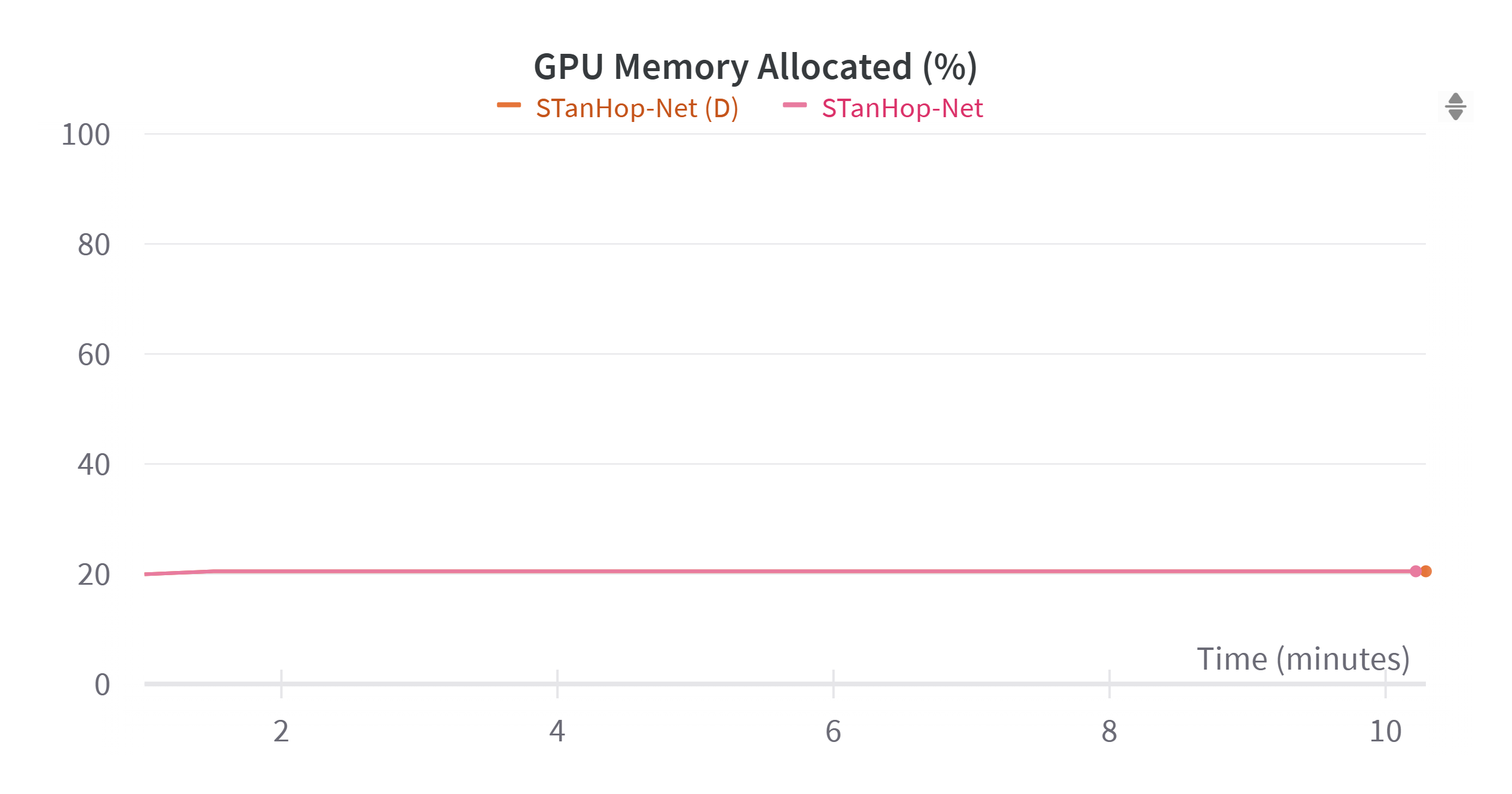}   
        \caption{The GPU memory allocation between STanHop-Net with and without learnable alpha (STanHop-Net (D)). We can see that with learnable alpha does not significantly increase reuired GPU memory.}
        \label{fig:gpu_usage}
    \end{minipage}\hfill
    \begin{minipage}{0.48\textwidth}
        \includegraphics[width=\textwidth]{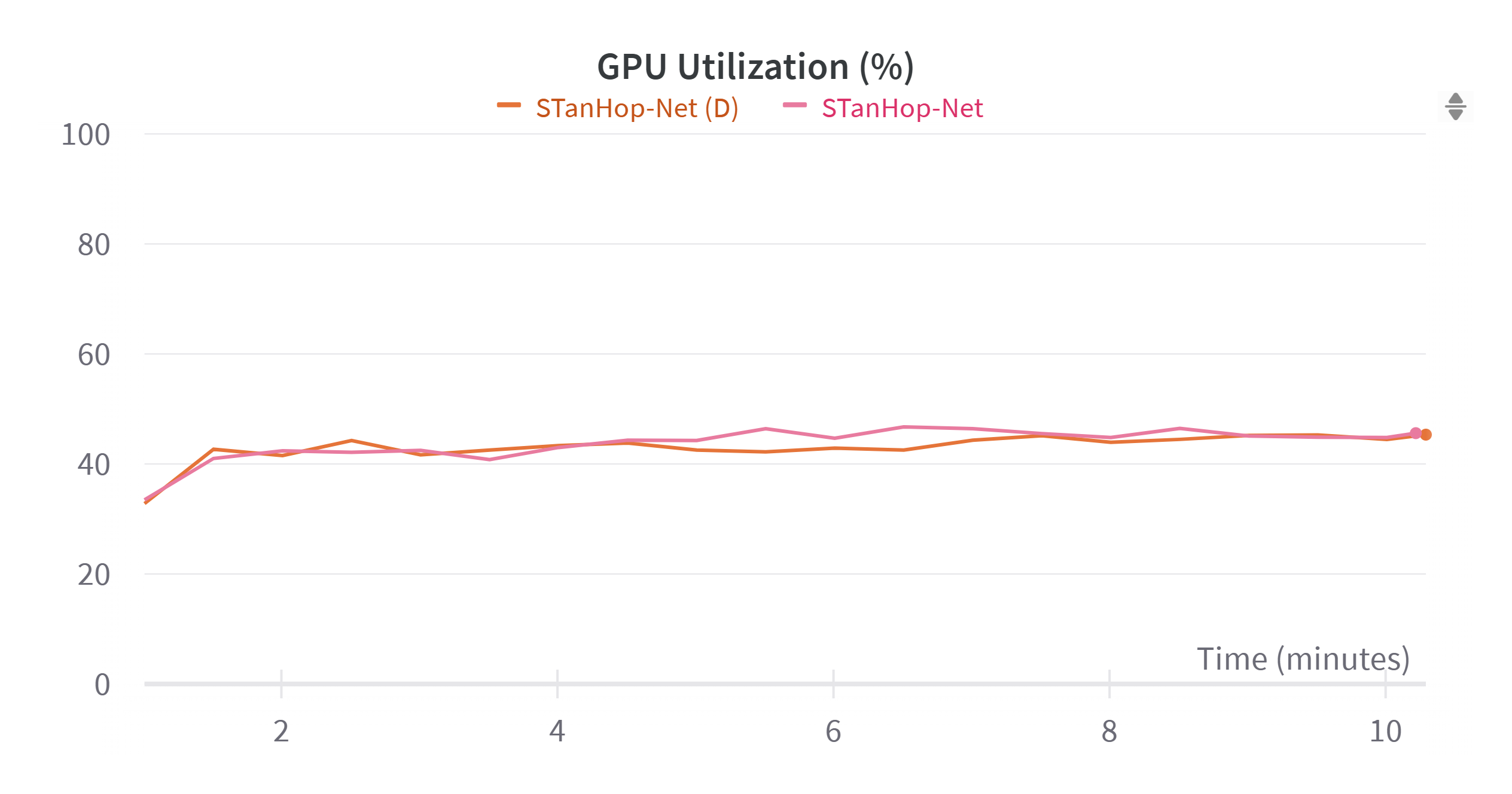}
        \caption{The percentage of gpu utilization between STanHop-Net with and without learnable alpha (STanHop-Net (D)). We can see that with learnable alpha does not significantly increase the GPU utilization.}
        \label{fig:usage_gpu_utl}
    \end{minipage}
\end{figure*}

\subsection{Time Complexity Analysis of STanHop-Net}
Here we use the same notation as introduced in the main paper.
Let $T$ be the number of input length on the time dimension, $D_{\text{hidden}}$ be the hidden dimension, $P$ be patch size, $D_{\text{out}}$ be the prediction horizon, $\text{len}_Q$ be the size of query pattern in $\mathtt{GSHPooling}$.

\begin{itemize}
    \item Patch embedding: $\mathcal{O}(D_{\text{hidden}} \times  T)  = \mathcal{O}(T)$.
    \item Temporal and Cross series GSH: $\mathcal{O}( D_{\text{hidden}}^2 \times T^2 \times P^{-2} ) = \mathcal{O}( T^2 \times P^{-2}) = \mathcal{O}(T^2)$
    \item Coarse graining: $\mathcal{O}(  D_{\text{hidden}}^2 \times T ) = \mathcal{O}(T)$
    \item GSHPooling : $\mathcal{O}( D_{\text{hidden}} \times \text{len}_Q \times  T^2 \times P^{-2}) = \mathcal{O}(\text{len}_Q \times T^2)$
    \item PlugMemory: $\mathcal{O}(T^2)$
    \item TuneMemory: $\mathcal{O}( (T+D_{\text{out}})^2 )$
\end{itemize}

Additionally, the number of parameters of STanHop-Net 0.78 million.
As a reference, with batch size of 32, input length of 168, STanHop-Net requires 2 minutes per epoch for the ETTh1 dataset.
Meanwhile, STanHop-Net (D) also requires 2 minutes per epoch under same setting.

\subsection{Hyperparameter Sensitivity Analysis}
\label{sec:stability}
We conduct experiments exploring the parameter sensitivity of  STanHop-Net on the ILI dataset. 
Our results (in \cref{tab:stability1} and \cref{tab:stability2}) show that STanHop-Net is not sensitive to hyperparameter changes.

The way that we conduct the hyperparameter sensitivity analysis is that to show the model's sensitivity to a hyperparameter $h$, we change the value of $h$ in each run and keep the rest of the hyperparameters' values as default and we record the MAE and MSE on the test set. We train and evaluate the model 3 times for 3 different values for each hyperparameter. We analyze the model's sensitivity to 7 hyperparameters respectively. We conduct all the experiments on the ILI dataset.

\begin{table*}[h]
\centering
\caption{Default Values of Hyperparameters in Sensitivity Analysis}

\begin{tabular}{lc}
\toprule
Parameter & Default Value \\
\midrule
seg\_len (patch size) & 6 \\
window\_size (coarse level) & 2 \\
e\_layers (number of encoder layers) & 3 \\
d\_model & 32 \\
d\_ff (feedforward dimension) & 64 \\
n\_heads (number of heads) & 2 \\
\bottomrule
\end{tabular}
\end{table*}

\begin{table}[h]
\caption{MAEs for each value of the hyperparameter with the rest of the hyperparameters as default. For each hyperparameter, each row's MAE score corresponds to the hyperparameter value inside the parentheses in the same order. For example, when lr is 1e-3, the MAE score is 1.313.}
\resizebox{\textwidth}{!}{ 
\begin{tabular}{cccccccccc}
\toprule
lr & seg\_len & window\_size & e\_layers & d\_model & d\_ff & n\_heads \\

(1e-3, 1e-4, 1e-5) & (6, 12, 24) & (2,4,8) & (3,4,5) & (32, 64, 128) & (64, 128, 256) & (2,4,8) \\
\midrule
1.313 & 1.313 & 1.313 & 1.313 & 1.313 & 1.313 & 1.313 \\
1.588 & 1.311 & 1.285 & 1.306 & 1.235 & 1.334 & 1.319 \\
1.673 & 1.288 & 1.368 & 1.279 & 1.372 & 1.302 & 1.312 \\ 

\bottomrule
\end{tabular}
\label{tab:stability1}
}
\end{table}

\begin{table}[h]
\caption{MSEs for each value of the hyperparameter with the rest of the hyperparameters as default. For each hyperparameter, each row's MSE score corresponds to the hyperparameter value inside the parentheses in the same order. For example, when lr is 1e-3, the MSE score is 3.948.}
\resizebox{\textwidth}{!}{ 
\begin{tabular}{cccccccccc}
\toprule
lr & seg\_len & window\_size & e\_layers & d\_model & d\_ff & n\_heads \\

(1e-3, 1e-4, 1e-5) & (6, 12, 24) & (2,4,8) & (3,4,5) & (32, 64, 128) & (64, 128, 256) & (2,4,8) \\
\midrule
3.948 & 3.948 & 3.948 & 3.948 & 3.948 & 3.948 & 3.948 \\
5.045 & 3.968 & 3.877 & 3.915 & 3.566 & 3.983 & 3.967 \\
5.580 & 3.865 & 4.267 & 3.834 & 4.078 & 3.866 & 3.998 \\ 

\bottomrule
\end{tabular}
}
\label{tab:stability2}
\end{table}

\clearpage
\subsection{Additional Multiple Instance Learning Experiments}
\label{sec:MIL}

We also evaluate the efficacy of the proposed $\mathtt{GSH}$ layer on Multiple Instance Learning (MIL) tasks.
In essence, MIL is a type of supervised learning whose training data are divided into bags and labeling
individual data points is difficult or impractical, but bag-level labels are available \cite{ilse2018attention}.
We follow \cite{ramsauer2020hopfield, hu2023SparseHopfield} to conduct the multiple instance learning experiments on MNIST.

\paragraph{Model.}
We first flatten each image and use a fully connected layer to project each image to the embedding space.
Then, we perform $\mathtt{GSHPooling}$ and use linear projection for prediction.
\paragraph{Hyperparameters.}
For hyperparameters, we use hidden dimension of $256$, number of head as $4$, dropout as $0.3$, training epoch as $100$, optimizer as AdamW, initial learning rate as $1e-4$, and we also use the cosine annealing learning rate decay.
\paragraph{Baselines.}
We benchmark the Generalized Sparse modern Hopfield model (GSH) with the Sparse \cite{hu2023SparseHopfield} and Dense \cite{ramsauer2020hopfield} modern Hopfield models.
\paragraph{Dataset.}
For the training dataset, we randomly sample 1000 positive and 1000 negative bags for each bag size.
For test data, we randomly sample 250 positive and 250 negative bags for each bag size.
We set the positive signal to the images of digit $9$, and the rest as negative signals.
We vary the bag size and report the accuracy and loss curves of 10 runs on both training and test data.

\paragraph{Results.}
Our results (shown in the figures below) demonstrate that $\mathtt{GSH}$ converges faster than the baselines in most settings, as it can adapt to varying levels of data sparsity. This is consistent with our theoretical findings in \cref{thm:eps_sparse_dense}, which state that $\mathtt{GSH}$ achieves higher accuracy and faster fixed-point convergence compared to the dense model.

\begin{figure*}[h!]
    \centering
    \includegraphics[width=0.24\textwidth]{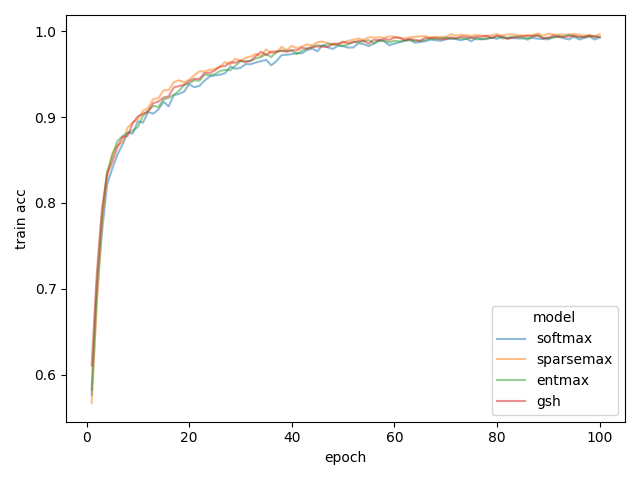}
    \hfill
    \includegraphics[width=0.24\textwidth]{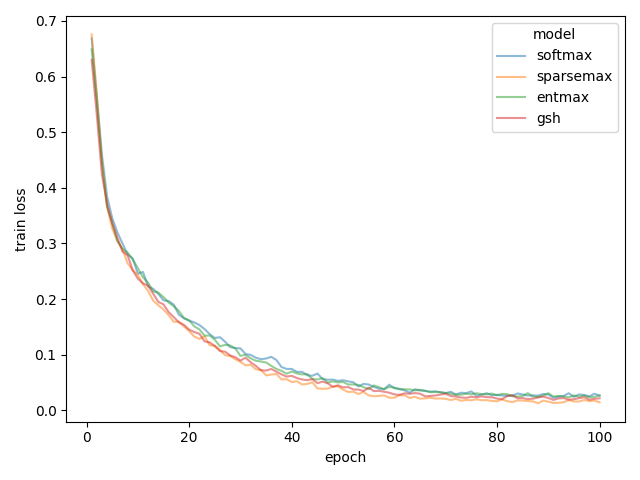}
    \hfill
    \includegraphics[width=0.24\textwidth]{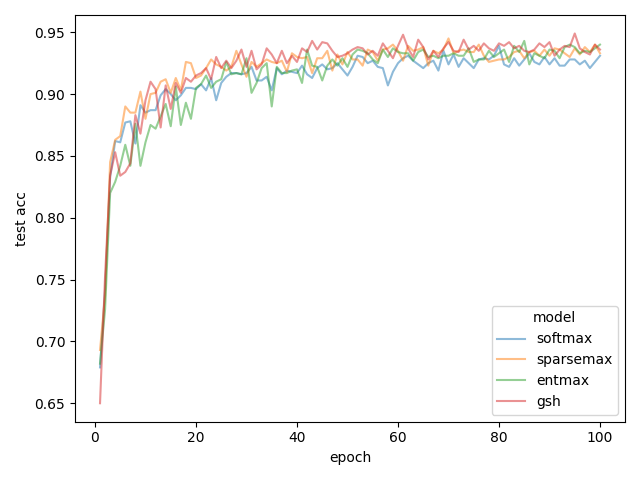}
    \hfill
    \includegraphics[width=0.24\textwidth]{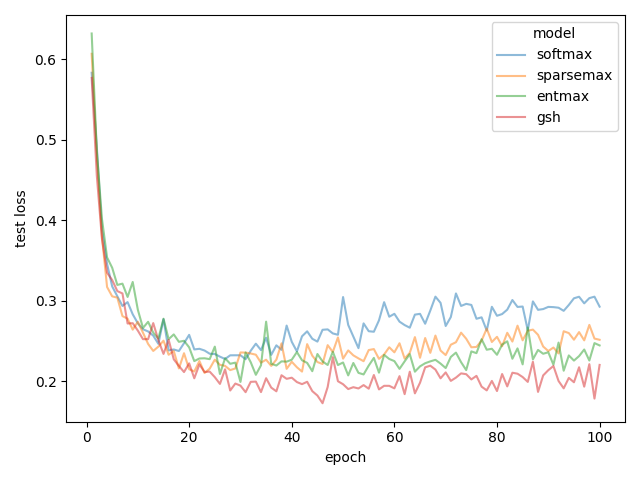}
    \vspace{-1em}
    \caption{\small The MIL experiment with bag size 5.
    From left to right:
    (1) Training data accuracy curve
    (2) Training data loss curve
    (3) Test data accuracy curve
    (4) Test data accuracy curve
    }
    \label{fig:gsh_mil_5}
\end{figure*}

\begin{figure*}[h!]
    \centering
    \includegraphics[width=0.24\textwidth]{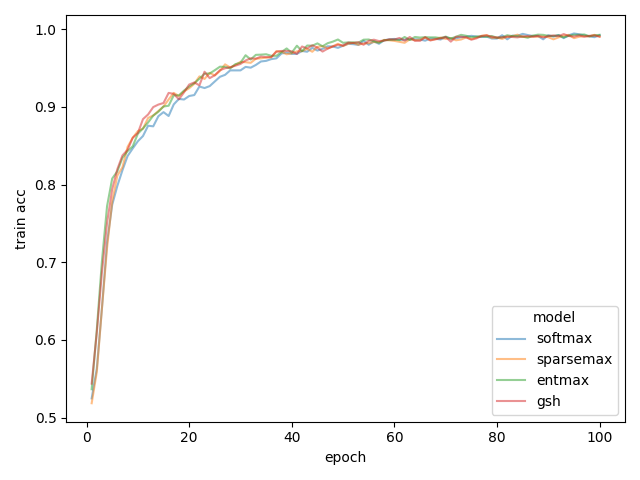}
    \hfill
    \includegraphics[width=0.24\textwidth]{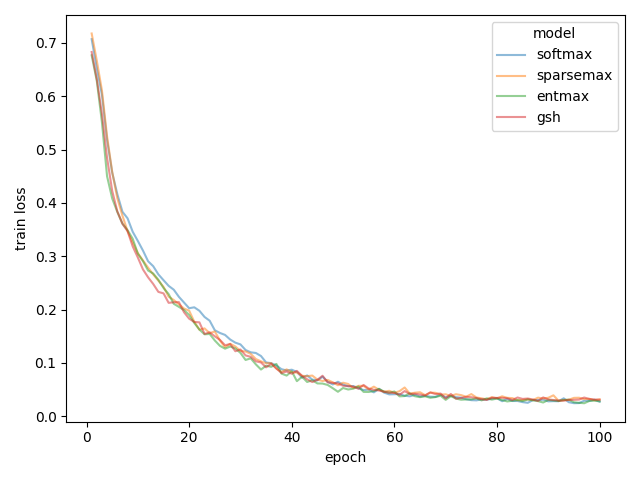}
    \hfill
    \includegraphics[width=0.24\textwidth]{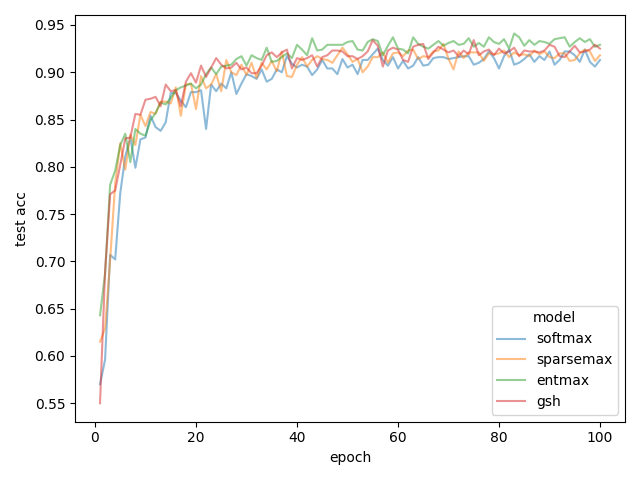}
    \hfill
    \includegraphics[width=0.24\textwidth]{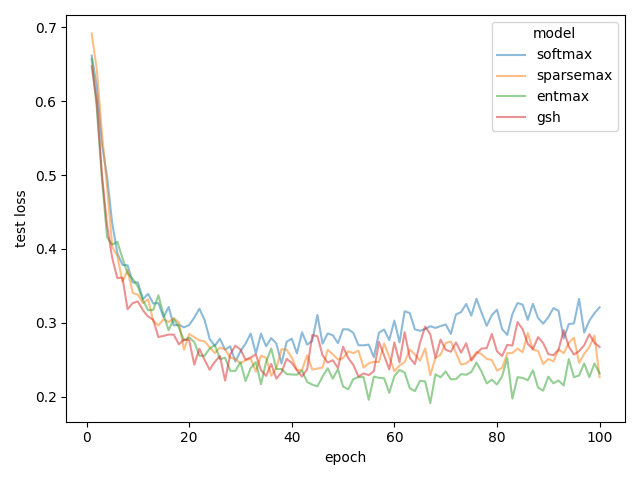}
    \vspace{-1em}
    \caption{\small The MIL experiment with bag size 10.
    From left to right:
    (1) Training data accuracy curve
    (2) Training data loss curve
    (3) Test data accuracy curve
    (4) Test data accuracy curve
    }
    \label{fig:gsh_mil_10}
\end{figure*}

\begin{figure*}[h!]
    \centering
    \includegraphics[width=0.24\textwidth]{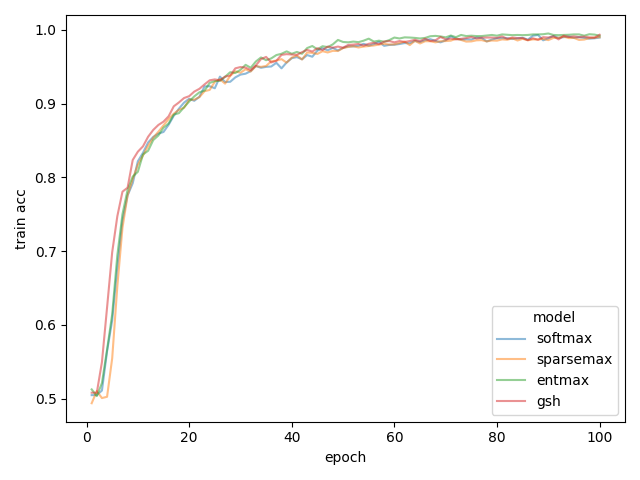}
    \hfill
    \includegraphics[width=0.24\textwidth]{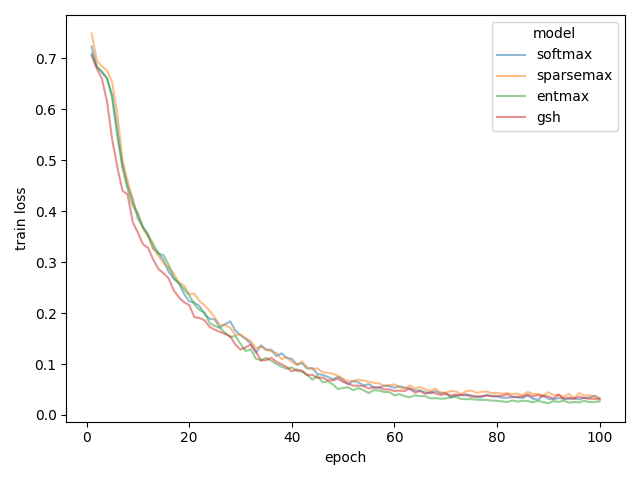}
    \hfill
    \includegraphics[width=0.24\textwidth]{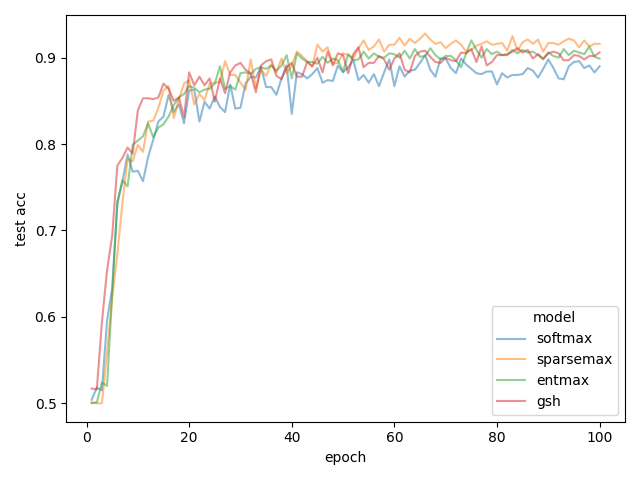}
    \hfill
    \includegraphics[width=0.24\textwidth]{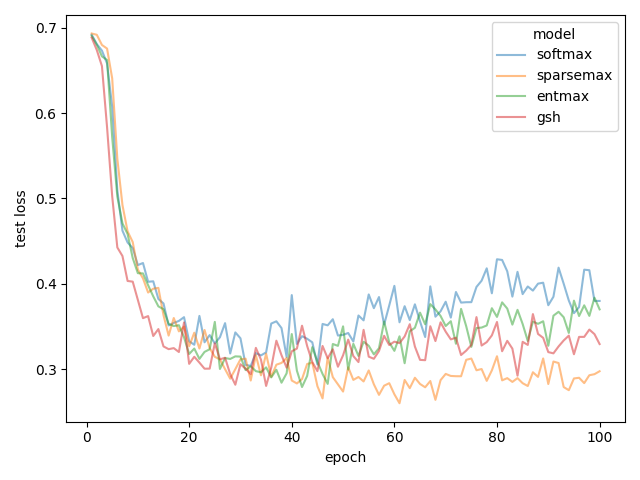}
    \vspace{-1em}
    \caption{\small The MIL experiment with bag size 20.
    From left to right:
    (1) Training data accuracy curve
    (2) Training data loss curve
    (3) Test data accuracy curve
    (4) Test data accuracy curve
    }
    \label{fig:gsh_mil_20}
\end{figure*}

\begin{figure*}[h!]
    \centering
    \includegraphics[width=0.24\textwidth]{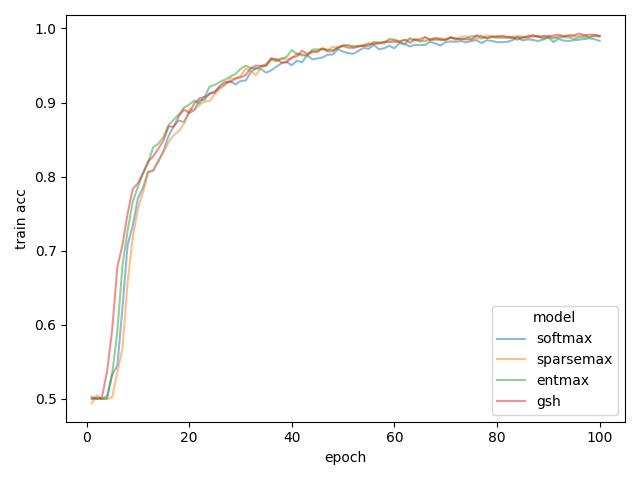}
    \hfill
    \includegraphics[width=0.24\textwidth]{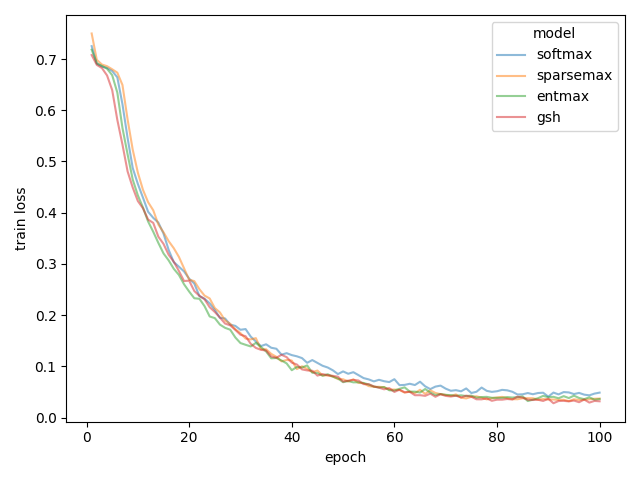}
    \hfill
    \includegraphics[width=0.24\textwidth]{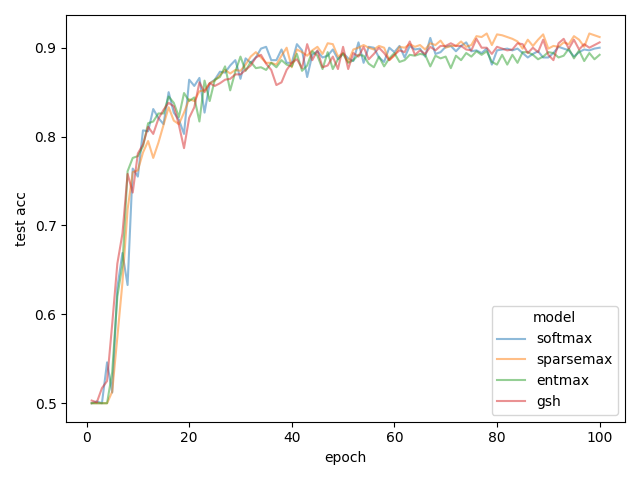}
    \hfill
    \includegraphics[width=0.24\textwidth]{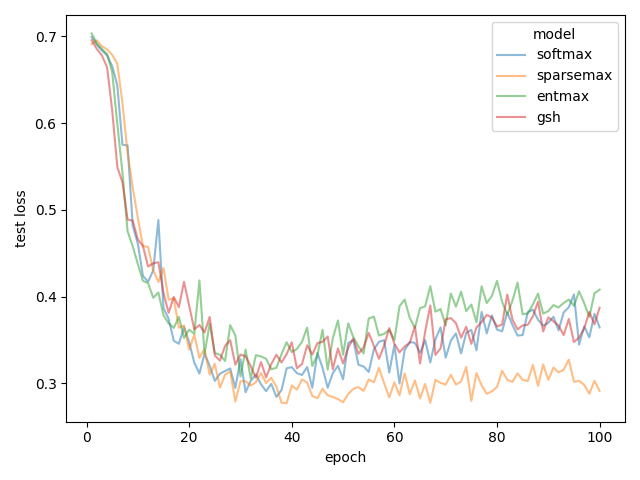}
    \vspace{-1em}
    \caption{\small The MIL experiment with bag size 30.
    From left to right:
    (1) Training data accuracy curve
    (2) Training data loss curve
    (3) Test data accuracy curve
    (4) Test data accuracy curve
    }
    \label{fig:gsh_mil_30}
\end{figure*}

\begin{figure*}[h!]
    \centering
    \includegraphics[width=0.24\textwidth]{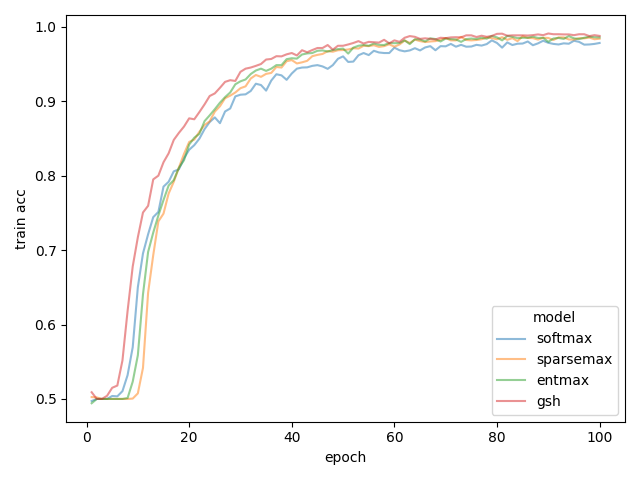}
    \hfill
    \includegraphics[width=0.24\textwidth]{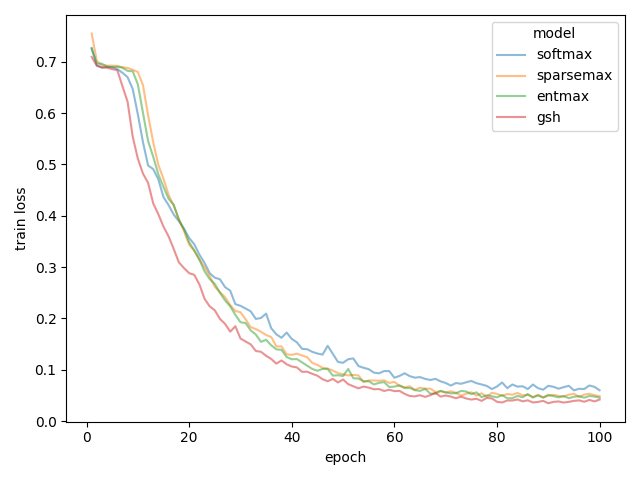}
    \hfill
    \includegraphics[width=0.24\textwidth]{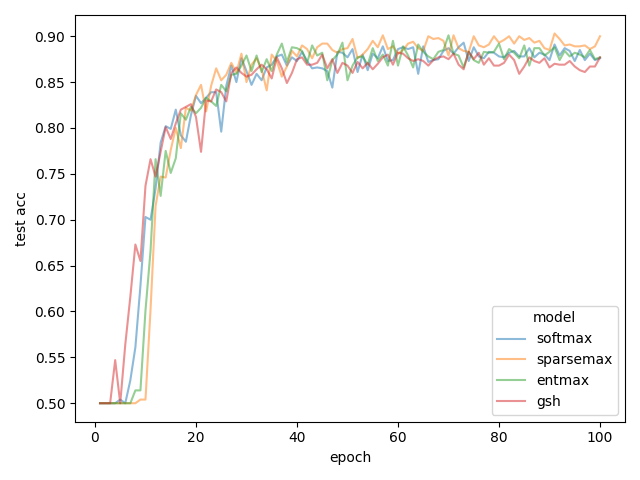}
    \hfill
    \includegraphics[width=0.24\textwidth]{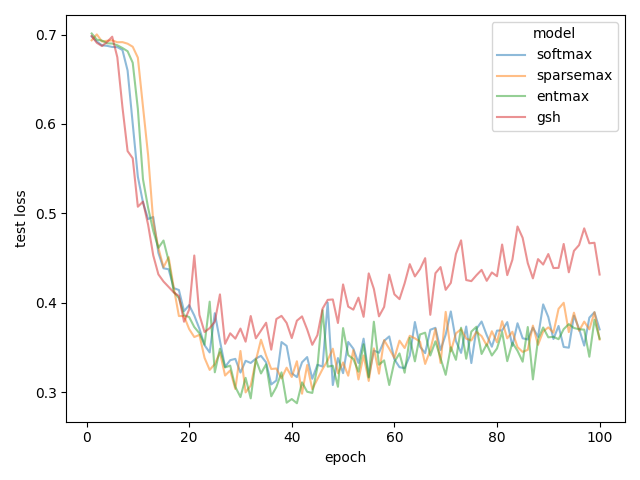}
    \vspace{-1em}
    \caption{\small The MIL experiment with bag size 50.
    From left to right:
    (1) Training data accuracy curve
    (2) Training data loss curve
    (3) Test data accuracy curve
    (4) Test data accuracy curve
    }
    \label{fig:gsh_mil_50}
\end{figure*}

\begin{figure*}[h!]
    \centering
    \includegraphics[width=0.24\textwidth]{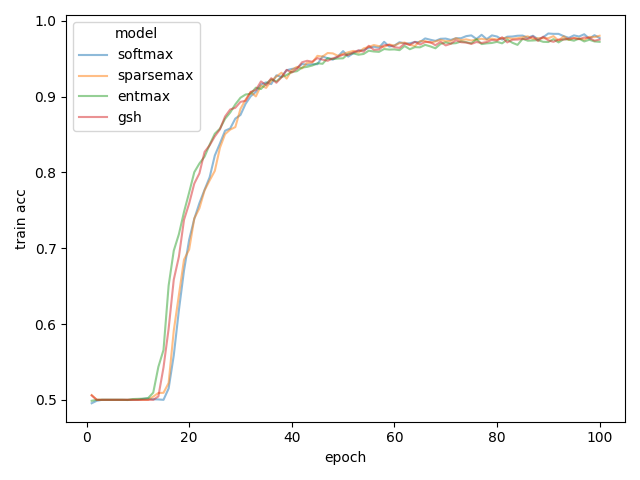}
    \hfill
    \includegraphics[width=0.24\textwidth]{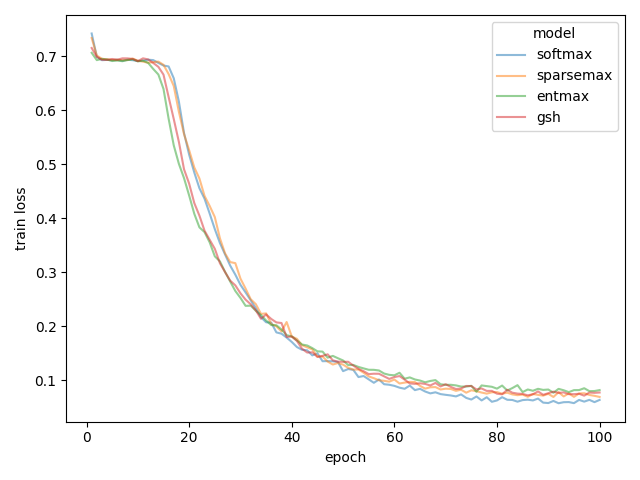}
    \hfill
    \includegraphics[width=0.24\textwidth]{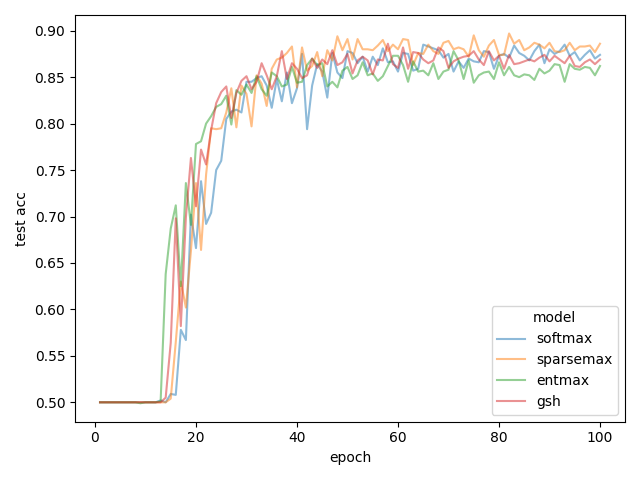}
    \hfill
    \includegraphics[width=0.24\textwidth]{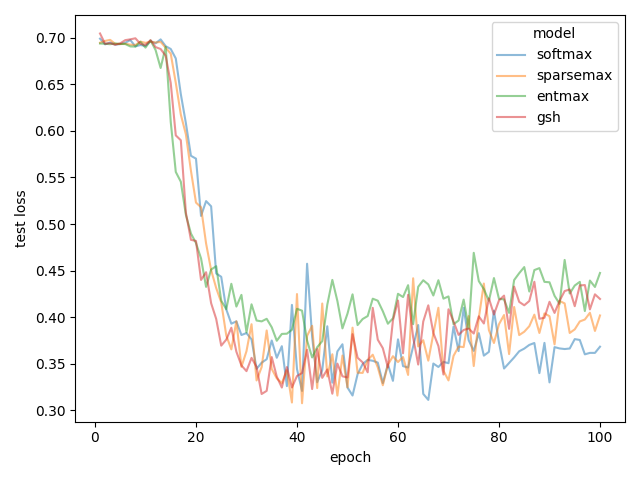}
    \vspace{-1em}
    \caption{\small The MIL experiment with bag size 100.
    From left to right:
    (1) Training data accuracy curve
    (2) Training data loss curve
    (3) Test data accuracy curve
    (4) Test data accuracy curve
    }
    \label{fig:gsh_mil_100}
\end{figure*}

\clearpage

\subsection{STanHop-Net Outperforms DLinear in Settings Dominated by Multivariate Correlations: A Case Study}

The performance of STanHop-Net in the main text, as presented in \cref{table:result}, does not show a clear superiority over DLinear \cite{zeng2023transformers}. 
We attribute this to the nature of the common benchmark datasets in \cref{table:result}, which are not dominated by multivariate correlations.

To verify our conjecture, we employ a strongly correlated multivariate time series dataset as a test bed, representing a practical scenario where multivariate correlations are the predominant source of predictive information in input features. 
In such a scenario, following the same setting in \cref{sec:MTS}, our experiments  show that STanHop-Net consistently outperforms DLinear.

Specifically, we follow the experimental settings outlined in \cref{sec:MTS}, but with a specific focus on cases involving \textit{small} lookback windows. 
This emphasis aims to reduce autoregressive correlation in data with smaller lookback windows, thereby increasing the dominance of multivariate correlations.

\paragraph{Dataset.}
We evaluate our model on the synthetic dataset\footnote{\url{https://www.dropbox.com/scl/fo/cg49nid4ogmd6f6sdbwqz/h?rlkey=jydf3ay1fzeotl5ws2s0cjnlc&dl=0}} generated in the ICML2023 Feature Programming paper \cite{reneau2023feature}.
Feature programming is an automated, programmable method for feature engineering.
It produces a large number of predictive features from any given input time series. 
These generated features, termed \textit{extended features}, are by construction highly correlated. 
The synthetic dataset, containing 44 extended features derived from the taxi dataset (see \cite[Section~D.1]{reneau2023feature} for the generation details), is thereby a strongly correlated multivariate time series dataset.
Special thanks to the authors of \cite{reneau2023feature} for sharing the dataset.

\paragraph{Baseline.}
We mainly compare our performance with DLinear \cite{zeng2023transformers} as it showed comparable performance in \cref{table:result}.

\paragraph{Setting.}
For both STaHop-Net and DLinear, we use the same hyperparameter setup as we used in the ETTh1 dataset.
We also conduct two ablation studies with varying lookback windows.
We report the mean MAE, MSE and R2 score over 5 runs.

\paragraph{Results.}
Our results (\cref{tab:stock}) demonstrate that STanHop-Net consistently outperforms DLinear when multivariate correlations dominate the predictive information in input features. 
Importantly, our ablation studies show that increasing the lookback window size, which reduces the dominance of multivariate correlations, results in DLinear's performance becoming comparable to, rather than being consistently outperformed by, STanHop-Net.
This explains why DLinear exhibits comparable performance to STanHop-Net in \cref{table:result}, when the datasets are not dominated by multivariate correlations.

\begin{table*}[h]
\centering
\caption{Comparison between DLinear and StanHop-Net on the synthetic dataset generated in \cite{reneau2023feature}.
This dataset is  by construction a strongly correlated multivariate time series dataset.
We report the mean MAE, MSE and R2 score over 5 runs.
The $A \rightarrow B$ denotes the input horizon $A$ and prediction horizon $B$.
\textbf{CSR (Cross-Sectional Regression):} 
Focuses on using single time step information from multivariate time series for predictions. 
\textbf{Ablation1:} With a prediction horizon of 1, the lookback window is increasing, thereby the dominance of multivariate correlations is decreasing. 
\textbf{Ablation2:} With a prediction horizon of 2, the lookback window is increasing, thereby the dominance of multivariate correlations is decreasing.
Our results aligns with our expectations:
STanHop-Net uniformly beats DLinear \cite{zeng2023transformers} in the Cross-Sectional Regression (CSR) setting.
Importantly, our ablation studies show that increasing the lookback window size, which reduces the dominance of multivariate correlations, results in DLinear's performance becoming comparable to, rather than being consistently outperformed by, STanHop-Net.
This explains why DLinear exhibits comparable performance to STanHop-Net in \cref{table:result}, when the datasets are not dominated by multivariate correlations.
}
\begin{tabular}{ccccccccc}
\toprule
 \multicolumn{2}{c}{\small lookback\_window $\to$ pred\_horizon} 
 & \multicolumn{3}{c}{DLinear}  
 & \multicolumn{3}{c}{STanHop-Net}  \\
\midrule
& & MSE & MAE & $R^2$ & MSE & MAE & $R^2$ \\
\midrule
\multirow{5}{1em}{\rot{CSR}} & 
$1 \rightarrow 1$ & 0.896 & 0.615 & 0.256 & \textbf{0.329} & \textbf{0.375} &\textbf{ 0.633} \\
& $1 \rightarrow 2$ & 1.193 & 0.794 & 0.001 & \textbf{0.417} & \textbf{0.428} & \textbf{0.552} \\ 
& $1 \rightarrow 4$ & 1.211 & 0.806 & -0.002 & \textbf{0.592} & \textbf{0.522} & \textbf{0.383} \\ 
& $1 \rightarrow 8$ & 1.333 & 0.868 & -0.100 & \textbf{0.812} & \textbf{0.636} & \textbf{0.182} \\ 
& $1 \rightarrow 16$ & 1.305 & 0.846 & -0.069 & \textbf{1.028} & \textbf{0.734} & \textbf{-0.058} \\
\midrule
\multirow{4}{1em}{\rot{Ablation1}} 
&  $2 \rightarrow 1$ & 0.514 & 0.504 & 0.573 & \textbf{0.328} & \textbf{0.366} & \textbf{0.710} \\
& $4 \rightarrow 1$ & 0.373 & 0.417 & 0.690 & \textbf{0.328} & \textbf{0.364} & \textbf{0.712} \\
& $8 \rightarrow 1$ & 0.328 & 0.380 & \textbf{0.727} & \textbf{0.327} & \textbf{0.367} & 0.715 \\
& $16 \rightarrow 1$ & \textbf{0.319} & 0.372 & \textbf{0.736}& 0.323 & \textbf{0.361} & 0.717 \\
\midrule
\multirow{4}{1em}{\rot{Ablation2}} 
& $2 \rightarrow 2$ & 0.771 & 0.632 & 0.359 & \textbf{0.424} & \textbf{0.425} & \textbf{0.630} \\
& $4 \rightarrow 2$ & 0.423 & 0.439 & \textbf{0.645} & \textbf{0.410} & \textbf{0.415} & 0.643 \\
& $8 \rightarrow 2$ & 0.647 & 0.441 & 0.646 & \textbf{0.402} & \textbf{0.412} & \textbf{0.655} \\
& $16 \rightarrow 2$ & \textbf{0.419} & 0.435 & \textbf{0.652} & 0.435 & \textbf{0.433} & 0.626 \\
\bottomrule
\end{tabular}
\label{tab:stock}
\end{table*}

\clearpage
\section{Experiment Details}\label{sec:exp_detail}
Here we present the  details of experiments in the main text.

\subsection{Experiment Details of Multivariate Time Series Predictions without Memory Enhancements}

\paragraph{Datasets.}
These datasets, commonly benchmarked in literature \cite{zhang2022crossformer,wu2021autoformer,zhou2021informer}.

\begin{itemize}
    \item \textbf{ETT (Electricity Transformer Temperature) \cite{zhou2021informer}:}
    ETT records 2 years of data from two counties in China. 
    We use two sub-datasets: ETTh1 (hourly) and ETTm1 (every 15 minutes). 
    Each entry includes the ``oil temperature" target and six power load features.
    
    \item \textbf{ECL (Electricity Consuming Load):}
    ECL records electricity consumption (Kwh) for 321 customers. 
    Our version, sourced from \cite{zhou2021informer}, covers hourly consumption over 2 years, targeting ``MT 320''.
    
    \item \textbf{WTH (Weather):}
    WTH records climatological data from approximately 1,600 U.S. sites between 2010 and 2013, measured hourly. 
    Entries include the ``wet bulb" target and 11 climate features.
    
    \item 
    \textbf{ILI (Influenza-Like Illness):}
    ILI records weekly data on influenza-like illness (ILI) patients from the U.S. Centers for Disease Control and Prevention between 2002 and 2021. 
    It depicts the ILI patient ratio against total patient count.
    
    \item \textbf{Traffic:}
    Traffic records hourly road occupancy rates from the California Department of Transportation, sourced from sensors on San Francisco Bay area freeways.
\end{itemize}

\begin{table*}[h]
\centering
\caption{Dataset Sources}
\resizebox{\textwidth}{!}{   
\begin{tabular}{ll}
\toprule
Dataset & URL 
\\
\midrule
ETTh1 \& ETTm1 & \url{https://github.com/zhouhaoyi/ETDataset} 
\\
ECL & \url{https://archive.ics.uci.edu/dataset/321/electricityloaddiagrams20112014} 
\\
WTH & \url{https://www.ncei.noaa.gov/data/local-climatological-data/} 
\\
ILI & \url{https://archive.ics.uci.edu/ml/datasets/seismic-bumps} 
\\
Traffic & \url{https://www.kaggle.com/shrutimechlearn/churn-modelling} 
\\
\bottomrule
\end{tabular}
}
\label{tab:link}
\end{table*}

\paragraph{Training.}
We use Adam optimizer to minimize the MSE Loss. The coefficients of Adam optimizer, betas, are set to (0.9, 0.999).
We continue training till there are $\mathtt{Patience=3}$ consecutive epochs where validation loss doesn't decrease or we reach 20 epochs. Finally, we evaluate our model on test set with the best checkpoint on validation set.

\paragraph{Hyperparameters.}
For hyperparameter search,
for each dataset, we conduct hyperparameter optimization using the ``Sweep'' feature of Weights and Biases \cite{biewald2020experiment}, with 200 iterations of random search for each setting to identify the optimal model configuration. 
The search space for all hyperparameters are reported in \cref{tab:bdhyper}.

\begin{table*}[h]
\centering
\caption{STanHop-Net hyperparameter space.}

\begin{tabular}{lccccccccc}
\toprule
 & Parameter & Distribution \\
\midrule
& Patch size $P$ & [6, 12, 24] \\
& FeedForward dimension & [64, 128, 256] \\
& Number of encoder layer & [1, 2, 3] \\
& Number of pooling vectors & [10] \\
& Number of heads & [4, 8] \\
& Number of stacked STanHop blocks & [1] \\
& Dropout & [0.1, 0.2, 0.3] \\
& Learning rate & [5e-4, 1e-4, 1e-5, 1e-3] \\
& Input length on ILI & [24, 36, 48, 60] \\
& Input length on ETTm1 & [24, 48, 96, 192, 288, 672] \\
& Input length on other dataset & [24, 48, 96, 168, 336, 720]  \\
& Course level & [2, 4] \\ 
& Weight decay & [0.0, 0.0005, 0.001] \\
\bottomrule
\end{tabular}
\label{tab:bdhyper}
\end{table*}

\clearpage
\subsection{External Memory Plugin Experiment Details}
\label{sec:ext-mem}

The hyperparameter of the external memory plugin experiment can be found in \cref{tab:bdhyper}.
For ILI\_OT, we set the input length as 24, feed forward dimension as 32 and hidden dimension as 64.
For prediction horizon 60, we set the input length as 48, feed-forward dimension as 128 and hidden dimension as 256.
For ETTh1, we use the same hyperparameter set found via random search in \cref{table:result}.

For the ``bad" external memory set intervals, we pick 40 and 200 for ILI\_OT and ETTh1, which represents 40 timesteps (weeks) earlier and 200 timesteps (hours) earlier.
For ILI dataset, we set the memory set size as 15 and for ETTh1, we set it as 20.

For \textbf{Case 3} (ETTh1), we select construct the external memory pattern with interval 168 timesteps earlier (equivalent to 1 week).
For \textbf{Case 4} (ETTm1), we select construct the external memory pattern with interval 672 timesteps earlier (equivalent to 1 week).

\begin{figure*}
    \centering
    \includegraphics[width=\textwidth]{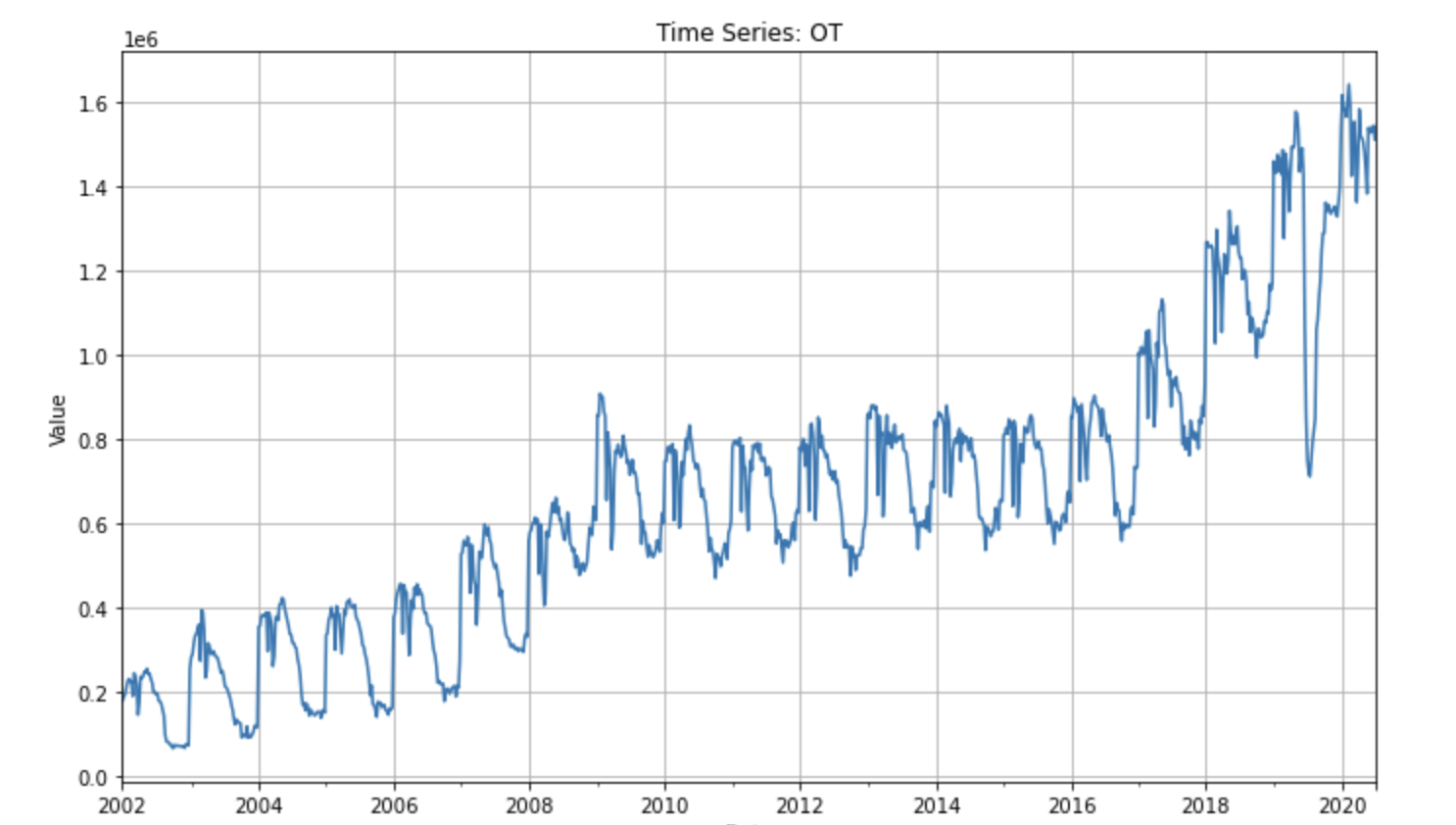}
    \caption{The visualization of ILI dataset ``OT" variate.}
    \label{fig:ILI_OT}
\end{figure*}

\section{Additional Theoretical Background}

\begin{remark}
\label{remark:closeform}
\citet{peters2019sparse} provide a closed-form expression for $\alphaentmax$ as
\bea
\alphaentmax(\bz)=[(\alpha-1)\bz-\tau(\bz)]^{\frac{1}{\alpha-1}},
\eea
where we denote $[t]_{+}\coloneqq \max\{t,0\}$, and $\tau$ is the threshold function $\R^M \to \R$ such that $\sum^M_{\mu=1}[(\alpha-1)\bz-\tau(\bz)]^{\frac{1}{\alpha-1}}=1$ satisfies the normalization condition of probability distribution.
\end{remark}

\end{document}